\documentclass[wcp]{jmlr}

\pdfoutput=1
\usepackage{rotating}% for sideways figures and tables
\usepackage{longtable}% for long tables
\usepackage{makecell,multirow}
\graphicspath{{arXivFigures/}} % Directory in which figures are stored
\usepackage{diagbox}
\usepackage{booktabs}
\usepackage{cancel}

\usepackage{xcolor}
\usepackage[final]{changes}
\colorlet{Changes@Color}{green!60!black}

\newtheorem{lem}{Lemma}
\newtheorem{rem}{Remark}

\newtheorem{cor}{Corollary}
\newtheorem{conj}{Conjecture}

\SetKwRepeat{Do}{do}{while}
\SetKw{KwGoTo}{go to}
\SetKwProg{Fn}{Function}{}{}
\SetKw{KwBreak}{break}

\jmlrvolume{101}
\jmlryear{2019}
\jmlrworkshop{ACML 2019}

\title[PAC-Bayesian Optimal Posteriors for Stochastic Classifiers]{Optimal PAC-Bayesian Posteriors for Stochastic Classifiers and their use for Choice of SVM Regularization Parameter}
   \author{\Name{Puja Sahu} \Email{puja.sahu@iitb.ac.in}\and
   \Name{Nandyala Hemachandra} \Email{nh@iitb.ac.in}\\
   \addr Indian Institute of Technology Bombay, Mumbai, India}

\editors{Wee Sun Lee and Taiji Suzuki}

\begin{document}

\maketitle

\begin{abstract}
PAC-Bayesian set up involves a stochastic classifier characterized by a posterior distribution on a classifier set, offers a high probability bound on its averaged true risk and is robust to the training sample used. For a given posterior, this bound captures the trade off between averaged empirical risk and KL-divergence based model complexity term. Our goal is to identify an optimal posterior with the least PAC-Bayesian bound. We consider a finite classifier set and 5 distance functions: KL-divergence, its Pinsker’s and a sixth degree polynomial approximations; linear and squared distances. Linear distance based model results in a convex optimization problem and we obtain a closed form expression for its optimal posterior. For uniform prior, this posterior has full support with weights negative-exponentially proportional to number of misclassifications. Squared distance and Pinsker’s approximation bounds are possibly quasi-convex and are observed to have single local minimum. We derive fixed point equations (FPEs) using partial KKT system with strict positivity constraints. This obviates the combinatorial search for subset support of the optimal posterior. For uniform prior, exponential search on a full-dimensional simplex can be limited to an ordered subset of classifiers with increasing empirical risk values. These FPEs converge rapidly to a stationary point, even for a large classifier set when a solver fails. We apply these approaches to SVMs generated using a finite set of SVM regularization parameter values on 9 UCI datasets. The resulting optimal posteriors (on the set of regularization parameters) yield stochastic SVM classifiers with tight bounds. KL-divergence based bound is the tightest, but is computationally expensive due to its non-convex nature and multiple calls to a root finding algorithm. Optimal posteriors for all 5 distance functions have lowest 10\% test error values on most datasets, with that of linear distance being the easiest to obtain.
\end{abstract}


\section{Introduction and Motivation}
Often we are faced with the issue of choosing a parameter of the learning algorithm, since this parameter has a significant role in determining the performance of the resulting classifier. For example, consider the Support Vector Machine (SVM) algorithm for classification with the regularization parameter, $\lambda > 0$. This parameter is a user input which trades off between model complexity and training error. The optimal classifier that we get, depends heavily on the sample $S$ that is used for training and the value of the parameter, $\lambda$. We can control only this parameter value for obtaining a classifier with low (training) error, \textit{but not} the given data. For a given training sample, we can choose the best value of the parameter from a prefixed set of values, which yields a classifier with the lowest error. However, this is a long drawn process. Additionally, there is no guarantee that the chosen value will yield a classifier having low(est) error on another sample from the same distribution. This implies that the best parameter value is sample dependent and that there is no unique value which is best for almost all the samples. However, if we determine the set of $\lambda$ values with lowest $\rho\%$ error rates on each sample, we observe a recurring subset of $\lambda$ values across the training samples. 
%We ran a set of experiments to validate this intuition.
(\textit{See Appendix A in the Suppl. file for an illustration}.) Thus, we have an ensemble of values to pick from. The PAC-Bayesian approach does such a stochastic selection.

%A natural suggestion is to use a stochastic selection among them, as PAC-Bayesian approach does. 

\paragraph{PAC-Bayesian Bounds and Optimal Posteriors} PAC-Bayesian approach assumes an arbitrary but fixed prior distribution on the space of classifiers and outputs a posterior distribution on this space, corresponding to a stochastic classifier. This approach provides a probabilistic bound on the difference between the posterior averaged true  and empirical risk of a stochastic classifier as measured by a convex distance function. For a given posterior, these bounds offer a trade-off between averaged empirical risk and a term which encompasses model complexity of the stochastic classifier. The bound is computed based on a single sample but with a high probability guarantee over different samples (from the same distribution). We are interested in the `optimal PAC-Bayesian posterior'. For a chosen distance function, the optimal posterior is defined as the one which minimizes the corresponding PAC-Bayesian bound. By design, these bounds and the resulting optimal posterior are robust to the choice of training sample, addressing the above sample bias.  % We believe that due to the regularized nature of the bound, we have a high probability guarantee on the generalization performance of this optimal posterior derived based on a single sample.

\paragraph{Relevant Work} % and Contributions of the Paper}
%A well known form of bounds estimating the unknown true risk of a classifier known as the  
PAC-Bayesian bounds were proposed by  \cite{mcallester2003pac,Seeger02theproof} and refined further by \cite{maurer2004note,langford2005tut,mcallester2013pac} using Bayesian priors and posteriors on the classifier space to provide better performance guarantees. Several authors improvised the bounds for the choice of distance function they considered. While \cite{maurer2004note} provided a bound for the KL-divergence as the distance function, $\phi$, by tightening up the threshold with a factor of $\sqrt{m}$ instead of $m$, \cite{germain2009pac} generalized the framework of PAC-Bayesian bounds for a broader class of convex $\phi$ functions and relaxed the constraints on tail bounds of empirical risks of the classifiers. \added{\cite{catoni2007pac} made an important contribution by considering bounds which are independent of distance function $\phi$, and instead require a parameter $C > 0$. Choice of $C$ can influence the bound on the performance of stochastic classifier just as the choice of $\phi$.} \cite{AmbroladzeNIPS2006} specialized PAC-Bayesian bounds using spherical Gaussian distributions on the space of linear classifiers. \cite{begin2016pac} introduced bounds based on  R{\'e}nyi divergence between posterior and prior distributions. We limit ourselves to KL-divergence based bounds.% in this paper.

All of the above consider a continuous (SVM) classifier space ($n$-dimensional Euclidean space) and continuous prior as well as posterior distributions on it (spherical Gaussian distributions) whereas we consider a finite set of classifiers such as those generated by a finite set of regularization parameter values for the SVM. Our PAC-Bayesian bounds are derived for the set up with a discrete prior distribution, and five different distance functions
% (KL-divergence and its two approximations based on Pinsker's inequality and its improvised version (a sixth degree polynomial), linear distance and squared distance) 
between posterior averaged empirical risk and posterior averaged true risk. 

\paragraph{Contributions} We consider \textit{optimal PAC-Bayesian posterior} which minimizes the PAC-Bayesian bound for a given distance function. We consider a finite classifier set and five distance functions: KL-divergence and its two approximations based on Pinsker's inequality and its improvised version (a sixth degree polynomial), linear distance and squared distance. The linear distance based optimal posterior is obtained via a convex program; is shown to have full support, with weights proportional to negative-exponential number of misclassifications when prior is uniform. Bounds based on KL-divergence as distance function and its sixth degree approximation are non-convex. Squared distance and Pinsker's approximation are possibly quasi-convex because they are observed to have single local minimum. We simplify the search for optimal posteriors via Fixed Point equations deduced from the partial KKT system with strict positivity constraints. We use these approaches on the set of SVMs generated by a finite set of regularization parameter values. This leads us to the notion of a \textit{stochastic SVM} characterized by an optimal posterior on the regularization parameter set. KL-distance yields the tightest bound, but is non-convex and has computational overhead of determining the root. All five distance functions have good generalization performance (lowest 10\% test error values) on most datasets considered, except for Bupa dataset and two almost linearly separable datasets, Banknote and Mushroom. Table \ref{tab:KL.resultoutline} describes theoretical and computational aspects of these optimal posteriors.

\paragraph{Outline} In Section \ref{secn:genericPACBbnd}, we consider PAC-Bayesian optimal posterior as the one minimizing the bound% based on a distance function
, and propose a Fixed Point (FP) scheme based on the partial KKT system. %In the next few sections, w
We analyze optimal posteriors for five distance functions: KL-distance (Section \ref{secn:klPACB}), its approximations (Section \ref{secn:klapproxmnPACB}), linear and squared distances (Sections \ref{secn:linPACB} and \ref{secn:sqPACB}). These approaches are applied to a set of SVMs (Section \ref{secn:SVM_PACB}) with summary in Section \ref{secn:conclusion}.

\renewcommand{\arraystretch}{3}
{ 
\begin{sidewaystable}[htp]
\begin{center}
\bgroup
\small
\def\arraystretch{0.99}% 
{\setlength{\tabcolsep}{0.1em}
\begin{tabular}{|c|c|c|c|c|}
\hline
 \makecell{{\textbf{Distance}} \\ {\textbf{fn $\phi$}}} & \multicolumn{4}{c|}{\normalsize {\textbf{Theoretical Aspects}}}  \\
\cline{2-5} 
  & {$\mathbf{\mathcal{I}^{K}_{\phi}} (m)$} &  {\textbf{Convexity}} & {\textbf{Global min}} &  {\textbf{Fixed Point (FP)}} \\
\hline 
$\phi_{\text{lin}}$ & Not required & Convex & $q^{\ast}_i \propto p_ie^{-m \hat{l}_i}$ & Not required \\
\hline
$\phi_{\text{sq}}$ & \makecell{{\tiny$\sum\limits_{k=0}^{m}{m \choose k} 0.5^m e^{m\left(\frac{k}{m}-0.5\right)^2}$}\\
 approximated by $2\sqrt{m}$} & \makecell{possibly quasi-convex}& \makecell{closed form \\ may not exist} & \makecell{ $ q^{FP}_{i, \text{sq, KL}} \propto p_ie^{ \left(-2\sqrt{m}\hat{l}_i \sqrt{\sum\limits_{i=1}^{H}q^{\text{FP}}_{i, \text{sq, KL}}\ln\frac{q^{\text{FP}}_{i, \text{sq, KL}}}{p_i} + \ln \frac{\mathcal{I}^{K}_{\text{sq}}(m)}{\delta}}\right)} $}\\
\hline
$kl$ & \makecell{$2\sqrt{m}$ \\(due to \\ \cite{maurer2004note})}& \makecell{Non-convex; \\ Difference of Convex  \\ (DC) functions}& \makecell{closed form \\ may not exist}& \makecell{$q^{FP}_{i, \text{kl, KL}}$ satisfies: \\ $q_{i} = p_i \exp \left \lbrace \sum\limits_{i=1}^{H}q_i\ln \frac{q_i}{p_i} - m\left(\sum\limits_{i=1}^{H}\hat{l}_iq_i - \hat{l}_i \right) \left[ \ln\left( \frac{(1 - r) \sum\limits_{i = 1}^H \hat{l}_i q_i}{r(1-\sum\limits_{i = 1}^H \hat{l}_i q_i)}\right) \right] \right \rbrace$ }\\
\hline
$\phi_{\text{P}}$ & \makecell{ approximated by $2\sqrt{m}$}  & possibly quasi-convex & \makecell{closed form \\ may not exist} & $q^{FP}_{i, \text{P, KL}} \propto p_ie^{ \left(-2\sqrt{2m}\hat{l}_i \sqrt{\sum_{i=1}^{H}q^{FP}_{i, \text{P, KL}}\ln\frac{q^{FP}_{i, \text{P, KL}}}{p_i} + \ln \frac{2\sqrt{m}}{\delta}}\right)}$
\\
\hline
$\phi_{\text{CH}}$ & \makecell{$0.9334m$ \\(due to \\ \cite{PACBIntervals})} & \makecell{shown to be \\ non-convex} & \makecell{closed form \\ may not exist}  &  \makecell{$q^{FP}_{i, \text{CH, KL}} \propto p_i\exp{\left\lbrace -(2m-1) \hat{l}_i \frac{ 2 \sqrt
{r_\text{CH}(R(Q^{FP}_{\text{CH, KL}}))}}{\frac{\partial
r_\text{CH}}{\partial R}}\right\rbrace }$ \\
($r_\text{CH}(R(Q))$ is the root of $\phi_{\text{CH}}$ for a given $\mathbb{E}_Q[\hat{l}]$ in \eqref{eqn:BCHKL})}\\
\hline
\end{tabular}}
\egroup
\end{center}

\begin{center}
\centering
\bgroup
\def\arraystretch{0.99}% 
{\setlength{\tabcolsep}{0.2em}
\begin{tabular}{|c|c|c|c|}
\hline
 \makecell{{\textbf{Distance}} \\ {\textbf{fn $\phi$}}} &  \multicolumn{3}{c|}{\small {\textbf{Computations}}} \\
\cline{2-4} 
  &  {\textbf{Solver (\texttt{Ipopt}) output}}  & {\textbf{Global minima}} & {\textbf{Fixed Point (FP)}}\\
\hline 
$\phi_{\text{lin}}$ &  identifies global minima & \makecell{identified analytically} & Not required\\
\hline
\makecell{$\phi_{\text{sq}}$ \\ $\phi_{\text{P}}$ \\ $\phi_{\text{CH}}$} & \makecell{identifies a unique (local) minima \\
 even with different initializations} & \makecell{closed form \\ may not exist} & matches solver output \\
\hline
$kl$ & \makecell{identifies multiple local minima  \\
 with different  initializations; \\
 throws up error for large $H$;\\
 especially for almost separable data} &  \makecell{closed form \\ may not exist} & \makecell{identifies same stationary point \\
 even with different initializations} \\
\hline
\end{tabular}}
\egroup
\end{center}
\caption[An outline of theoretical aspects and computational results for the KL-divergence based PAC-Bayesian bound minimization problem]{\small An outline of theoretical aspects and computational results for optimal posteriors $Q^{\ast}_{\phi, \text{KL}} = \lbrace q^{\ast}_{i, \phi, \text{KL}} \rbrace_{i = 1}^H$ for minimization of the PAC-Bayesian bound, $B_{\phi, \text{KL}}(Q)$, based on KL-divergence $KL[Q||P] = \sum_{i = 1}^H q_i \ln \frac{q_i}{p_i}$ between a posterior $Q$ and a prior $P$ on the classifier space $\mathcal{H}$. We consider five different distance functions, $\phi$: KL-divergence $kl(\hat{l}, l) = \hat{l}\ln \frac{\hat{l}}{l} + (1 - \hat{l}) \ln \left( \frac{1 - \hat{l}}{1 - l} \right)$, its Pinsker's approximation $\phi_\text{P}(\hat{l}, l) = 2(l - \hat{l})^2$ and a tighter approximation (a sixth degree polynomial) $\phi_\text{CH} = (l - \hat{l})^2 + \tfrac{2}{9}(l - \hat{l})^4 + \tfrac{16}{135}(l - \hat{l})^6$; linear $\phi_{\text{lin}} (\hat{l}, l) = l - \hat{l}$ and squared distances $\phi_{\text{sq}}(\hat{l}, l) = (l - \hat{l})^2$ for $l, \hat{l} \in (0,1)$.  $H$ denotes the classifier set size and $H^{\ast}$ denotes the size of the support set of the optimal posterior $Q^{\ast}_{\phi, \text{KL}}$. $\hat{l}_i$ denotes empirical risk value of classifier $ h_i \in \mathcal{H}$ computed on a sample of size $m$. ${\mathcal{I}^{K}_{\phi}} (m)$ is a sample size based constant for a distance function $\phi$. It is a component of the bound function $B_{\phi, \text{KL}}(Q)$. \label{tab:KL.resultoutline}}
\end{sidewaystable}
}
\renewcommand{\arraystretch}{1}

\section{PAC-Bayesian Bound Minimization, Optimal Posteriors and the Fixed Point Approach} \label{secn:genericPACBbnd}
We recall the general version of the PAC-Bayesian theorem \cite{germain2009pac, begin2016pac} for a given distance function and describe the notion of a PAC-Bayesian optimal posterior which minimizes the bound derived from the PAC-Bayesian theorem.
\begin{theorem}
[PAC-Bayesian Theorem \cite{germain2009pac, begin2016pac}] For any data distribution $\mathcal{D}$ over input space $\mathcal{X} \times \mathcal{Y}$, the following bound holds for any prior $P$ over the set of classifiers $\mathcal{H}$ and any $\delta \in (0, 1)$, where the probability is over random i.i.d. samples $S_m = \{(x_i, y_i) | i = 1, \ldots , m\}$ of size $m$ drawn from $\mathcal{D}$, for any convex function $\phi : [0,1] \times [0, 1] \rightarrow \mathbb{R}$:
\begin{equation}
\mathbb{P}_{S_m} \left \lbrace \forall Q \text{ on } \mathcal{H}: \; \phi\left(\mathbb{E}_Q [\hat{l}], \mathbb{E}_Q[l] \right) \leq \frac{KL[Q||P] + \ln\left( \frac{\mathbb{E}_{S \sim \mathcal{D}^m} \mathbb{E}_{h \sim P} e^{m \phi(\hat{l}, l)}}{\delta} \right)}{m} \right \rbrace \geq 1 - \delta. \label{eqn:genPAC_KLRHS}
\end{equation}
Here, $Q$ is an arbitrary posterior distribution on $\mathcal{H}$, which may depend on the sample $S_m$ and on the prior $P$. $\mathbb{E}_Q [\hat{l}] := \mathbb{E}_{h \sim Q} \sum_{i = 1}^m \frac{1}{m}[l(h, \mathbf{x}_i, y_i)]$ denotes the averaged empirical risk  and $\mathbb{E}_Q [l] := \mathbb{E}_{h \sim Q} \mathbb{E}_{(\mathbf{x}, y) \sim \mathcal{D}} [l]$ denotes averaged true risk of a classifier $ h \in \mathcal{H}$ computed using a loss function, $l(h, \mathbf{x}, y): \mathcal{H} \times \mathcal{X} \times \mathcal{Y} \rightarrow [a, b)$ (here, $0 \leq a < b$).
\end{theorem}

For a choice of distance function, $\phi$, the upper bound on $\mathbb{E}_{S \sim \mathcal{D}^m} \mathbb{E}_{h \sim P} e^{m \phi(\hat{l}(h), l(h))}$ determines the tightness of PAC-Bayesian bound. \cite{begin2016pac} give \linebreak$\mathcal{I}^{K}_{\phi}(m) :=\sup\limits_{l \in [0, 1]} \left[ \sum_{k = 0}^m \binom{m}{k} l^k (1- l)^{m - k} e^{m \phi(\frac{k}{m}, l)}\right]$ as an upper bound on $\mathbb{E}_{S_m \sim \mathcal{D}^m} \mathbb{E}_{h \sim P} e^{m \phi(\hat{l}(h), l(h))}$.% :
%\begin{align*}
%\mathbb{E}_{S_m \sim \mathcal{D}^m} \mathbb{E}_{h \sim P} e^{m \phi(\hat{l}, l)} 
%% &=  \mathbb{E}_{h \sim P} \mathbb{E}_{S_m \sim \mathcal{D}^m} e^{m \phi(\hat{l}, l)} \\
%% &= \mathbb{E}_{h \sim P} \left[ \sum_{k = 0}^m e^{m \phi(\hat{l}, l)} \mathbb{P} \left( \hat{l} = \frac{k}{m}\right) \right] \\
%%&=  \mathbb{E}_{h \sim P} \left[ \sum_{k = 0}^m \binom{m}{k} l^k (1- l)^{m - k} e^{m \phi(\frac{k}{m}, l)}\right] 
%& \leq \sup\limits_{l \in [0, 1]} \left[ \sum_{k = 0}^m \binom{m}{k} l^k (1- l)^{m - k} e^{m \phi(\frac{k}{m}, l)}\right] =: \mathcal{I}^{K}_{\phi}(m)
%\end{align*}

Thus, with the above upper bound on the right hand side threshold, \eqref{eqn:genPAC_KLRHS} becomes:

\begin{equation}
\mathbb{P}_{S_m} \left \lbrace  \forall Q \text{ on } \mathcal{H}: \;\phi\left(\mathbb{E}_Q [\hat{l}], \mathbb{E}_Q[l] \right) \leq \frac{KL[Q||P] + \ln\left( \frac{\mathcal{I}^{K}_{\phi}(m)}{\delta} \right)}{m} \right \rbrace \geq 1 - \delta. \label{eqn:BeginPACB}
\end{equation}
%where $\mathcal{I}^{K}_{\phi}(m) :=\sup_{l \in [0, 1]} \left[ \sum_{k = 0}^m \binom{m}{k} l^k (1- l)^{m - k} e^{m \phi(\frac{k}{m}, l)}\right]$ is an upper bound on \linebreak $\mathbb{E}_{S_m \sim \mathcal{D}^m} \mathbb{E}_{h \sim P} e^{m \phi(\hat{l}(h), l(h))}$.
%\begin{rem}
\added{For illustrating the role of this upper bound, $Q^{\ast}_{\text{sq, KL}}$ is computed with two values: $\mathcal{I}^K_{sq}(m)$ defined by \cite{begin2016pac} and $2\sqrt{m}$ by \cite{maurer2004note}. Bounds with $\mathcal{I}^K_{sq}(m)$ are tighter than those with $2\sqrt{m}$, and test error rates increase only marginally} (Please see Table \ref{tab:phiKLCompare}).
%\end{rem}

\subsection{Optimal posteriors via PAC-Bayesian bound minimization}
The PAC-Bayesian theorem \eqref{eqn:BeginPACB} gives the following high probability upper bound on averaged true risk, $\mathbb{E}_Q[l]$, assuming distance function $\phi(\mathbb{E}_Q[\hat{l}], \cdot)$ is invertible for given $\mathbb{E}_Q[\hat{l}]$:
\begin{equation}
\hspace{-5mm} B_{\phi, \text{KL}} (Q) \equiv B_{\phi, \text{KL}} (\mathbb{E}_Q[\hat{l}], S_m, \delta, P) = f_{\phi}\left(\mathbb{E}_Q[\hat{l}], \phi^{-1}_{\mathbb{E}_Q[\hat{l}]} \left( \frac{KL[Q||P] + \ln\left( \frac{\mathcal{I}^{K}_{\phi}(m)}{\delta} \right)}{m} \right) \right),
\end{equation}
where $\phi^{-1}_{\mathbb{E}_Q[\hat{l}]}(K) = b$ implies $\phi(\mathbb{E}_Q[\hat{l}], b) = K$ for some $b \in (0, 1)$ and a given $K > 0$. Generally $f_{\phi}(\cdot, \cdot)$ is the sum of its arguments except when $\phi$ is KL-distance function. That is, bound function $B_{\phi, \text{KL}} (Q)$ is the sum of averaged empirical risk, $\mathbb{E}_Q[\hat{l}]$, and a model complexity term which depends on system parameters, $S_m, \delta, P$. 
%Recall that we consider five different distance functions -- KL-divergence and its two approximations based on Pinsker's inequality and its improvised version (a sixth degree polynomial), linear distance and squared distance. The linear distance based optimal posterior is obtained via a convex program. Other distance functions generally result in non-convex programs. 
We are interested in determining an optimal posterior distribution $Q^{\ast}_{\phi, \text{KL}}$ which minimizes $B_{\phi, \text{KL}}(Q)$ for a given $\phi$. 

\subsection{The fixed point approach to determine PAC-Bayesian optimal posterior}
To characterize the minimum of $B_{\phi, \text{KL}}(Q)$, we make use of the first order KKT conditions which are necessary for a stationary point of a non-convex problem. These KKT conditions require the objective function and the active constraints to be differentiable at the local minimum. We derive fixed point (FP) equations for the optimal posterior for various distance functions in \eqref{eqn:qFP_klKL}, \eqref{eqn:qFP_PinKL}, \eqref{eqn:qFP_CHKL} and \eqref{eqn:qFP_sqKL} (with derivations in supplemenatry file). These FP equations use KKT system with strict positivity constraints due to which complementary slackness conditions are automatically satisfied; hence called `\textit{partial}' KKT system. We consider strict positivity constraints on posterior weights to avoid the combinatorial problem of choosing the subset of classifiers which form the support set of the optimal posterior. 
%If the FP equation corresponds to a non-expansive mapping, its convergence is guaranteed. We don't have a proof for the non-expansiveness, but t
Computations illustrate that these FP equations always converge to a stationary point at a very fast rate, even for a large classifier set when a non-convex solver fails to identify a local solution. (Please see Table \ref{tab:Bnd.klKL} for an illustration of such cases.)

We work with a finite set  of classifiers: $\mathcal{H} = \lbrace h_i \rbrace_{i = 1}^H$ of size $H$. The prior, $P = \lbrace p_i \rbrace_{i =1}^H$ and posterior, $Q =\lbrace q_i \rbrace_{i =1}^H$ are discrete distributions on $\mathcal{H}$, where $p_i, q_i \geq 0 \; \forall i = 1, \ldots, H$ with $\sum_{i = 1}^H p_i = 1$ and $\sum_{i = 1}^H q_i = 1$. For differentiability required by KKT conditions, our objective function should have open domain, that is, the interior of the $H$-dimensional probability simplex: $int(\Delta^H) = \lbrace (q_1, \ldots, q_H) \vert q_i > 0 \; \forall i = 1, \ldots, H; \sum_{i = 1}^H q_i = 1 \rbrace$. In computations, we consider $q_i \geq \epsilon \; \forall i = 1, \ldots, H$ for $\epsilon > 0$ to ensure existence of a minimizer in $int(\Delta^H)$. Our FP equations are derived using partial KKT system on $int(\Delta^H)$.

\section{Optimal posterior, $Q^{\ast}_{\phi, \text{KL}}$,  for uniform prior}
%In the previous subsection, the optimal posterior is identified under the assumption that the entire set $\mathcal{H}$ is the support of posterior $Q$. 
We consider the special case of uniform prior on entire $\mathcal{H}$. We want to identify the optimal posterior $Q^{\ast}_{\phi, \text{KL}}$ with the $H$-dimensional probability simplex as the feasible region. We show below that it is enough to restrict the search space to certain subsets of this simplex. This reduces the computational complexity of the search from exponential scale to linear scale.

\begin{theorem} \label{thm:increasing.subsets.phiKL}
Consider a uniform prior distribution on the set $\mathcal{H}$ of classifiers, and a given set of posterior weights $Q = \lbrace q_j \rbrace_{j = 1}^{H'}$. We have three choices of distance function $\phi = \lbrace \phi_{\text{lin}}, \phi_{\text{sq}}, kl \rbrace$. Then among all subsets $\mathcal{H}' \subset \mathcal{H}$ of size $H'$, the smallest bound value $B_{\phi, \text{KL}}(Q, \mathcal{H}')$ corresponding to the given posterior weights $Q$ is achieved when $\mathcal{H}'$ is the subset formed by the first $H'$ elements of the ordered set of classifiers ranked by non-decreasing empirical risk values, $\hat{l}_1 \leq \hat{l}_2 \leq \ldots \leq \hat{l}_{H}$.
\end{theorem}
\begin{proof} (\textit{Please see Appendix C in suppl. file for other distance functions})
We consider linear distance based bound, $B_{\text{lin, KL}}(Q, \mathcal{H}')$ under the given set up, defined as follows:
\begin{align}
B_{\text{lin, KL}}(Q, \mathcal{H}') &:= \sum_{i\in \mathcal{H'}} \hat{l}_i q_i + \frac{\sum\limits_{i \in \mathcal{H}'}q_i \ln q_i + \ln H + \ln\left( \frac{\mathcal{I}^{K}_{{\text{lin}}}(m)}{\delta} \right)}{m}\label{eqn:BlinKL_unifP.genQ} %\\
%B_{\text{sq, KL}}(Q, \mathcal{H}') &:= \sum_{i\in \mathcal{H'}} \hat{l}_i q_i + \sqrt{\frac{\sum\limits_{i \in \mathcal{H}'}q_i \ln q_i + \ln H + \ln\left( \frac{\mathcal{I}^{K}_{{\text{sq}}}(m)}{\delta} \right)}{m}}. \label{eqn:BsqKL_unifP.genQ}
\end{align} 
For a given set of posterior weights $\lbrace q_j \rbrace_{j = 1}^{H'}$, the term $\sum_{i \in \mathcal{H}'}q_i \ln q_i$ of the bound $B_{\text{lin, KL}}(Q, \mathcal{H}')$ is invariant of the support set $\mathcal{H}'$ as long as its cardinality is $H'$. Thus $B_{\text{lin, KL}}(Q, \mathcal{H}')$ is the smallest when the sum $\sum_{i \in \mathcal{H}'} \hat{l}_i q_i$ is minimized. This will happen when $\mathcal{H}'$ consists of classifiers with smallest $H'$ values in the set $\lbrace \hat{l}_i \rbrace_{i = 1}^{H}$.
%of empirical risks of all the classifiers in $\mathcal{H}$.
Furthermore, if the elements of $\mathcal{H}'$ are ordered by non-decreasing empirical risk values, $\hat{l}_1 \leq \hat{l}_2 \leq \ldots \leq \hat{l}_{H'}$, the weights $\lbrace q_j \rbrace_{j = 1}^{H'}$ should be ordered non-increasingly. So, the theorem holds for linear distance function.
\end{proof}
\begin{cor}
As a consequence of the above Theorem \ref{thm:increasing.subsets.phiKL}, for determining the (globally) optimal posterior $Q^{\ast}_{\phi, \text{KL}}$, it is sufficient to compare the bound values corresponding to the best  posteriors on ordered subsets of $\mathcal{H}$, ranked by non-decreasing $\hat{l}_i$ values. These ordered subsets can be uniquely identified by their size.
% An ordered subset of size 1 is $\lbrace \hat{l}_1 \rbrace$, of size 2 is $\lbrace \hat{l}_1, \hat{l}_2 \rbrace$ and so on. Thus there exists an isomorphism between the set $\lbrace 1, \ldots, H \rbrace$ (which denote the subset size) and the family of ordered increasing subsets of $\mathcal{H}$.
\end{cor}

\section{Optimal PAC-Bayesian Posterior using KL-distance } \label{secn:klPACB}
The most commonly referenced version of the PAC-Bayesian theorem was given by \cite{Seeger02theproof} and improved by \cite{maurer2004note}, as given below:
\begin{theorem}
[PAC-Bayesian Theorem for KL-distance \cite{maurer2004note}] For any data distribution $\mathcal{D}$ over input space $\mathcal{X} \times \mathcal{Y}$, the following bound holds for any prior $P$ over the set of classifiers $\mathcal{H}$ and any $\delta \in (0, 1)$, where the probability is over random i.i.d. samples $S_m = \{(x_i, y_i) | i = 1, . . . , m\}$ of size $m$ drawn from $\mathcal{D}$:
\begin{equation}
\mathbb{P}_{S_m} \left \lbrace  \forall Q \text{ on } \mathcal{H}: \;kl\left(\mathbb{E}_Q [\hat{l}], \mathbb{E}_Q[l] \right) \leq \frac{KL[Q||P] + \ln\left( \frac{2 \sqrt{m}}{\delta} \right)}{m} \right \rbrace \geq 1 - \delta. 
\end{equation}
Here, $Q$ is an arbitrary posterior distribution on $\mathcal{H}$, which may depend on the sample $S_m$ and on the prior $P$, and where $kl(p , q) = p \ln \left( \frac{p}{q}\right) + (1 - p) \ln \left( \frac{1 - p}{1 - q}\right)$ for any $p, q \in (0, 1)$.
\end{theorem}

The upper bound on the averaged true risk $\mathbb{E}_Q[l]$ corresponding to the above PAC-Bayesian theorem is obtained as:
\begin{equation}
B_{\text{kl, KL}}(Q) = \sup_{r \in (0, 1)} \left\lbrace r : kl\left(\mathbb{E}_Q [\hat{l}], r \right) \leq \frac{KL[Q||P] + \ln\left( \frac{2 \sqrt{m}}{\delta} \right)}{m} \right\rbrace
\end{equation}

An inverse $kl( \cdot,  \cdot)$ function does not exist since it is not a monotone function, and so the bound $B_{\text{kl, KL}}(Q)$ does not have an explicit form. 
%Therefore, the bound minimization problem corresponding to this choice of $\phi$ function is difficult to model. 
However, we can employ a numerical root finding algorithm such as that described in \cite{PACBIntervals} (Algo. (\textsc{KLroots})) to obtain $B_{\text{kl, KL}}(Q)$ for a given instance of system parameters.
%the upper bound for this function for given values of $\mathbb{E}_Q\left[ \hat{l}\right]$ and the right hand side threshold. \textcolor{blue}{cross-ref to algo, etc., saying that this is what we use??}

\subsection{The KL-distance bound minimization problem}
For a finite classifier space $\mathcal{H} = \lbrace h_i\rbrace_{i =1}^H$, this optimization problem can be described as:
\begin{subequations}  \label{eqn:BklKLOP}
\begin{align}
&\min_{q_1, \ldots, q_H, r} r\\
\hspace{-4cm}\text{s.t.} \quad & \left(\sum_{i = 1}^H \hat{l}_iq_i \right) \ln\left( \frac{\sum\limits_{i = 1}^H \hat{l}_iq_i}{r}\right) + \left(1 - \sum_{i = 1}^H \hat{l}_iq_i \right) \ln\left( \frac{ 1  - \sum\limits_{i = 1}^H \hat{l}_iq_i}{1 - r}\right) 
%\\ &\hspace{8cm}
= \frac{\sum\limits_{i = 1}^H q_i \ln \frac{q_i}{p_i} + \ln\frac{2\sqrt{m}}{\delta}}{m} \label{eqn:klDCcons} \\
& r \geq \sum_{i = 1}^H \hat{l}_iq_i  \label{eqn:r>EQl}\\
& \sum_{i = 1}^H q_i = 1,  %\label{eqn:sumQ}\\
\; q_i \geq 0, \; \forall i = 1, \ldots, H  \label{eqn:qi>0}
\end{align} 
\end{subequations}
Here, $r$ is the right root of $kl\left(\mathbb{E}_Q [\hat{l}], r \right) = \frac{KL[Q||P] + \ln\left( \frac{2 \sqrt{m}}{\delta} \right)}{m}$ for a given $\mathbb{E}_Q [\hat{l}]$. The above is known to be a non-convex problem with a difference of convex (DC) equality constraint \eqref{eqn:klDCcons}. The constraint \eqref{eqn:r>EQl} is a strict inequality which is relaxed for modelling purpose to have a feasible region with a closed domain.

\subsection{The posterior based on fixed point scheme, $Q^{\text{FP}}_\text{kl,KL}$}
We derive FP equation for KL-distance based bound optimization problem below:
\begin{theorem} \label{thm:qFP_klKL}
The bound minimization problem \eqref{eqn:BklKLOP} for the bound $B_{\text{kl, KL}}(Q)$ has a stationary point $Q^{\text{FP}}_\text{kl,KL}$ which can be obtained as the solution to the following fixed point equation:
\begin{equation}
q_i = p_i \exp \left \lbrace \sum_{i=1}^{H}q_i\ln \frac{q_i}{p_i} - m\left(\sum_{i=1}^{H}\hat{l}_iq_i - \hat{l}_i \right)\left[ \ln\left( \frac{(1 - r) \sum_{i = 1}^H \hat{l}_i q_i}{r(1-\sum_{i = 1}^H \hat{l}_i q_i)}\right) \right] \right \rbrace \; \forall i = 1, \ldots, H
\label{eqn:qFP_klKL}
\end{equation}
where $r$ is the solution to \eqref{eqn:klDCcons} and \eqref{eqn:r>EQl} for a given $Q = (q_1, \ldots, q_H)$.
\end{theorem}
\begin{proof}
The Lagrangian function for \eqref{eqn:BklKLOP} can be written as follows:
\begin{multline}
\mathcal{L}_{\text{kl, KL}} = r - \beta_0 \left[ \left(\sum_{i = 1}^H \hat{l}_iq_i \right) \ln\left( \frac{\sum_{i = 1}^H \hat{l}_iq_i}{r}\right) + \left(1 - \sum_{i = 1}^H \hat{l}_iq_i \right) \ln\left( \frac{ 1  - \sum_{i = 1}^H \hat{l}_iq_i}{1 - r}\right) \right.\\
\left. - \frac{\left(\sum_{i = 1}^H q_i \ln \frac{q_i}{p_i} + \ln\frac{2\sqrt{m}}{\delta} \right)}{m} \right] -\beta_1 \left( r - \sum_{i = 1}^H \hat{l}_iq_i \right) - \mu_0 \left(\sum_{i=1}^H q_i -1 \right) - \sum_{i = 1}^H \mu_i q_i
\end{multline}
Due to the strict inequality constraint \eqref{eqn:r>EQl}, complementary slackness conditions for a stationary point imply that the Lagrange multiplier $\beta_1$ should vanish at optimality ($\beta_1 = 0$). 

We assume that $q_i > 0 \forall i = 1, \ldots, H$, since otherwise $\ln q_i = \ln(0)$ is undefined. Even if we use fact that $\lim_{x \rightarrow 0^{+}} \ln x = - \infty$ to define $\frac{\partial \mathcal{L}_{\text{kl, KL}}}{\partial q_j}$ for some $j \in [H]$, the KKT condition will mean that  $\mu_j$ is infeasible. Therefore, for a stationary point, we have $q_i > 0$. And the complementary slackness conditions imply that $\mu_i = 0$ for all $i = 1, \ldots, H$. 

At an optimal solution, derivatives of $\mathcal{L}_{\text{kl, KL}}$ with respect to primal variables $r$ and $q_i$s, should be set to zero. By solving for these derivatives, we get the FP equation \eqref{eqn:qFP_klKL} which identifies a stationary point of \eqref{eqn:BklKLOP}. (\textit{Please see Appendix D.1 in suppl. file for details.})
\end{proof}
\textbf{Note:} The requirement that $q_i > 0 ~ \forall$ $i = 1, \ldots, H$ holds true for the KKT system of a generic PAC-Bayesian bound minimization because of KL-divergence measure  between posterior and prior distributions; so, we assume this condition for the other four $\phi$s also.

\begin{algorithm2e}[ht] \label{algo:FP_klKL}
\DontPrintSemicolon
\KwIn{$\delta \in (0,1), m, H, \lbrace \hat{l}_i \rbrace_{i=1}^H, \lbrace p_i \rbrace_{i=1}^H,\texttt{tol}>0$}
\KwOut{Fixed point solution: $\lbrace q^{FP}_{i, \text{kl, KL}} \rbrace_{i=1}^H$}
\tcc{Intialize $Q^0 = \lbrace q^0_i \rbrace_{i = 1}^H$ with a random distribution from $\Delta^{H}$ simplex}
%\For{$i  = 1$ to $H$}{
$q^{0}_i \sim \exp(1), \; \forall i = 1, \ldots, H$ \;
%}
%\For{$i  = 1$ to $H$}{
$q^{0}_i \gets \frac{q^{0}_i}{\sum_{j = 1}^H q^{0}_j} \; \forall i = 1, \ldots, H$ \;
%}
$RHS \gets \frac{\sum_{i=1}^{H}q^{0}_i\ln \frac{q^{0}_i}{p_i} + \ln \frac{2\sqrt{m}}{\delta}}{m}$\;
$r \gets $\textsc{KLroots}($\sum_{i=1}^{H}\hat{l}_iq^{0}_i, RHS)_2$\;
%Compute the next iterate using
%\begin{equation*}
$q_i^1 \gets p_i \exp \left \lbrace \sum_{i=1}^{H}q^{0}_i\ln \frac{q^{0}_i}{p_i} - m\left(\sum_{i=1}^{H}\hat{l}_iq^{0}_i - \hat{l}_i \right)\left[ \ln\left( \frac{(1 - r) \sum_{i = 1}^H \hat{l}_i q^{0}_i}{r(1-\sum_{i = 1}^H \hat{l}_i q^{0}_i)}\right) \right] \right \rbrace \; \forall i = 1, \ldots, H$\;
%\end{equation*}
\Do{$\Vert q^{1} - q^{0} \Vert >$ \texttt{tol}}{
	\For{$i  = 1$ to $H$}{
	$q^0_i \gets q^1_i $\; 
	}
	$RHS \gets \frac{\sum_{i=1}^{H}q^{0}_i\ln \frac{q^{0}_i}{p_i} + \ln \frac{2\sqrt{m}}{\delta}}{m}$\;
$r \gets $\textsc{KLroots}($\sum_{i=1}^{H}\hat{l}_iq^{0}_i, RHS)_2$\;
$q_i^1 \gets p_i \exp \left \lbrace \sum_{i=1}^{H}q^{0}_i\ln \frac{q^{0}_i}{p_i} - m\left(\sum_{i=1}^{H}\hat{l}_iq^{0}_i - \hat{l}_i \right)\left[ \ln\left( \frac{(1 - r) \sum_{i = 1}^H \hat{l}_i q^{0}_i}{r(1-\sum_{i = 1}^H \hat{l}_i q^{0}_i)}\right) \right] \right \rbrace \; \forall i = 1, \ldots, H$
%	\For{$i  = 1$ to $H$}{
%	$q_i^1 \gets p_i \exp \left \lbrace \sum_{i=1}^{H}q^{0}_i\ln \frac{q^{0}_i}{p_i} - m\left(\sum_{i=1}^{H}\hat{l}_iq^{0}_i - \hat{l}_i \right)\left[ \ln\left( \frac{(1 - r) \sum_{i = 1}^H \hat{l}_i q^{0}_i}{r(1-\sum_{i = 1}^H \hat{l}_i q^{0}_i)}\right) \right] \right \rbrace $\; 
%	}
}
\Return $\lbrace q^{1}_i \rbrace_{i=1}^H$
\caption[FP klKL]{\textsc{FP klKL}: Fixed point solution for PAC-Bayesian bound with KL-distance}
\end{algorithm2e}

%The proof of the following is available in Appendix C.1.1:
%\begin{lem}
%When all the classifiers have same empirical risk (all $\hat{l}_i$s are same), the optimal posterior for the bound minimization problem \eqref{eqn:BklKLOP} is $Q \equiv P$.
%\end{lem}
%\begin{proof}
%\textit{Please refer to the suppl. file for the proof.}
%\end{proof}

KL-distance based bound minimization is non-convex with multiple stationary points which makes it difficult to identify the global minimum even by FP scheme. The iterative root finding algorithm adds to the computational complexity of the bound minimization algorithm. Therefore, in the next section, we look for simpler and easily invertible approximations to KL-distance function in the PAC-Bayesian bound minimization.

\section{Optimal Posterior for PAC-Bayesian Bound Minimization based on approximations to KL-distance function} \label{secn:klapproxmnPACB}
We explore two approximations to the KL-distance function: a known Pinsker's approximation and another tighter approximation based on improvised Pinsker's inequality.

\subsection{Optimal PAC-Bayesian Posterior based on Pinsker's approximation} %, $\phi_\text{P}$}
Based on Pinsker's inequality \cite{fedotov2003refinements}, we get the following second order polynomial approximation to $kl(l,l')$: %\begin{equation}
$\phi_\text{P}(l, l') = 2(l - l')^2 \quad \forall l, l' \in [0,1]\times [0,1]$
%\end{equation}
which serves as a distance function in the PAC-Bayesian theorem:

\begin{equation} 
\mathbb{P}_{S_m} \left \lbrace  \forall Q \text{ on } \mathcal{H}: 2 \left( \mathbb{E}_Q\left[ \hat{l} \right] - \mathbb{E}_Q[l] \right)^2 \leq \frac{KL[Q||P] + \ln\left( \frac{2\sqrt{m}}{\delta} \right)}{m} \right \rbrace \geq 1- \delta. \label{eqn:PinskPACBbound}
\end{equation}
The associated PAC-Bayesian bound function is:
\begin{equation} \label{eqn:BPinKL}
B_\text{P, KL}(Q) := \sum_{i = 1}^H \hat{l}_i q_i + \sqrt{ \frac{\sum_{i = 1}^H q_i \ln \frac{q_i}{p_i} + \ln\left( \frac{2\sqrt{m}}{\delta} \right)}{2m} }.
\end{equation}
We wish to determine the optimal posterior $Q^{\ast}_{\text{P, KL}}$ which minimizes $B_\text{P, KL}(Q)$ subject to the constraints given in \eqref{eqn:qi>0}. The convexity of this bound function could not be established, but computationally this bound minimization problem is observed to have single local minimum. We propose that \eqref{eqn:BPinKL} is possibly quasi-convex.
Based on the proof for Theorem \ref{thm:qFP_klKL} for  KL-distance function, we identify the following FP equation for stationary point of \eqref{eqn:BPinKL}:
%\textcolor{blue}{mention the optimization problem `OF is as in Eq**? and modify the title??}
\begin{equation}
q^{FP}_{i, \text{P, KL}} = \frac{p_ie^{ \left(-2\sqrt{2m}\hat{l}_i \sqrt{\sum_{i=1}^{H}q^{FP}_{i, \text{P, KL}}\ln\frac{q^{FP}_{i, \text{P, KL}}}{p_i} + \ln \frac{2\sqrt{m}}{\delta}}\right)}}{\displaystyle\sum\limits_{i =1}^{H} p_i e^{ \left(-2\sqrt{2m}\hat{l}_i \sqrt{\sum_{i=1}^{H}q^{FP}_{i, \text{P, KL}}\ln\frac{q^{FP}_{i, \text{P, KL}}}{p_i} + \ln \left( \frac{2\sqrt{m}}{\delta} \right)} \right)}}  \quad \forall i = 1, \ldots, H.
\label{eqn:qFP_PinKL}
\end{equation}
%We propose that \eqref{eqn:BPinKL} is possibly quasi-convex since it has single local minimum as observed computationally.

\subsection{Optimal PAC-Bayesian Posterior based on improvised Pinsker's approximation, $\phi_{\text{CH}}$}
A lower bound for KL-divergence $kl(l,l')$ given by an improvised version of Pinsker's inequality \cite{fedotov2003refinements} is the following tighter sixth degree polynomial approximation: 
\begin{equation}
\phi_{\text{CH}}(l, l') = (l - l')^2 + \frac{2}{9}(l - l')^4 + \frac{16}{135}(l - l')^6 \quad \forall l,l' \in [0,1] \times [0,1] \label{eqn:phiCH}
\end{equation}
$\phi_{\text{CH}}$ is a valid distance function since it satisfies the Seeger's assumptions \cite{Seeger02theproof}. 

\begin{theorem}[\cite{PACBIntervals}]
PAC-Bayesian theorem with $\phi_{\text{CH}}$ is:
\begin{eqnarray} 
\mathbb{P}_{S_m} \left \lbrace  \forall Q \text{ on } \mathcal{H}: \phi_{\text{CH}} \left( \mathbb{E}_Q\left[ \hat{l} \right], \mathbb{E}_Q[l] \right) \leq \frac{KL[Q||P] + \ln\left( \frac{K_{\phi_{\text{CH}}}}{\delta} \right)}{2m - 1} \right \rbrace \geq 1- \delta, \label{eqn:tightCHbound} \\
\text{where  } K_{\phi_{\text{CH}}} :=4m \times \left[ 1 - e^{-\phi_{\text{CH}} \left(\frac{1}{2} \right)} \right] \approx 0.9334m.
\end{eqnarray}
\end{theorem}
Due to its structure, $\phi_\text{CH}(\hat{l}, \cdot)$ has a single positive real root and has a PAC-Bayesian bound:
\begin{subequations} \label{eqn:BCHKL}
\begin{align}
B_\text{CH, KL}(Q) &:= \mathbb{E}_Q [\hat{l}] + \sqrt{r_\text{CH}(R(Q))} \\
\text{where, } r_\text{CH}(R(Q)) &= - \frac{5}{8} + \sqrt[3]{\left( \frac{1225}{512} 
+ \frac{135}{32}R(Q)\right) + \frac{5}{32} \sqrt{729R^2(Q) + \frac{6615}{8}R(Q) + \frac{208980}{256}}} \nonumber \\
&+ \sqrt[3]{\left( \frac{1225}{512} + \frac{135}{32}R(Q)\right) - \frac{5}{32} \sqrt{729R^2(Q) + \frac{6615}{8}R(Q) + \frac{208980}{256}}},\\
R(Q) &= \frac{KL[Q||P] + \ln\left( \frac{K_{\phi_{\text{CH}}}}{\delta} \right)}{2m - 1} = \frac{\sum_{i = 1}^H q_i \ln \frac{q_i}{p_i} + \ln\left( \frac{K_{\phi_{\text{CH}}}}{\delta} \right)}{2m -1}.
\end{align}
\end{subequations}
The optimal posterior distribution $Q^{\ast}_\text{CH, KL}$ is the one which minimizes $B_\text{CH, KL}(Q)$  in \eqref{eqn:BCHKL}. %subject to constraints \eqref{eqn:qi>0}.
\begin{lem}
The bound function $B_\text{CH, KL}(Q)$ defined in \eqref{eqn:BCHKL} is a non-convex function and hence the associated bound minimization problem is non-convex program.
\end{lem}
%\begin{proof}
%We have examples which violate the first order test for convexity of $B_\text{CH, KL}(Q)$. (\textit{See suppl. file}). 
%But computations show that this optimization problem has a single local minimum.
%\end{proof}
We identify the following FP equation for a stationary point for minimizing \eqref{eqn:BCHKL}, based on the partial KKT system:
\begin{equation} \label{eqn:qFP_CHKL}
    q^{FP}_{i,\text{CH, KL}} =  \frac{p_i\exp{\left\lbrace -(2m-1) \hat{l}_i \frac{ 2 \sqrt
{r_\text{CH}(R(Q^{FP}_{\text{CH, KL}}))}}{\frac{\partial
r_\text{CH}}{\partial R}}\right\rbrace }}{\sum_{i =1}^{H}  p_i\exp\left\lbrace
-(2m-1) \hat{l}_i \frac{ 2 \sqrt
{r_\text{CH}(R(Q^{FP}_{\text{CH, KL}}))}}{\frac{\partial
r_\text{CH}}{\partial R}}\right\rbrace} \quad \forall i = 1, \ldots, H.
\end{equation}
%In computations, $Q^{FP}_ \text{CH, KL}$ identifies the local minimum for \eqref{eqn:BCHKL}.

\section{Optimal PAC-Bayesian Posterior using Linear Distance Function} \label{secn:linPACB}
One of the simplest distance functions is the linear distance function, $\phi_{\text{lin}}(\hat{l}, l) = l - \hat{l}$ for $\hat{l}, l \in [0,1]$. The PAC-Bayesian bound in this case takes the following simplified form:
\begin{equation} \label{eqn:linKLPACB}
\mathbb{P}_{S_m} \left \lbrace  \forall Q \text{ on } \mathcal{H}: \mathbb{E}_Q[l] - \mathbb{E}_Q [\hat{l}] \leq  \frac{KL[Q||P] + \ln\left( \frac{\mathcal{I}^{K}_{{\text{lin}}}(m)}{\delta} \right)}{m} \right \rbrace \geq 1 - \delta
\end{equation}
where $\mathcal{I}^{K}_{\text{lin}}(m) := \sup\limits_{l \in [0, 1]} \left[ \sum_{k = 0}^m \binom{m}{k} l^k (1- l)^{m - k} e^{m \left(l - \frac{k}{m} \right)}\right]$is a function of the sample size, $m$.

Thus, the corresponding PAC-Bayesian bound is:
\begin{equation} \label{eqn:BlinKL}
B_{\text{lin, KL}}(Q) := \mathbb{E}_Q [\hat{l}] + \frac{KL[Q||P] + \ln\left( \frac{\mathcal{I}^{K}_{{\text{lin}}}(m)}{\delta} \right)}{m}.
\end{equation}
We want to find the optimal distribution $Q^{\ast}_{\text{lin, KL}}$ which minimizes the bound $B_{\text{lin, KL}}(Q)$. 

\begin{rem} 
For $m \geq 1028$, computing $\mathcal{I}^K_\text{lin}(m)$ is difficult due to storage limitations in the range of floating point numbers -- gives  $\mathcal{I}^K_\text{lin}(m)$ as NaN. As it is just an additive term in the bound, it does not influence the optimal solution. Hence we can determine $Q^{\ast}_{\text{lin, KL}}$ even for large $m$ as shown in Table \ref{tab:phiKLCompare}, but is needed for computing $B_{\text{lin, KL}}(Q^{\ast}_{\text{lin, KL}})$. 
\end{rem}

\subsection{The linear distance bound minimization problem}
For a finite classifier space $\mathcal{H} = \lbrace h_i\rbrace_{i =1}^H$, this optimization problem can be described as:
\begin{equation} 
\begin{split}
\min_{q_1, \ldots, q_H} \; &\sum_{i = 1}^H \hat{l}_i q_i + \frac{\sum_{i = 1}^H q_i \ln \frac{q_i}{p_i}}{m}  \\
\text{s. t.} \; & \sum_{i = 1}^H q_i = 1, \; q_i \geq 0 \quad \forall i = 1, \ldots, H.
\end{split}\label{eqn:BlinKLOP}
\end{equation}
%Here, $Q = \left( q_1, \ldots, q_H \right)$ represents a discrete distribution on $\mathcal{H}$.

\subsection{Convexity of the bound function, \texorpdfstring{$B_{\text{lin, KL}}(Q)$}{B{lin, KL}(Q)}}
The bound function $B_{\text{lin, KL}}(Q)$ is convex in $Q$ since it is a positive affine transformation of $KL[Q \Vert P]$, which in turn is convex in $Q$. Also, the feasible region is the $H$-dimensional probability simplex which is a closed convex set. Hence \eqref{eqn:BlinKLOP} is a convex optimization problem. 
%In fact, $KL[Q||P]$ is strictly convex in $Q$ for a given $P$ and has a unique global minimum. 
Thus, KKT conditions are both necessary and sufficient for \eqref{eqn:BlinKLOP}. 

\subsection{The optimal posterior, $Q^{\ast}_{\text{lin, KL}}$}
\begin{theorem} \label{thm:OptQ_linKL}
The distribution $Q^{\ast}_{\text{lin, KL}} = (q^{\ast}_{1, \text{lin, KL}}, \ldots, q^{\ast}_{H, \text{lin, KL}})$ where
\begin{equation} \label{eqn:optQ_linKL}
q^{\ast}_{i, \text{lin, KL}} = \frac{p_ie^{-m\hat{l}_i}}{\sum_{i=1}^{H}p_i e^{-m\hat{l}_i}} ~ \forall i = 1, \ldots, H 
\end{equation}
is the optimal PAC-Bayesian posterior which minimizes the bound $B_{\text{lin, KL}} (Q)$ in \eqref{eqn:BlinKL}.
% \begin{equation*}
% B_{\text{lin, KL}} (Q) = \sum_{i = 1}^H \hat{l}_i q_i + \frac{\sum_{i = 1}^H q_i \ln \frac{q_i}{p_i} }{m} \end{equation*}
\end{theorem}
\begin{proof}
Since this is a differentiable convex OP, we identify the global minimizer \eqref{eqn:optQ_linKL} using the associated KKT system. (\textit{Please refer to details in Appendix E.2 in suppl. file})
\end{proof}
\begin{rem}
$Q^{\ast}_{\text{lin, KL}}$ in \eqref{eqn:optQ_linKL} is a Boltzmann distribution for a given $P$. In case of uniform prior, the optimal posterior weight ($q^{\ast}_{i, \text{lin, KL}}$) on a classifier is negative-exponentially proportional to the number of misclassifications ($m\hat{l}_i$) it makes on the (validation) sample.
\end{rem}

\begin{theorem}
When the prior is a uniform distribution on the set $\mathcal{H}$ of classifiers, the optimal posterior $Q_{\text{lin, KL}}^{\ast}$ for the bound minimization problem \eqref{eqn:BlinKLOP} has full support. That is, all the classifiers in $\mathcal{H}$ will have strictly positive posterior weight at optimality.
\end{theorem}
\begin{proof}
Using the result of Theorem \ref{thm:increasing.subsets.phiKL}, it is sufficient to compare the bound values corresponding to the best posteriors for all ordered subsets of $\mathcal{H}$, ranked by non-decreasing $\hat{l}_i$ values, to determine the optimal posterior for \eqref{eqn:BlinKLOP}. Using Theorem \ref{thm:OptQ_linKL}, the optimal posterior $Q_\text{lin, KL}^{\ast}(H')$ on an ordered subset of classifiers of size $H' \in [H]$ is given as: 
\begin{equation*}
q_\text{i, lin, KL}^{\ast} (H') = 
\begin{cases}
\frac{e^{-m\hat{l}_i}}{\sum_{i=1}^{H'} e^{-m\hat{l}_i}} ~ &\forall i = 1, \ldots, H' \\
0 ~ &\forall i = H' + 1, \ldots, H,
\end{cases}
\end{equation*}
and the optimal objective value is:
\begin{align*}
B_\text{lin, KL}(Q_\text{lin, KL}^{\ast}(H')) &= \sum_{i = 1}^{H} \hat{l}_iq^{\ast}_{i, lin, KL} + \frac{\sum_{i =1}^{H} q^{\ast}_{i, lin, KL} \ln (q^{\ast}_{i, lin, KL} H)}{m} \\
% &= \sum_{i = 1}^{H'} \hat{l}_i \left( \frac{e^{-m\hat{l}_i}}{\sum_{i=1}^{H'} e^{-m\hat{l}_i}} \right) + \frac{\sum\limits_{i =1}^{H'} \frac{e^{-m\hat{l}_i}}{\sum_{i=1}^{H'} e^{-m\hat{l}_i}} \cdot \ln \left( \frac{H e^{-m\hat{l}_i}}{\sum_{i=1}^{H'} e^{-m\hat{l}_i}}\right)}{m} \\
% &=  \frac{\sum_{i = 1}^{H'} \hat{l}_ie^{-m\hat{l}_i}}{\sum_{i=1}^{H'} e^{-m\hat{l}_i}}  + \frac{\sum\limits_{i =1}^{H'} \left( \frac{e^{-m\hat{l}_i}}{\sum_{i=1}^{H'} e^{-m\hat{l}_i}} \left[-m\hat{l}_i + \ln \left( \frac{H}{\sum_{i=1}^{H'} e^{-m\hat{l}_i}}\right) \right] \right)}{m} \\
% &=  \frac{\sum_{i = 1}^{H'} \hat{l}_ie^{-m\hat{l}_i}}{\sum_{i=1}^{H'} e^{-m\hat{l}_i}}  + \frac{-\cancel{m} \cdot \sum\limits_{i =1}^{H'} \left( \frac{e^{-m\hat{l}_i}}{\sum_{i=1}^{H'} \hat{l}_i e^{-m\hat{l}_i}} \right)}{\cancel{m}} + \frac{\sum\limits_{i =1}^{H'} \left( \frac{e^{-m\hat{l}_i}}{\sum_{i=1}^{H'} e^{-m\hat{l}_i}} \ln \left( \frac{H}{\sum_{i=1}^{H'} e^{-m\hat{l}_i}}\right) \right)}{m} \\
% &=  \cancel{\frac{\sum_{i = 1}^{H'} \hat{l}_ie^{-m\hat{l}_i}}{\sum_{i=1}^{H'} e^{-m\hat{l}_i}}} - \cancel{\frac{\sum_{i = 1}^{H'} \hat{l}_ie^{-m\hat{l}_i}}{\sum_{i=1}^{H'} e^{-m\hat{l}_i}}} + \ln \left( \frac{H}{\sum_{i=1}^{H'} e^{-m\hat{l}_i}}\right) \cdot \frac{\left( \frac{\cancel{\sum_{i =1}^{H'} e^{-m\hat{l}_i}}}{\cancel{\sum_{i=1}^{H'} e^{-m\hat{l}_i}}} \right)}{m} \\
&= \frac{\ln H - \ln \left( \sum_{i =1}^{H'} e^{-m\hat{l}_i} \right)}{m}
\end{align*}
%Since $e^{-x} > 0 $ for all $x \in \mathbb{R}$, the sum $\sum_{i =1}^{H'} e^{-m\hat{l}_i}$ is an increasing function of $H' \in [H]$. Using the monotone increasing property of natural logarithm function,
The bound, $B_\text{lin, KL}(Q_\text{lin, KL}^{\ast}(H'))$ is a decreasing function of $H' = 1, \ldots, H$. Therefore the least bound value is achieved when all classifiers are assigned strictly positive weights, that is, the optimal posterior has full support. (\textit{Details are in Appendix E.2 in suppl. file})
\end{proof} 
\begin{rem}
We believe that this full support for the optimal posterior, $Q_\text{lin, KL}^{\ast}$, is due to the KL-divergence measure on the right hand side threshold of the PAC-Bayesian bound, \eqref{eqn:linKLPACB}. As an implication, $Q_\text{lin, KL}^{\ast}$ considers even the worst performing classifier but with infinitesimally positive (negative-exponential) posterior weight. 
\end{rem}

\section{Optimal PAC-Bayesian Posterior using Squared Distance Function} \label{secn:sqPACB}

We now consider a widely used squared distance function \cite{mcallester2003pac,Seeger02theproof} between the averaged empirical risk and the averaged true risk :
$\phi_{\text{sq}} \left(\hat{l}, l \right) =  \left(\hat{l} - l \right)^2$ for $\hat{l}, l \in [0,1]$. With $\phi_{\text{sq}}$, the PAC-Bayesian theorem takes the following form:
\begin{equation}
\mathbb{P}_{S_m} \left \lbrace  \forall Q \text{ on } \mathcal{H}:\left(\mathbb{E}_Q [\hat{l}], \mathbb{E}_Q[l] \right)^2 \leq \frac{KL[Q||P] + \ln\left( \frac{\mathcal{I}^{K}_{{\text{sq}}}(m)}{\delta} \right)}{m} \right \rbrace \geq 1 - \delta,
\end{equation}
where $\mathcal{I}^{K}_{\text{sq}}(m) := \sup\limits_{l \in [0, 1]} \left[ \sum_{k = 0}^m \binom{m}{k} l^k (1- l)^{m - k} e^{m \left(\frac{k}{m} - l \right)^2}\right]$ is a function of the sample size, $m$.

The above PAC-Bayesian statement gives the following high probability upper bound:
\begin{equation}
B_{\text{sq, KL}}(Q) := \mathbb{E}_Q [\hat{l}] + \sqrt{ \frac{KL[Q||P] + \ln\left( \frac{\mathcal{I}^{K}_{{\text{sq}}}(m)}{\delta} \right)}{m} } . \label{eqn:BsqKL}
\end{equation}
We identify the constant term $\mathcal{I}^{K}_{{\text{sq}}}(m)$ in \eqref{eqn:BsqKL} based on \cite{begin2016pac}'s result.

\begin{lem}
For a given sample size, $m$, $\mathcal{I}^K_{\text{sq}}(m) := \sum_{k=0}^{m}{m \choose k} 0.5^me^{2m\left(\frac{k}{m}-0.5\right)^2} $.
\end{lem} 

\begin{rem}
On a machine equipped with 4 Intel Xeon 2.13 GHz cores and 64 GB RAM, we couldn't compute $\mathcal{I}^K_{\text{sq}} (m)$ for $m \geq 1028$ due to storage limitations for floating point numbers. Therefore, we upper bound it by $2\sqrt{m}$ for $m \geq 8$  \cite{begin2016pac}.
\end{rem}

\subsection{The squared distance bound minimization problem}
We want to determine the optimal posterior $Q^{\ast}_{\text{sq, KL}}$ which minimizes $B_{\text{sq, KL}}(Q)$. For a finite classifier space $\mathcal{H} = \lbrace h_i\rbrace_{i =1}^H$, this optimization problem can be described as:
\begin{equation} \label{eqn:BsqKLOP}
\begin{split}
\min_{q_1, \ldots, q_H} &\sum_{i = 1}^H \hat{l}_i q_i + \sqrt{ \frac{\sum_{i=1}^{H} q_i\ln\frac{q_i}{p_i}  + \ln\left( \frac{\mathcal{I}^{K}_{{\text{sq}}}(m)}{\delta} \right)}{m} } \\
\text{s. t.} \; & \sum_{i = 1}^H q_i = 1, \;  q_i \geq 0 \quad \forall i = 1, \ldots, H.
\end{split}
\end{equation}
%Here, $Q = \left( q_1, \ldots, q_H \right)$ represents a discrete distribution on $\mathcal{H}$.
%, $B_{\text{sq, KL}}(Q) =  \sum_{i=1}^{H} \hat{l}_iq_i  +   \sqrt{\frac{\sum_{i=1}^{H} q_i \ln\frac{q_i}{p_i} + \ln \frac{2\sqrt{m}}{\delta}}{m}} $

% \begin{rem} 
% The above optimization model imposes non-negativity constraint on the weights $q_i$ for all $i = 1, \ldots, H$, allowing some of them to take value 0. This means that in the computation of $KL[Q||P] = \sum_{i = 1}^{H} q_i \ln(H q_i)$, we can come across a term of the form \mbox{$0 \cdot \log (0)$}, where $\log(0)$ is rendered undefined. For analytical purposes, we define $\lim_{x \rightarrow 0} x \log x = 0$. However, the code shows an error when $\log(0)$ is encountered since it is undefined. So to avoid such errors during the computation, we set the lower bound on $q_i$ to be a small, positive number. As a result, we may end up with a sub-optimal posterior. To identify the global minimum, we optimize over (increasing) subsets of $\mathcal{H}$ rendering zero posterior mass on classifiers not in the subset.  \textcolor{blue}{not clear; including figure may be useful??}  %(Illustrated in Table \ref{tab:optQsubset}.)
% \label{remark:positiveLB_q}
% \end{rem}
% \begin{theorem}
% The bound function, $B_{\text{sq, KL}}(Q) = \sum_{i=1}^{H} \hat{l}_iq_i + \sqrt {\frac{\sum_{i=1}^{H} q_i\ln\frac{q_i}{p_i} + \ln \frac{2\sqrt{m}}{\delta}}{m}}$ is non-convex.
% \end{theorem}
% \begin{proof}
The convexity of this bound function could not be established, but computationally this bound minimization problem is observed to have a single local minimum, hinting at quasi-convexity of $B_{\text{sq, KL}}(Q)$. (\textit{Please see Appendices F.1 and F.2 in Suppl. file for proof.})

\subsection{The posterior based on fixed point scheme, $Q^{FP}_{\text{sq,KL}}$} 
We can identify a FP solution for \eqref{eqn:BsqKLOP} based on the partial KKT system by setting the derivatives of the Lagrange function for \eqref{eqn:BsqKLOP} to zero, and using the complementary slackness conditions, we get the FP equation \eqref{eqn:qFP_sqKL}. (\textit{Proof details are in Appendix F.3 in Suppl.file.})

\begin{theorem}
The bound minimization problem \eqref{eqn:BsqKLOP} 
% for the bound $B_{\text{sq, KL}}(Q) =\sum\limits_{i=1}^{H}\hat{l}_i q_i + \sqrt{\frac{\sum\limits_{i=1}^{H}q_i\ln\frac{q_i}{p_i} + \ln \left( \frac{\mathcal{I}^{K}_{\text{sq}}(m)}{\delta} \right) }{m}}$ 
has a stationary point which can be obtained as the solution to the following fixed point equation:
\begin{equation}
q^{FP}_{i, \text{sq, KL}} = \frac{p_ie^{ \left(-2\sqrt{m}\hat{l}_i \sqrt{\sum_{i=1}^{H}q^{FP}_{i, \text{sq, KL}}\ln\frac{q^{FP}_{i, \text{sq, KL}}}{p_i} + \ln \frac{\mathcal{I}^{K}_{\text{sq}}(m)}{\delta}}\right)}}{\displaystyle\sum\limits_{i =1}^{H} p_i e^{ \left(-2\sqrt{m}\hat{l}_i \sqrt{\sum_{i=1}^{H}q^{FP}_{i, \text{sq, KL}}\ln\frac{q^{FP}_{i, \text{sq, KL}}}{p_i} + \ln \left( \frac{\mathcal{I}^{K}_{\text{sq}}(m)}{\delta} \right)} \right)}},  \quad \forall i = 1, \ldots, H
\label{eqn:qFP_sqKL}
\end{equation}
\end{theorem}
%\begin{proof}
%By setting the derivatives of the Lagrange function for \eqref{eqn:BsqKLOP} to zero, and using the complementary slackness conditions, we get the FP equation \eqref{eqn:qFP_sqKL}; \textit{Details in Suppl.file.}
%%It can be easily verified from \eqref{eqn:qFP_sqKL} that all $q^{FP}_{i, \text{sq, KL}} > 0$ and they sum up to 1. Hence $q^{FP}_{i, \text{sq, KL}} $ is a feasible solution. Also, it is derived using the KKT conditions, hence it is a stationary point. 
%\end{proof}
\section{Choice of Regularization Parameter for SVMs} \label{secn:SVM_PACB}
For computations, we included nine datasets from UCI repository \cite{UCI:2017} with small to moderate number of examples (306 examples to 5463 examples)  and small to moderate number of 
features (3 features to 57 features). These datasets span a variety ranging from almost linearly separable (Banknote, Mushroom and Wave datasets) to moderately inseparable (Wdbc, Mammographic and Ionosphere datasets) to inseparable data (Spambase, Bupa and Haberman datasets). SVMs on these datasets have varying ranges and degrees of variation in their empirical risk values. We consider a finite set of SVM regularization parameter values $\Lambda = \lbrace \lambda_i \rbrace_{i = 1}^{H}$, say, between $0$ and an upper bound $\lambda_0 > 0$, since small values of $\lambda_i$'s are preferable. We took $\Lambda = \lbrace 0.1, 0.11, \ldots, 20\rbrace$ at a granularity of 0.01. SVM QP (with RBF kernels) was implemented using \texttt{ksvm} function in \texttt{kernlab} package \cite{kernlab} in { \em R (version 3.1.3 (2015-03-09))}. The Gaussian width parameter is estimated by \texttt{kernlab} using \texttt{sigest} function which estimates 0.1 and 0.9 quantiles of squared distance between the data points.

Each of these datasets was partitioned such that 80\% of the examples formed a composition of training set and validation set (in equal proportion) used for constructing the set $\mathcal{H} = \lbrace h(\lambda_i) | \lambda_i \in \Lambda \rbrace_{i = 1}^H$ of SVM classifiers and remaining 20\% used for computing their test error rates. The training set size ($m$), validation set size ($v$) and test set size ($t$) are in the ratio $m:v:t = 0.4:0.4:0.2$. The role of the validation set is to compute the empirical risk $\hat{l}_i$ of the SVM $h(\lambda_i) \in \mathcal{H}$ which will be used for deriving the PAC-Bayesian bound.  We follow the scheme provided in \cite{begin2016pac, thiemann2016quasiconvexPACB} to generate the set $\mathcal{H}$. Each classifier $h(\lambda_i) \in \mathcal{H}$ is trained on $m$ training examples subsampled from this composite set and validated on the remaining $v$ examples. Overlaps between training sets of different classifiers are allowed. Same is true for their validation sets. 

The PAC-Bayesian bound minimization problem for finding the optimal posterior was implemented in AMPL Interface and solved using \texttt{Ipopt} software package {\em (version 3.12 (2016-05-01))} \cite{ipopt}, a library for large-scale nonlinear optimization
%, released as open source code under the Eclipse Public License (EPL) 
(\url{http://projects.coin-or.org/Ipopt}). All computations were done on a machine equipped with 12 Intel Xeon 2.20 GHz cores and 64 GB RAM. We summarize comparisons among optimal posteriors for different distance functions in Table \ref{tab:phiKLCompare}.
 \paragraph{Fixed point solutions can be more reliable than solver output}
In case of KL-distance based bound, we observe that the FP scheme is able to converge to a stationary point even when solver fails to identify a local solution, as seen in Table  \ref{tab:Bnd.klKL}. More such cases are illustrated in Table 5 in supplementary file with 7 other datasets.
\begin{table}[]
{\footnotesize
\begin{center}
\setlength{\tabcolsep}{0.1em}
\begin{tabular}{|c|c c c c c||c c c c c|}
\hline
\textbf{Dataset} & \multicolumn{5}{c||}{\textbf{PAC-Bayesian Bound}, $B^{\ast}_{\phi, \text{KL}}$ } & \multicolumn{5}{c|}{\textbf{Average Test Error}, $T_{\phi, \text{KL}}$} \\
\cline{2-11}
 & $B^{\ast}_{\text{lin, KL}}$ & $B^{\ast}_{\text{sq, KL}}$ & $B^{\ast}_{\text{P, KL}}$ & $B^{\ast}_{\text{CH, KL}}$ &$B^{\ast}_{\text{kl, KL}}$ & $T_{\text{lin, KL}}$ & $T_{\text{sq, KL}}$ & $T_{\text{P, KL}}$ & $T_{\text{CH, KL}}$ &$T_{\text{kl, KL}}$
\\ \hline
Spambase & NaN & 0.20046 & 0.17361 & 0.17958 & \textbf{0.15332}\textcolor{magenta}{$\star$} & \textbf{0.15684} & \textbf{0.15392} & \textbf{0.15423} & \textbf{0.15434} & \textbf{0.15487}\textcolor{magenta}{$\star$} \\ \hline
Bupa & 0.27005 & \makecell{0.38167 \\ \textit{0.34547}} & 0.29265 & 0.30537 & \textbf{0.23851}\textcolor{magenta}{$\star$} & \textbf{0.13207} & \makecell{0.145801 \\ \textit{0.14873}} & 0.13631 & 0.13382 & \textbf{0.11998}\textcolor{magenta}{$\star$} \\ \hline
\makecell{Mammographic} & 0.29518 & \makecell{0.34187 \\ \textit{0.31290}}& 0.28790 & 0.29659 & \textbf{0.26063}\textcolor{magenta}{$\star$} & \textbf{0.20462} &\makecell{\textbf{0.21120} \\ \textit{0.21386}} &\textbf{0.20716} & \textbf{0.20628} & \textbf{0.20519}\textcolor{magenta}{$\star$} \\ \hline
Wdbc & 0.20706 & \makecell{0.26000 \\ \textit{0.22122}}& 0.20236 & 0.21646 & \textbf{0.14759}\textcolor{magenta}{$\star$} & \textbf{0.06489} & \makecell{\textbf{0.06901} \\ \textit{0.07052}}& \textbf{0.06650} & \textbf{0.06584} & \textbf{0.06541}\textcolor{magenta}{$\star$}\\ \hline
Banknote & 0.13647 & \makecell{0.13225 \\ \textit{0.10343}}& 0.09538 & 0.10672 & \textbf{0.02051} & 0.00161 & \makecell{0.00561 \\ \textit{0.00592}} & 0.00500 & 0.00469 & \textbf{0.00037} \\ \hline
Mushroom & NaN & 0.06584 & 0.04702 & 0.05399 & \textbf{0.00489} & \textbf{8.92{e-05}} & 0.00066 & 0.00057 & 0.00053 & \textbf{1.39{e-05}}\\ \hline
Ionosphere & 0.20816 & \makecell{0.30151 \\ \textit{0.25884}} & 0.22508 & 0.24011 & \textbf{0.14707}\textcolor{magenta}{$\star$} & \textbf{0.04494} & \makecell{\textbf{0.04781} \\ \textit{0.04899}}& \textbf{0.04393}
& \textbf{0.04553} & \textbf{0.04359}\textcolor{magenta}{$\star$} \\ \hline
Waveform & NaN & 0.12875 & 0.10335 & 0.11103 & \textbf{0.06338} & 0.05847 & \textbf{0.05175} & \textbf{0.05276} & \textbf{0.05345} & 0.05792 \\ \hline
Haberman & \textbf{0.37277} & \makecell{ 0.48385 \\ \textit{0.43977}} & \textbf{0.39769} & 0.41178 & \textbf{0.37998}\textcolor{magenta}{$\star$} & \textbf{0.29157} & \makecell{\textbf{0.29069} \\ \textit{0.29007}} & \textbf{0.29163} & \textbf{0.29162} & \textbf{0.28997}\textcolor{magenta}{$\star$} \\ 
\hline
\end{tabular}
\end{center}
}
\vspace{-3mm}
\caption[PAC-Bayesian bounds and averaged test error rates for $Q^{\ast}_{\phi,\text{KL}}$]{\small \textbf{PAC-Bayesian bounds and averaged test error rates for $Q^{\ast}_{\phi,\text{KL}}$} We compare bound values $B^{\ast}_{\phi,\text{KL}}$ and average test error rates $T_{\phi,\text{KL}}$ of optimal posteriors due to five distance functions, $\phi$:  KL-divergence $kl$, its Pinsker's approximation $\phi_\text{P}$ and a sixth degree polynomial approximation $\phi_\text{CH}$; linear $\phi_{\text{lin}}$ and squared distances $\phi_{\text{sq}}$ for $H=1990$ SVM classifiers. For large sample size ($m \geq 1028$), $\mathcal{I}^K_\text{lin}(m)$ cannot be computed due to storage limitations for floating point numbers and in that case, $B^{\ast}_{\text{lin,KL}}$ is denoted by NaN. $Q^\ast_\text{sq, KL}$ was determined using: $2\sqrt{m}$ (in regular font) and $\mathcal{I}^K_\text{sq}(m)$ (in italicized font). $\mathcal{I}^K_\text{sq}(m)$ cannot be computed for $m \geq 1028$ due to storage limitations. For such cases, we report the values computed using $2\sqrt{m}$ alone. \textcolor{magenta}{$\star$} refers to values obtained using fixed point(FP) equation because \texttt{Ipopt} solver does not converge. Lowest 10\% bound values and test error rates for each dataset are denoted in bold face. $kl$ has the tightest bound and lowest 10\% error rate for most datasets, but is computationally expensive and has multiple local minima. Between $\phi_\text{P}$ and $\phi_\text{CH}$, the latter has lower test error values but a slightly complicated bound evaluation. \deleted{$\phi_\text{sq}$ is ranked lowest on bound values and test error rates, followed by $\phi_\text{CH}$ and $\phi_\text{P}$.} $\phi_\text{sq}$ and $\phi_\text{P}$ are related by a scaling ($\phi_\text{P} = 2\phi_\text{sq}$). $\phi_\text{P}$ provides a lower bound value than that of $\phi_\text{sq}$, but both have comparable test set performances with differences of at most 3\%. $\phi_\text{lin}$ has second lowest bound value for all datasets (except where $m \geq 1028$, namely, Spambase, Mushroom and Waveform, where $B^{\ast}_{\text{lin}, KL}$ cannot be computed) and also has the lowest 10\% test error rates for most datasets. All 5 $\phi$s have lowest 10\% test error values on most datasets considered, except for Bupa dataset and two almost separable datasets, Banknote and Mushroom, where $\phi_\text{lin}$ and $\phi_\text{kl}$ do better.}% than other three $\phi$s.}
\label{tab:phiKLCompare}
%\end{table}
%\vspace{-.5cm}
%\begin{table}[]
{ \setlength{\tabcolsep}{0.1em}
\footnotesize
\begin{tabular}{|l|c|c|c|c|c|c|c|c|c|c|c|}
\hline
\diagbox{\textbf{Dataset}}{\textbf{H}} & \multicolumn{2}{c|}{\textbf{50}} & \multicolumn{2}{c|}{\textbf{200}} & \multicolumn{2}{c|}{\textbf{500}} & \multicolumn{2}{c|}{\textbf{1000}} & \multicolumn{2}{c|}{\textbf{1990}} \\  \cline{2-11}
\makecell{(Validation  \\ set size, $v$)}& $B^{FP}_{\text{kl, KL}}$ &  $B^{solver}_{\text{kl, KL}}$  &  $B^{FP}_{\text{kl, KL}}$ &  $B^{solver}_{\text{kl, KL}}$ & $B^{FP}_{\text{kl, KL}}$ &  $B^{solver}_{\text{kl, KL}}$ & $B^{FP}_{\text{kl, KL}}$ &  $B^{solver}_{\text{kl, KL}}$  & $B^{FP}_{\text{kl, KL}}$ &  $B^{solver}_{\text{kl, KL}}$ \\ \hline
\makecell{Spambase \\ $(v = 1840)$ } & 0.14726 & 0.14726  & 0.14942 & 0.14942  & 0.15157 & 0.27004(\textbf{E}) & 0.15202 & 0.29484(\textbf{E})  & 0.15332 & 0.31452(\textbf{E})
 \\ \hline
 \makecell{Bupa \\ $( v = 138)$ }  & 0.20833 & 0.20833    & 0.22006 & 0.22006   & 0.22750 &  0.43732(\textbf{E})  & 0.23300 & 0.50867(\textbf{E})   & 0.23851 & 0.57682(\textbf{E})
 \\ \hline
% \makecell{Mammographic \\ $( v = 332)$ } & 0.241706 & 0.241680  & 0.249234 & 0.249235 & 0.253854 & 0.253847  & 0.257411 & 0.302582(E)   &0.260632 & 0.335105(E)
% \\ \hline
%\makecell{Wdbc \\ $( v = 227)$ }  & 0.127827 & 0.127827  & 0.134727 & 0.134714 & 0.139659 & 0.139655  & 0.14363 & 0.143656   & 0.147595 & 0.187134(E)
% \\ \hline
%\makecell{Banknote \\ $( v = 549)$ } & 0.015278 & 0.015278 & 0.016358 & 0.016356 & 0.018065 & 0.018065   & 0.232998 & 0.513805(E)  & 0.238509 & 0.573999(E)
% \\ \hline
%\makecell{Mushroom \\ $( v = 2257)$ } & 0.004050 & 0.004050 &  0.004050 & 0.004050 & 0.004150  & 0.004150 &0.004517 &0.004517 &0.004882 &0.004883 
% \\ \hline
%\makecell{Ionosphere \\ $( v = 140)$ } & 0.119248 & 0.122997(M) & 0.129552 & 0.141167(M) & 0.13658 & 0.136579 & 0.141938 & 0.275581(E) & 0.147074 & 0.404999(E) 
% \\ \hline
%\makecell{Waveform \\ $( v = 1323)$ } & 0.058419  & 0.058424 & 0.060210 & 0.060206 & 0.061562 & 0.06157 & 0.062467 & 0.062473 & 0.063376 & 0.063387 
% \\ \hline
%\makecell{Haberman \\ $( v = 122)$ } & 0.342978 & 0.350085(M) & 0.356983 & 0.356982 & 0.366412 & 0.407535(E) & 0.373351 & 0.421606(E) & 0.379982 & 0.427411(E)  \\ \hline
\end{tabular}}
\caption[Bound values for kl-KL case]{\small Comparing bound values due to fixed point solution, $B^{KKT}_{\text{kl, KL}}$, and bound values due to solver output, $B^{solver}_{\text{kl, KL}}$, for bound minimization problem \eqref{eqn:BklKLOP} involving KL-distance function with KL-divergence measure. We observe that the fixed point equation always converges to a solution, even when the \texttt{Ipopt} solver is not able to identify a solution (denoted by `\textbf{E}' (Unknown Error)). Other examples of solver failure are in Table 5 in Suppl. file (eg.  `\textbf{R}' (Restoration Phase Failed) or `\textbf{M}' (Maximum Number of Iterations Exceeded)).}
\label{tab:Bnd.klKL}
\end{table}

\section{Conclusion and Future Directions} \label{secn:conclusion}
We considered the PAC-Bayesian bound minimization problem for a finite classifier set with 5 distance functions. %and determine their optimal posteriors. 
The optimal posterior weights are negative-exponentially decreasing with empirical risk values. For linear distance and uniform prior, weights are negative-exponentially proportional to number of misclassifications. Since some of these minimization problems are non-convex, we proposed fixed point (FP) iterates to identify posteriors with good test error rates. We apply these ideas for choosing SVM regularization parameter via an optimal posterior on the regularization parameter set, yielding a stochastic SVM.

As a part of the future work, we wish to investigate the convergence of FP iterates, and the reason for uniqueness of local minimum for some non-convex cases. For a comparative study, we can consider the PAC-Bayesian counterpart based on R{\'e}nyi divergence between posterior and prior (proposed by \cite{begin2016pac}) for the distance functions considered.

%\acks{We thank the Reviewers for their comments which helped us in improvising this paper.}
%\acks{Puja Sahu is supported by IIT Bombay IRCC financial support. Prof. N. Hemachandra is supported by Govt. of India. We would like to thank the reviewers for their comments which helped us in improvising this paper.}

%\bibliographystyle{apalike}
\bibliography{ACML2019bibDOI}

 \appendix
\section{No unique best parameter}
We are given a dataset and we fix a set of $\lambda$ values to choose from. Let this set be denoted as $\Lambda = \lbrace \lambda_1, \ldots, \lambda_H \rbrace$ where $H$ is the number of parameter values that we consider. To generate the classifiers, we create training samples by partitioning the given dataset. On every sample, we learn SVM classifiers by considering each parameter value in the set $\Lambda$ and then choose the best value $\lambda^\ast(S)$  for a sample $S$ by comparing the 0-1 training errors of the classifiers obtained. To see how these values fare on the scale of generalization performance, we compare their test error rates computed on a common test set. In the adjoining Figure \ref{fig:nobestlam}, we plot the test error rates of the best parameter value $\lambda^\ast(S)$ (with this best value mentioned above the lines representing the error rates) for each sample in the set of samples drawn from a UCI dataset \cite{UCI:2017}.

From Figure \ref{fig:nobestlam}, we observe that the best parameter value is sample dependent and that there is no unique value which is best for almost all the samples (at least 75\% of the samples). However, if we determine the set of $\lambda$ values with lowest $\rho\%$ error rates on each sample, we observe a recurring subset of $\lambda$ values across the training samples. 
%We ran a set of experiments to validate this intuition.
Thus, we have an ensemble of values to pick from as in Table \ref{tab:top30perclambda_mamm}. The PAC-Bayesian approach does such a stochastic selection.

\begin{figure}[ht] 
\centering
\includegraphics[width=0.45\textwidth]{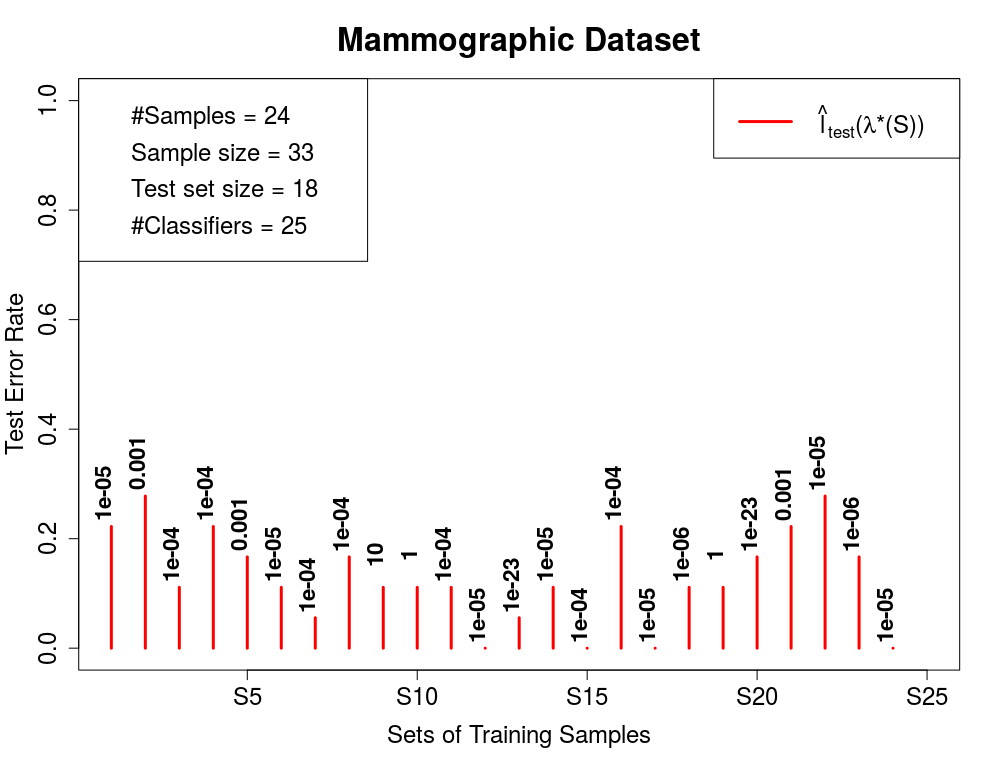} 
\includegraphics[width=0.45\textwidth]{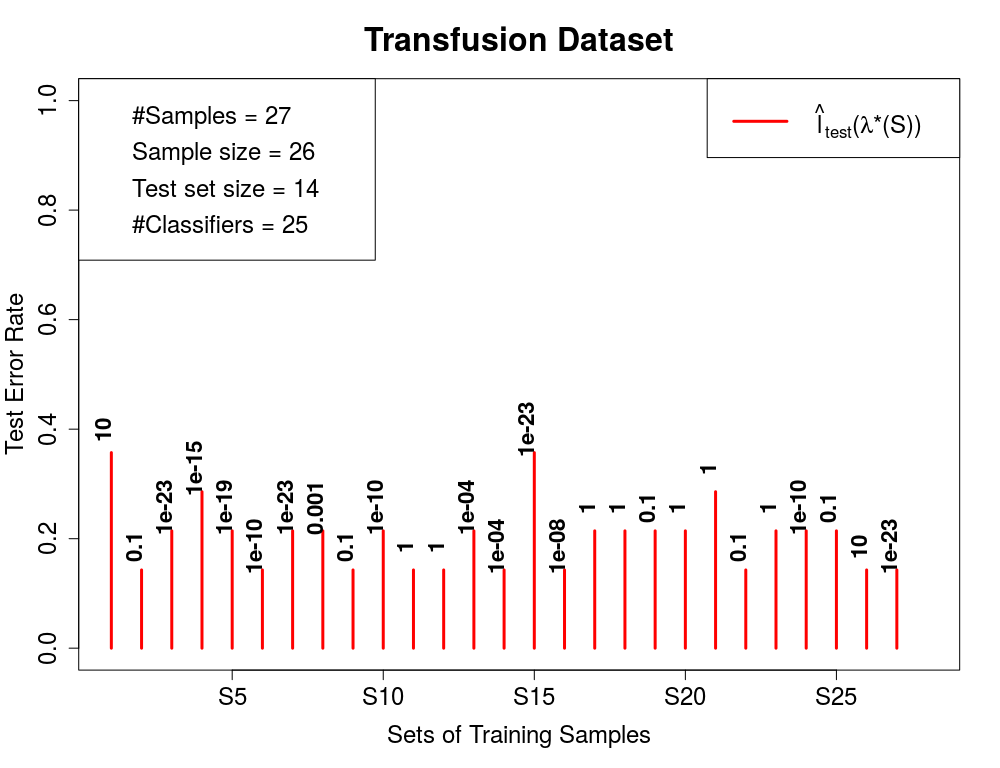}
\caption[No best parameter across samples]{We illustrate that there is no unique value of the regularization parameter $\lambda$ of the SVM algorithm which has the lowest 0-1 training error for almost all the samples considered.  We partition the dataset into training samples of equal size indicated on the horizontal axis of the graph. A total of $H = 25$ values of the regularization parameter ($\Lambda = \lbrace \lambda_i := 10^{(2-i)} \rbrace_{i =1}^{25}$; ranging from $10$ to $10^{-23}$ on a log scale with base 10) were used for obtaining SVMs on each of these training samples. For every sample $S_j$, the best parameter value $\lambda^{\ast}(S_j)$ in the set $\Lambda$ was determined based on the lowest training error obtained on this sample. A test set was set aside beforehand. This common test set was used for evaluating the classifiers generated by different samples, $S_j$s and different parameter values, $\lambda_i$s. We plot the test error rate $\hat{l}_{\text{test}}(\lambda^{\ast}(S_j))$ for each sample $S_j$ along with the best parameter value $\lambda^{\ast}(S_j)$. We can easily see that there is no single value of $\lambda_i$ which serves as best parameter value of most of the samples. \label{fig:nobestlam}}
\end{figure}

\begin{table}[ht]
\centering
\begin{tabular}{c|rrrrrrrr}
  \hline
\makecell{Sample \\Index} & $\lambda^\ast_1$ & $\lambda^\ast_2$ & $\lambda^\ast_3$ & $\lambda^\ast_4$ & $\lambda^\ast_5$ & $\lambda^\ast_6$ & $\lambda^\ast_7$ &$\lambda^\ast_8$ \\ 
  \hline
$S_1$ & \textbf{1e-05} & \textbf{0.1} &   \textbf{1} &  \textbf{10} & \textbf{0.0001} & \textbf{0.001} & \textbf{0.01} & \textbf{1e-23} \\ 
  $S_2$ & \textbf{0.001} & \textbf{0.01} & \textbf{0.1} &   \textbf{1} &  \textbf{10} & \textbf{0.0001} & \textbf{1e-23} & 1e-22 \\ 
  $S_3$ & \textbf{0.0001} & \textbf{0.001} & \textbf{0.01} & \textbf{0.1} &   \textbf{1} &  \textbf{10} & \textbf{1e-05} & \textbf{1e-23} \\ 
  $S_4$ & \textbf{0.0001} & \textbf{0.001} & \textbf{0.01} & \textbf{0.1} &   \textbf{1} &  \textbf{10} & \textbf{1e-23} & 1e-22 \\ 
  $S_5$ & \textbf{0.001} & \textbf{0.01} & \textbf{0.1} & \textbf{0.0001} &   \textbf{1} &  \textbf{10} & 1e-06 & \textbf{1e-05} \\ 
  $S_6$ & \textbf{1e-05} & \textbf{0.001} & \textbf{0.01} & \textbf{0.1} &   \textbf{1} & \textbf{0.0001} &  \textbf{10} & \textbf{1e-23} \\ 
  $S_7$ & \textbf{0.0001} & \textbf{0.001} & \textbf{0.01} & \textbf{0.1} &   \textbf{1} &  \textbf{10} & \textbf{1e-05} & \textbf{1e-23} \\ 
  $S_8$ & \textbf{0.0001} & \textbf{0.001} & \textbf{0.01} & \textbf{0.1} & \textbf{1e-05} &   \textbf{1} &  \textbf{10} & 1e-06 \\ 
  $S_9$ &  \textbf{10} & \textbf{1e-23} & 1e-22 & 1e-21 & 1e-20 & 1e-19 & 1e-18 & 1e-17 \\ 
  $S_{10}$ &   \textbf{1} & \textbf{1e-05} & \textbf{0.01} & \textbf{0.1} &  \textbf{10} & \textbf{0.0001} & \textbf{0.001} & 1e-06 \\ 
  $S_{11}$ & \textbf{0.0001} & \textbf{0.001} & \textbf{0.01} & \textbf{0.1} &   \textbf{1} &  \textbf{10} & \textbf{1e-23} & 1e-22 \\ 
  $S_{12}$ & \textbf{1e-05} & \textbf{0.0001} & \textbf{0.001} & \textbf{0.01} & \textbf{0.1} &   \textbf{1} &  \textbf{10} & \textbf{1e-23} \\ 
  $S_{13}$ & \textbf{1e-23} & 1e-22 & 1e-21 & 1e-20 & 1e-19 & 1e-18 & 1e-17 & 1e-16 \\ 
  $S_{14}$ & \textbf{1e-05} & \textbf{0.0001} & \textbf{0.001} & \textbf{0.01} & \textbf{0.1} &   \textbf{1} &  \textbf{10} & \textbf{1e-23} \\ 
  $S_{15}$ & \textbf{0.0001} & \textbf{0.001} & \textbf{0.01} & \textbf{0.1} &   \textbf{1} &  \textbf{10} & \textbf{1e-05} & \textbf{1e-23} \\ 
  $S_{16}$ & \textbf{0.0001} & \textbf{0.001} & \textbf{0.01} & \textbf{0.1} & \textbf{1e-23} & 1e-22 & 1e-21 & 1e-20 \\ 
  $S_{17}$ & \textbf{1e-05} & \textbf{0.0001} & \textbf{0.001} & \textbf{0.01} & \textbf{0.1} &   \textbf{1} &  \textbf{10} & 1e-06 \\ 
  $S_{18}$ & 1e-06 & \textbf{1e-05} & \textbf{0.0001} & \textbf{0.001} & \textbf{0.01} &   \textbf{1} & \textbf{0.1} &  \textbf{10} \\ 
  $S_{19}$ &   \textbf{1} &  \textbf{10} & \textbf{0.0001} & \textbf{0.001} & \textbf{0.01} & \textbf{0.1} & 1e-06 & \textbf{1e-05} \\ 
  $S_{20}$ & \textbf{1e-23} & 1e-22 & 1e-21 & 1e-20 & 1e-19 & 1e-18 & 1e-17 & 1e-16 \\ 
  $S_{21}$ & \textbf{0.001} & \textbf{0.01} & \textbf{0.1} &   \textbf{1} &  \textbf{10} & \textbf{0.0001} & \textbf{1e-23} & 1e-22 \\ 
  $S_{22}$ & \textbf{1e-05} & \textbf{0.0001} & \textbf{0.001} & \textbf{0.01} & \textbf{0.1} &   \textbf{1} &  \textbf{10} & \textbf{1e-23} \\ 
  $S_{23}$ & 1e-06 & \textbf{1e-05} & \textbf{0.0001} & \textbf{0.001} & \textbf{0.01} & \textbf{0.1} &   \textbf{1} &  \textbf{10} \\ 
  $S_{24}$ & \textbf{1e-05} & \textbf{0.0001} & \textbf{0.001} & \textbf{0.01} & \textbf{0.1} &   \textbf{1} &  \textbf{10} & \textbf{1e-23} \\ 
  \hline
\end{tabular}
\caption[SVM parameter values with lowest training errors across training samples]{\textbf{SVM parameter values with lowest training errors across training samples}. We partition the Mammographic dataset (830 samples, 5 features) into 24 training samples of equal size ($m = 33$). A total of $H = 25$ values of the regularization parameter were used for obtaining SVMs on each of these training samples ($\Lambda = \lbrace \lambda_i := 10^{(2-i)} \rbrace_{i =1}^{25}$; ranging from $10$ to $10^{-23}$ on a log scale with base 10). Each row represents top 30\% ($\approx 8$ values) parameter values ranked by increasing training errors on a training sample from the set of samples used. We observe a few \textit{recurring} values of $\lambda_i$s (represented in boldface) across different sample. This subset of $\Lambda$, namely $\lbrace 1\mathrm{e}{-23}, 0.0001, 0.001, 0.01, 0.1, 1 , 10 \rbrace$, generates classifiers which are among the top performers (on the training error values) for most of the samples (about 21 of the 24 samples), save a few (namely, samples $S_9$, $S_{13}$ and $S_{20}$). \label{tab:top30perclambda_mamm}}
\end{table}

 \section{PAC-Bayesian bound illustration}
 We illustrate PAC-Bayesian bounds with two distance functions -- linear distance and squared distance function. We state the correspoding PAC-Bayesian theorems below:
 
\begin{eqnarray}
%\label{eqn:}
\mathbb{P}_{S_m} \left \lbrace \forall Q \text{ on } \mathcal{H}: \mathbb{E}_Q[l] - \mathbb{E}_Q[\hat{l}] \leq  \frac{KL[Q||P] + \ln\left( \frac{\mathcal{I}^{K}_{{\text{lin}}}(m)}{\delta} \right)}{m} \right \rbrace \geq 1 - \delta. \label{eqn:linKLPACB}\\
\mathbb{P}_{S_m} \left \lbrace  \forall Q \text{ on } \mathcal{H}: (\mathbb{E}_Q[l] - \mathbb{E}_Q[\hat{l}])^2 \leq  \frac{KL[Q||P] + \ln\left( \frac{\mathcal{I}^{K}_{{\text{sq}}}(m)}{\delta} \right)}{m} \right \rbrace \geq 1 - \delta. \label{eqn:sqKLPACB}
\end{eqnarray} 

 \begin{figure}[ht]
    \centering
    \includegraphics[width = 0.8\textwidth]{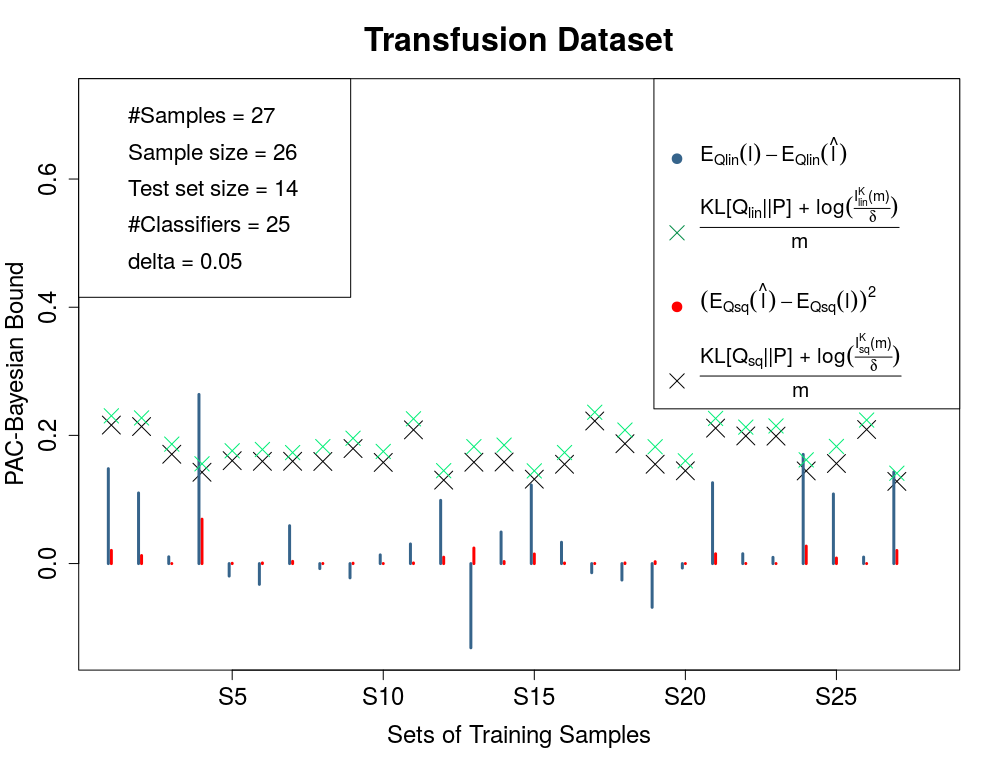}
    \caption{\textbf{Robustness of PAC-Bayesian bound to training sample}. The blue vertical lines correspond to LHS of \eqref{eqn:linKLPACB} and the green cross marks represent the RHS of \eqref{eqn:linKLPACB} for different training samples partitioned from Transfusion dataset. We observe that the RHS bounds the LHS for most of the samples (high probability). The red vertical lines correspond to LHS of \eqref{eqn:sqKLPACB} and the black cross marks represent the RHS of \eqref{eqn:sqKLPACB} for different training samples. In this case, RHS bounds LHS for all the training samples (high probability).}
    \label{fig:samplerobustPACB}
\end{figure}

\section{Optimal posterior, $Q^{\ast}_{\phi, \text{KL}}$,  for uniform prior}
%In the previous subsection, the optimal posterior is identified under the assumption that the entire set $\mathcal{H}$ is the support of posterior $Q$. 
We consider the special case of uniform prior on entire $\mathcal{H}$. We want to identify the optimal posterior $Q^{\ast}_{\phi, \text{KL}}$ with the $H$-dimensional probability simplex as the feasible region. We show below that it is enough to restrict the search space to certain subsets of this simplex. This reduces the computational complexity of the search from exponential scale to linear scale.

\begin{theorem} \label{thm:increasing.subsets.phiKL}
Consider a uniform prior distribution on the set $\mathcal{H}$ of classifiers, and a given set of posterior weights $Q = \lbrace q_j \rbrace_{j = 1}^{H'}$. We have three choices of distance function $\phi = \lbrace \phi_{\text{lin}}, \phi_{\text{sq}}, kl \rbrace$. Then among all subsets $\mathcal{H}' \subset \mathcal{H}$ of size $H'$, the smallest bound value $B_{\phi, \text{KL}}(Q, \mathcal{H}')$ corresponding to the given posterior weights $Q$ is achieved when $\mathcal{H}'$ is the subset formed by the first $H'$ elements of the ordered set of classifiers ranked by non-decreasing empirical risk values, $\hat{l}_1 \leq \hat{l}_2 \leq \ldots \leq \hat{l}_{H}$.
\end{theorem}
\begin{proof}
We first consider the case of linear and squared distance based bounds. Under the given set up, these bound functions are defined as follows:
\begin{align}
B_{\text{lin, KL}}(Q, \mathcal{H}') &:= \sum_{i\in \mathcal{H'}} \hat{l}_i q_i + \frac{\sum\limits_{i \in \mathcal{H}'}q_i \ln q_i + \ln H + \ln\left( \frac{\mathcal{I}^{K}_{{\text{lin}}}(m)}{\delta} \right)}{m}\label{eqn:BlinKL_unifP.genQ} \\
B_{\text{sq, KL}}(Q, \mathcal{H}') &:= \sum_{i\in \mathcal{H'}} \hat{l}_i q_i + \sqrt{\frac{\sum\limits_{i \in \mathcal{H}'}q_i \ln q_i + \ln H + \ln\left( \frac{\mathcal{I}^{K}_{{\text{sq}}}(m)}{\delta} \right)}{m}}. \label{eqn:BsqKL_unifP.genQ}
\end{align} 

For a given set of posterior weights $\lbrace q_j \rbrace_{j = 1}^{H'}$, the second terms of $B_{\text{lin}, \chi^2}(Q, \mathcal{H}')$ and $B_{\text{sq}, \chi^2}(Q, \mathcal{H}')$ are invariant of the support set $\mathcal{H}'$ as long as its cardinality is $H'$. Thus the value of the bound depends on the common first term which is a sum of positive quantities. For given weights $\lbrace q_j \rbrace_{j = 1}^{H'}$, the bounds \eqref{eqn:BlinKL_unifP.genQ} and \eqref{eqn:BsqKL_unifP.genQ} are the smallest when the sum $\sum_{i \in \mathcal{H}'} \hat{l}_i q_i$ is minimized. This will happen when $\mathcal{H}'$ consists of classifiers with smallest $H'$ values in the set $\lbrace \hat{l}_i \rbrace_{i = 1}^{H}$.
%of empirical risks of all the classifiers in $\mathcal{H}$.
Furthermore, if the elements of $\mathcal{H}'$ are ordered by non-decreasing empirical risk values, $\hat{l}_1 \leq \hat{l}_2 \leq \ldots \leq \hat{l}_{H'}$, the posterior weights should be ordered non-increasingly. Hence, the claim of the theorem holds true.

Now, for the KL-divergence as distance function, the bound value, $r$, is the solution to following two equations:
\begin{align}
&kl\left( \sum_{i \in \mathcal{H}'} \hat{l}_i q_i , r \right) = \sum_{i\in \mathcal{H'}} \hat{l}_i q_i + \frac{\sum\limits_{i \in \mathcal{H}'}q_i \ln q_i + \ln H + \ln\left( \frac{2\sqrt{m}}{\delta} \right)}{m} \label{eqn:klDCcons.unifP.genQ}\\
& r \geq \sum_{i \in \mathcal{H}'} \hat{l}_iq_i \label{eqn:r>EQl.unifP.genQ}
\end{align}
The right hand side term of \eqref{eqn:klDCcons.unifP.genQ} is invariant of support $\mathcal{H}'$ as long as it is of size $H'$. Let $\hat{L} := \sum_{i \in \mathcal{H}'} \hat{l}_iq_i$, then \eqref{eqn:klDCcons.unifP.genQ} is an implicit function of variables $\hat{L}$  and $r$. Using implicit function theorem, we have
\begin{equation}
\frac{dr}{d\hat{L}} = \frac{- \partial kl/\partial \hat{L}}{ \partial kl/\partial r} = \frac{\ln \frac{\hat{L}}{r} - \ln \frac{1 - \hat{L}}{1 - r}}{\frac{\hat{L}}{r} - \frac{1 - \hat{L}}{1 - r}}
\end{equation}
Using \eqref{eqn:r>EQl.unifP.genQ} and strict monotonicity of natural logarithm function, we can claim that $\frac{dr}{d\hat{L}} > 0$. That is, the bound $r$ is a strictly increasing function of $\hat{L} := \sum_{i \in \mathcal{H}'} \hat{l}_iq_i$ under the given set up. To find the least $r$ for a given $Q(\mathcal{H}') = \lbrace q_j \rbrace_{j = 1}^{H'}$, we need to find the least $\sum_{i \in \mathcal{H}'} \hat{l}_iq_i$ on all possible subsets $\mathcal{H}'$. This happens when $\mathcal{H}'$ is the subset formed by the first ordered $H'$ elements $\hat{l}_1 \leq \hat{l}_2 \leq \ldots \leq \hat{l}_{H'}$. Hence proved.
\end{proof}

\begin{cor}
As a consequence of the above Theorem \ref{thm:increasing.subsets.phiKL}, for determining the (globally) optimal posterior $Q^{\ast}_{\phi, \text{KL}}$, it is sufficient to compare the bound values corresponding to the best  posteriors on ordered subsets of $\mathcal{H}$, ranked by non-decreasing $\hat{l}_i$ values. These ordered subsets can be uniquely identified by their size. An ordered subset of size 1 is $\lbrace \hat{l}_1 \rbrace$, of size 2 is $\lbrace \hat{l}_1, \hat{l}_2 \rbrace$ and so on. Thus there exists an isomorphism between the set $\lbrace 1, \ldots, H \rbrace$ (which denote the subset size) and the family of ordered increasing subsets of $\mathcal{H}$.
\end{cor}

\section{PAC-Bayesian bound with KL-divergence as distance function}

\subsection{The posterior based on fixed point solution, $Q^{FP}_\text{kl,KL}$}
\begin{theorem} \label{thm:qFP_klKL}
The bound minimization problem (6) (in paper) for the bound $B_{\text{kl, KL}}(Q)$ has a stationary point $Q^{\text{FP}}_\text{kl,KL}$ which can be obtained as the solution to the following fixed point equation:
\begin{equation}
q_i = p_i \exp \left \lbrace \sum_{i=1}^{H}q_i\ln \frac{q_i}{p_i} - m\left(\sum_{i=1}^{H}\hat{l}_iq_i - \hat{l}_i \right)\left[ \ln\left( \frac{(1 - r) \sum_{i = 1}^H \hat{l}_i q_i}{r(1-\sum_{i = 1}^H \hat{l}_i q_i)}\right) \right] \right \rbrace \; \forall i = 1, \ldots, H
\label{eqn:qFP_klKL}
\end{equation}
where $r$ is the solution to (6b) and (6c) in paper for a given $Q = (q_1, \ldots, q_H)$.
\end{theorem}
The Lagrangian function for (6) (in paper) can be written as follows:
\begin{multline}
\mathcal{L}_{\text{kl, KL}}(Q,r,\beta_0, \beta_1, \mu_0,\mu_i) = r - \beta_0 \left[ \left(\sum_{i = 1}^H \hat{l}_iq_i \right) \ln\left( \frac{\sum_{i = 1}^H \hat{l}_iq_i}{r}\right)  \right.\\
\left. + \left(1 - \sum_{i = 1}^H \hat{l}_iq_i \right) \ln\left( \frac{ 1  - \sum_{i = 1}^H \hat{l}_iq_i}{1 - r}\right) - \frac{\left(\sum_{i = 1}^H q_i \ln \frac{q_i}{p_i} + \ln\frac{2\sqrt{m}}{\delta} \right)}{m} \right]  \\
 -\beta_1 \left( r - \sum_{i = 1}^H \hat{l}_iq_i \right) - \mu_0 \left(\sum_{i=1}^H q_i -1 \right) - \sum_{i = 1}^H \mu_i q_i
\end{multline}
Note that we have a strict inequality constraint in our optimization model (6) (in paper), namely, $r > \sum_{i = 1}^H \hat{l}_iq_i$. Hence, due to complementary slackness conditions for a stationary point, the associated Lagrange multiplier $\beta_1$ should vanish at optimality, i.e., $\beta_1 = 0$. 
% Therefore the above Lagrangian becomes:
% \begin{multline}
% \mathcal{L}_{\text{kl, KL}}(Q,r,\beta_0, \beta_1, \mu_0,\mu_i) = r - \beta_0 \left[ \left(\sum_{i = 1}^H \hat{l}_iq_i \right) \ln\left( \frac{\sum_{i = 1}^H \hat{l}_iq_i}{r}\right)  \right.\\
% \left. + \left(1 - \sum_{i = 1}^H \hat{l}_iq_i \right) \ln\left( \frac{ 1  - \sum_{i = 1}^H \hat{l}_iq_i}{1 - r}\right) - \frac{\left(\sum_{i = 1}^H q_i \ln \frac{q_i}{p_i} + \ln\frac{2\sqrt{m}}{\delta} \right)}{m} \right]  \\
%   - \mu_0 \left(\sum_{i=1}^H q_i -1 \right) - \sum_{i = 1}^H \mu_i q_i
% \end{multline}

Differentiating $\mathcal{L}_{\text{kl, KL}}$ with respect to primal variables $r$ and $q_i$s, and also with respect to dual variable $\mu_0$, we get:
\begin{align}
\frac{\partial \mathcal{L}_{\text{kl, KL}}}{\partial r} &= 1- \beta_0\left[- \left(\frac{\sum_{i=1}^{H}\hat{l}_iq_i}{r}\right) + \left( \frac{1-\sum_{i = 1}^H \hat{l}_i q_i}{1-r}\right)\right] \label{eqn:deriver_klKL}\\
\frac{\partial \mathcal{L}_{\text{kl, KL}}}{\partial q_i} 
%&= -\beta_0 \left[ \hat{l}_i \ln\left( \frac{\sum_{i = 1}^H \hat{l}_i q_i}{r}\right) + \cancel{\hat{l}_i} - \hat{l}_i \ln \left( \frac{1-\sum_{i = 1}^H \hat{l}_i q_i}{1-r}\right) - \cancel{\hat{l}_i} - \left( \frac{1+\ln \frac{q_i}{p_i}}{m}\right) \right] \nonumber \\ &\hspace{8cm}  - \mu_0 -\mu_i \quad \forall i = 1, \ldots, H \nonumber\\[3mm]
&= -\beta_0 \left\{ \hat{l}_i \left[ \ln\left( \frac{\sum_{i = 1}^H \hat{l}_i q_i}{r} \right) - \ln \left( \frac{1-\sum_{i = 1}^H \hat{l}_i q_i}{1-r}\right)\right] - \left( \frac{1+\ln \frac{q_i}{p_i}}{m}\right) \right\} \nonumber \\
&\hspace{8cm} -\mu_0 -\mu_i \; \forall i = 1, \ldots, H \label{eqn:deriveq_klKL}\\
\frac{\partial \mathcal{L}_{\text{kl, KL}}}{\partial \mu_0} &= \sum_{i=1}^{H}q_i -1 \label{eqn:derivmu0_klKL}
\end{align}
At an optimal solution, these derivatives should be set to zero. Let us first consider the derivative \eqref{eqn:deriver_klKL} and set it to zero. We get:
\begin{align}
% &\dfrac{\partial \mathcal{L}_{\text{kl, KL}}}{\partial r} = 0 \nonumber\\
% \Rightarrow &1- \beta_0\left[- \left(\frac{\sum_{i=1}^{H}\hat{l}_iq_i}{r}\right) + \left( \frac{1-\sum_{i = 1}^H \hat{l}_i q_i}{1-r}\right)\right] = 0 \nonumber\\[3mm]
% \Rightarrow & \beta_0\left[ \frac{-\sum\limits_{i=1}^{H}\hat{l}_iq_i + \cancel{r \left(\sum\limits_{i=1}^{H}q_i\hat{l}_i \right)}  + r - \cancel{r \left(\sum\limits_{i=1}^{H}q_i\hat{l}_i \right)} }{r(1-r)} \right] =1 \nonumber\\
%\Rightarrow &
\beta_0 = \frac{r(1-r)}{r - \sum\limits_{i=1}^{H}\hat{l}_iq_i} >0 \label{eqn:beta0LagrangeklKL}
\end{align}
The denominator in above is strictly positive since $r>\sum\limits_{i=1}^{H}q_i\hat{l}_i$ . The inequality constraint also implies that $r \in (0, 1)$, which means that the numerator term is also strictly positive. Hence, we have $\beta_0 > 0$ which is a feasible value for the Lagrange parameter.

Next consider the derivative \eqref{eqn:deriveq_klKL} of the Lagrange $\mathcal{L}_{\text{kl, KL}}$. We multiply it with $q_i$ and set it zero to get:
\begin{align}
&q_i\frac{\partial \mathcal{L}_{\text{kl, KL}}}{\partial q_i} = 0 \quad \forall i = 1, \ldots, H \nonumber\\
\Rightarrow & -\beta_0 \left\{ \hat{l}_iq_i \left[ \ln\left( \frac{\sum_{i = 1}^H \hat{l}_i q_i}{r}\right) - \ln \left( \frac{1-\sum_{i = 1}^H \hat{l}_i q_i}{1-r}\right)\right] - \left( \frac{q_i +q_i\ln \frac{q_i}{p_i}}{m}\right) \right\}  - \mu_0q_i -\mu_iq_i =0 \label{eqn:qmultiplyLagrangederive}
\end{align}
where $\mu_iq_i = 0$ due to complementary slackness conditions, since $\mu_i$ is the Lagrange multiplier for the constraint $q_i \geq 0$.

We assume that $q_i > 0$ for all $i = 1, \ldots, H$, since otherwise $\ln q_i = \ln(0)$ is undefined. Even if we use fact that $\lim_{x \rightarrow 0^{+}} \ln x = - \infty$, the KKT condition will mean that the dual variable $\mu_i$ is infeasible. Therefore our assumption holds true for a stationary point.
Due to complementary slackness conditions, we have $\mu_i q_i = 0$ which implies $\mu_i = 0$ since $q^{\ast}_i > 0$ for all $i = 1, \ldots, H$. 

Summing \eqref{eqn:qmultiplyLagrangederive} over all $i=1,\ldots,H$, we get
\begin{multline*}
\sum\limits_{i=1}^{H}q_i\frac{\partial \mathcal{L}_{\text{kl, KL}}}{\partial q_i} = -\beta_0 \left\{ \sum\limits_{i=1}^{H}q_i\hat{l}_i \left[ \ln\left( \frac{\sum_{i = 1}^H \hat{l}_i q_i}{r}\right) - \ln \left( \frac{1-\sum_{i = 1}^H \hat{l}_i q_i}{1-r}\right)\right] \right. \\
\left. - \left( \frac{\sum\limits_{i=1}^{H}q_i +\sum\limits_{i=1}^{H}q_i\ln \frac{q_i}{p_i}}{m}\right) \right\} - \sum\limits_{i=1}^{H}q_i\mu_0 -0 =0
\end{multline*}
Since $\sum\limits_{i=1}^{H}q_i=1$, solving the above equation for $\mu_0$,we get,
\begin{eqnarray*}
\mu_0 = -\beta_0 \left\{ \sum\limits_{i=1}^{H}q_i\hat{l}_i \left[ \ln\left( \frac{\sum_{i = 1}^H \hat{l}_i q_i}{r}\right) - \ln \left( \frac{1-\sum_{i = 1}^H \hat{l}_i q_i}{1-r}\right)\right] - \left( \frac{1 +\sum\limits_{i=1}^{H}q_i\ln \frac{q_i}{p_i}}{m}\right) \right\}
\end{eqnarray*}
If $q_i=0$ for some $i \in [H]$, then by setting $\frac{\partial \mathcal{L}_{\text{kl, KL}}}{\partial q_i} =0$, we get
\begin{align*}
&-\beta_0  \left\{ \hat{l}_i \left[ \ln\left( \frac{\sum_{i = 1}^H \hat{l}_i q_i}{r}\right) - \ln \left( \frac{1-\sum_{i = 1}^H \hat{l}_i q_i}{1-r}\right)\right] - \left( \frac{1+\ln \frac{q_i}{p_i}}{m}\right) \right\} - \mu_0 - \mu_i =0 \\
\Rightarrow &\mu_i = -\beta_0 \left\{ \hat{l}_i \left[ \ln\left( \frac{\sum_{i = 1}^H \hat{l}_i q_i}{r}\right) - \ln \left( \frac{1-\sum_{i = 1}^H \hat{l}_i q_i}{1-r}\right)\right] - \left( \frac{1-\infty}{m}\right) \right\} - \mu_0 \\
\Rightarrow &\mu_i = -\infty \qquad (\text{since } \ln q_i = \ln (0) = -\infty)
\end{align*}
Therefore, $\mu_i =-\infty$. This implies that whenever $q_i = 0$, the corresponding multiplier $\mu_i$ is dual infeasible. Hence $q_i>0 ~~\forall i = 1, \ldots, H$. And by complementary slackness conditions, $\mu_i=0 \; \forall i = 1, \ldots, H$. Using this we can simplify $\frac{\partial \mathcal{L}_{\text{kl, KL}}}{\partial q_i} = 0$ to get the following fixed point equation:

% \begin{eqnarray*}
% & & -\beta_0 \left\{ \hat{l}_i \left[ \ln\left( \frac{\sum_{i = 1}^H \hat{l}_i q_i}{r}\right) - \ln \left( \frac{1-\sum_{i = 1}^H \hat{l}_i q_i}{1-r}\right)\right] - \left( \frac{1+\ln \frac{q_i}{p_i}}{m}\right) \right\} - \mu_0 = 0 \; \forall i = 1, \ldots, H\\
%  & \Rightarrow & -\beta_0 \left\{ \hat{l}_i \left[ \ln\left( \frac{\sum_{i = 1}^H \hat{l}_i q_i}{r}\right) - \ln \left( \frac{1-\sum_{i = 1}^H \hat{l}_i q_i}{1-r}\right)\right] - \left( \frac{1+\ln \frac{q_i}{p_i}}{m}\right) \right\}  \\
% & & +\beta_0 \left\{ \sum\limits_{i=1}^{H}q_i\hat{l}_i \left[ \ln\left( \frac{\sum_{i = 1}^H \hat{l}_i q_i}{r}\right) - \ln \left( \frac{1-\sum_{i = 1}^H \hat{l}_i q_i}{1-r}\right)\right] - \left( \frac{1 +\sum\limits_{i=1}^{H}q_i\ln \frac{q_i}{p_i}}{m}\right) \right\} = 0 
% \end{eqnarray*}

 \begin{eqnarray*}
% &\Rightarrow  & \beta_0 \left\{ \left(\sum\limits_{i=1}^{H}q_i\hat{l}_i - \hat{l}_i \right)\left[ \ln\left( \frac{\sum_{i = 1}^H \hat{l}_i q_i}{r}\right) - \ln \left( \frac{1-\sum_{i = 1}^H \hat{l}_i q_i}{1-r}\right)\right] \right. \\
% & & \hspace{7cm}\left. - \cancel{\frac{1}{m}} - \frac{\sum\limits_{i=1}^{H}q_i\ln \frac{q_i}{p_i}}{m} + \cancel{\frac{1}{m}} + \frac{\ln \frac{q_i}{p_i}}{m} \right\} =0\\
% & \Rightarrow & \left(\sum_{i=1}^{H}q_i\hat{l}_i - \hat{l}_i \right)\left[ \ln\left( \frac{\sum_{i = 1}^H \hat{l}_i q_i}{r}\right) - \ln \left( \frac{1-\sum_{i = 1}^H \hat{l}_i q_i}{1-r}\right)\right] - \frac{\sum_{i=1}^{H}q_i\ln \frac{q_i}{p_i}}{m} + \frac{\ln \frac{q_i}{p_i}}{m} =0 \\
 q^{FP}_{i, \text{kl, KL}} = p_i \exp \left \lbrace \sum_{i=1}^{H}q_i\ln \frac{q_i}{p_i} - m\left(\sum_{i=1}^{H}\hat{l}_iq_i - \hat{l}_i \right)\left[ \ln\left( \frac{(1 - r) \sum_{i = 1}^H \hat{l}_i q_i}{r(1-\sum_{i = 1}^H \hat{l}_i q_i)}\right) \right] \right \rbrace \; \forall i = 1, \ldots, H
\end{eqnarray*}
The above fixed point equation in variables $q_i$s will result in a feasible stationary point for the bound minimization problem (6) (in paper)

%\subsubsection{Algorithm for finding the fixed point solution, \texorpdfstring{$Q^{FP}_{\text{kl, KL}}$}{QFP{kl,KL}}}

\begin{algorithm2e} 
\DontPrintSemicolon
\KwIn{$\hat{p} \in [0, 1], x_{\delta, m} > 0, \texttt{tol}>0, \epsilon > 0$}
\KwOut{Roots: $\underline{q}_{KL}$ and $\overline{q}_{KL}$}
$q_{\text{low}} \gets \epsilon$\;
$q_{\text{high}} \gets  1 - \epsilon$ \;
$Z_{\text{low}} \gets \hat{p} \ln \left(\frac{\hat{p}}{q_{\text{low}}} \right) + (1-\hat{p}) \ln \left(\frac{1-\hat{p}}{1-q_{\text{low}}} \right)$ \;
$Z_{\text{high}} \gets \hat{p} \ln \left(\frac{\hat{p}}{q_{\text{high}}} \right) + (1-\hat{p}) \ln \left(\frac{1-\hat{p}}{1-q_{\text{high}}} \right)$ \;
\Fn{BisectKL($q_0, p, x_{\delta, m}$)}{
$p_0 \gets p$\;
$q_{\text{mid}} \gets \frac{p_0 + q_0}{2}$ \;
$Z_{q_{\text{mid}}} \gets p \ln \left(\frac{p}{q_{\text{mid}}} \right) + (1-p) \ln \left(\frac{1-p}{1-q_{\text{mid}}} \right)$ \;
%\nlset{BM}
 \While{$|{Z_{q_{\text{mid}}} - x_{\delta, m}}|> \texttt{tol}$ \label{step:while}}{
  \eIf{$Z_{q_{\text{mid}}} > x_{\delta, m}$}{
 $q_0 \gets q_{\text{mid}}$ \; }{
 $p_0 \gets q_{\text{mid}}$ \;}
 
$q_{\text{mid}} \gets \frac{p_0 + q_0}{2}$ \;
\If{$|{q_{\text{mid}} - q_0| \label{step:breakKLalgo}} < \epsilon$}{\KwBreak\;} 

 $Z_{q_{\text{mid}}} \gets p \ln \left(\frac{p}{q_{\text{mid}}} \right) + (1-p) \ln \left(\frac{1-p}{1-q_{\text{mid}}} \right)$ \;
}
\Return{$q_{\text{mid}}$}\;
}
\uIf{$\hat{p} = 0$}
	{$\underline{q}_{KL} \gets 0$ \;
	 $\overline{q}_{KL} \gets 1 - e^{-x_{\delta, m}}$ \;}
	\uElseIf{$\hat{p} = 1$}
	{$\underline{q}_{KL} \gets e^{-x_{\delta, m}}$ \;
	 $\overline{q}_{KL} \gets 1$ \;}
	 \uElseIf{$x_{\delta, m} \geq \max \lbrace Z_{\text{low}}, Z_{\text{high}} \rbrace$}{\KwBreak\;}
	\uElseIf{$x_{\delta, m} < \min \lbrace Z_{low}, Z_{high} \rbrace$}{
$\underline{q}_{KL} \gets$ BisectKL($q_{\text{low}}, \hat{p}, x_{\delta, m}$) \;
$\overline{q}_{KL} \gets$  BisectKL($q_{\text{high}}, \hat{p}, x_{\delta, m}$) \;
}
\uElseIf{$Z_{low} \leq x_{\delta, m} < Z_{high}$ }{
$\overline{q}_{KL} \gets$  BisectKL($q_{\text{high}}, \hat{p}, x_{\delta, m}$) \;}
\Else
{$\underline{q}_{KL} \gets$ BisectKL($q_{\text{low}}, \hat{p}, x_{\delta, m}$) \;}
\Return{$\underline{q}_{KL}, \overline{q}_{KL}$}
\caption{{\sc KLRoots} \cite{PACBIntervals}}\label{algo:KLroot}
\end{algorithm2e}

\subsubsection{Special Case: Optimal posterior when all \texorpdfstring{$\hat{l}_i$s}{hat l i's} are same}

\begin{lem}
When all the classifiers have same empirical risk (all $\hat{l}_i$s are same), the optimal posterior for the bound minimization problem (6) (in paper) is $Q \equiv P$.
\end{lem}
\begin{proof}
When all $\hat{l}_i$s are same, the averaged empirical risk, $\mathbb{E}_Q \left[ \hat{l} \right]$ is independent of the posterior $Q$ which it is averaged over, i.e.,
\begin{equation*}
\mathbb{E}_Q \left[ \hat{l} \right] = \sum_{i = 1}^H \hat{l}_iq_i =\hat{l}_1 \left( \sum_{i = 1}^H q_i \right) = \hat{l}_1 \quad \forall Q \in \Delta^H,
\end{equation*}
where $\Delta^{H} := \lbrace (q_1, \ldots, q_H) \vert \sum_{i = 1}^H q_i = 1 , q_i \geq 0 \; \forall i = 1, \ldots, H \rbrace$ represents the H-dimensional simplex. We have assumed that $\hat{l}_i = \hat{l}_1$ for all $i = 1, \ldots, H$, without loss of generality.

Hence, the bound minimization problem (6) (in paper) becomes:
\begin{subequations} \label{eqn:BklKLOP_l1}
\begin{align} 
&\min_{q_1, \ldots, q_H, r} r\\
\text{s.t.} \quad & \hat{l}_1 \ln\left( \frac{ \hat{l}_1}{r}\right) + \left(1 - \hat{l}_1 \right) \ln\left( \frac{ 1  - \hat{l}_1}{1 - r}\right) = \frac{\sum_{i = 1}^H q_i \ln \frac{q_i}{p_i} + \ln\frac{2\sqrt{m}}{\delta}}{m} \label{eqn:klKLcons_l1}\\
& r \geq \hat{l}_1 \label{eqn:upperklroot}\\
& \sum_{i = 1}^H q_i = 1 \\
& q_i \geq 0 \; \forall i = 1, \ldots, H 
\end{align} 
\end{subequations}
Here, $\hat{l}_1$ is the model parameter. The only constraint which combines the variables $r$ and $(q_1, \ldots, q_H)$ is \eqref{eqn:klKLcons_l1}, with the special structure that the left hand side is a function of $r$ alone, whereas the right hand side is a function of $(q_1, \ldots, q_H)$ alone. Note that, the RHS of \eqref{eqn:klKLcons_l1} is strictly positive for any $Q = (q_1, \ldots, q_H)$ for parameters $m \geq 2$ and $\delta \in (0, 1)$. Therefore, the roots of the LHS function will be away from $\hat{l}_1$, which implies that the inequality in \eqref{eqn:upperklroot} will always be strict.

Let us consider the nature of the function on the LHS of \eqref{eqn:klKLcons_l1}, $kl_{\hat{l}_1} \left( r \right) := \hat{l}_1 \ln\left( \frac{ \hat{l}_1}{r}\right) + \left(1 - \hat{l}_1 \right) \ln\left( \frac{ 1  - \hat{l}_1}{1 - r}\right)$ with respect to the variable $r \in [0, 1]$ for a given value of the parameter $\hat{l}_1 \in [0, 1)$. Differentiating $kl_{\hat{l}_1} \left( r \right)$ with respect to $r$, we have:
\begin{align*}
\frac{\partial kl_{\hat{l}_1}}{\partial r} &= \frac{\partial}{\partial r} \left[ \hat{l}_1 \ln\left( \frac{ \hat{l}_1}{r}\right) + \left(1 - \hat{l}_1 \right) \ln\left( \frac{ 1  - \hat{l}_1}{1 - r}\right) \right] \\
&= - \frac{\hat{l}_1}{r} + \frac{1 - \hat{l}_1}{1 - r} \\
&= \frac{r - \hat{l}_1}{r(1 - r)}.
\end{align*}
The denominator in the above is strictly positive, since the derivative is defined only for $r \in (0, 1)$. The monotone nature of the function $kl_{\hat{l}_1}$ depends on the sign of $r - \hat{l}_1$. Thus, $kl_{\hat{l}_1}(r)$ is strictly increasing when $r > \hat{l}_1$, and is strictly decreasing when $r < \hat{l}_1$.

Using this fact, minimizing $r$ on the feasible region of our optimization problem \eqref{eqn:BklKLOP_l1} is equivalent to minimizing $kl_{\hat{l}_1}(r)$ on $r \geq \hat{l}_1$. By the restriction imposed by the constraint \eqref{eqn:klKLcons_l1}, this is equivalent to minimizing $KL[Q \Vert P] = \sum_{i = 1}^H q_i \ln \frac{q_i}{p_i}$ on the simplex $\Delta^H$. We know that $\min_{Q \in \Delta^H} KL[Q \Vert P] = 0$ when $Q \equiv P$. Hence $Q \equiv P$ is the minimizer for the bound minimization problem \eqref{eqn:BklKLOP_l1} which refers to the case when all the classifiers have same empirical risk.
\end{proof}

\subsection{Convex-concave procedure for finding a local solution for minimization of $B_{\text{kl, KL}}(Q)$} \label{secn:klKLCCP}
We have seen that our optimization problem (7) (in the main paper) for finding the bound $B_{\text{kl, KL}}(Q)$ consists of a linear objective function and linear constraints, except for the constraint \eqref{eqn:klKLDCcons}, which takes the form:
\begin{eqnarray}
&kl\left(\mathbb{E}_Q [\hat{l}], r \right) = \frac{KL[Q||P] + \ln\left( \frac{2 \sqrt{m}}{\delta} \right)}{m} \label{eqn:klKLDCcons}\\
&\Leftrightarrow \left(\sum_{i = 1}^H \hat{l}_iq_i \right) \ln\left( \frac{\sum\limits_{i = 1}^H \hat{l}_iq_i}{r}\right) + \left(1 - \sum\limits_{i = 1}^H \hat{l}_iq_i \right) \ln\left( \frac{ 1  - \sum\limits_{i = 1}^H \hat{l}_iq_i}{1 - r}\right) = \frac{\sum\limits_{i = 1}^H q_i \ln \frac{q_i}{p_i} + \ln\frac{2\sqrt{m}}{\delta}}{m}
\end{eqnarray}
We know that $KL[Q||P]$ is jointly convex in both its arguments \cite{VanErvenHarremoes2014}. The right hand side of the constraint is a positive affine transformation of $KL[Q||P]$ for given system parameters $m$ and $\delta$, and hence again a convex function. The left hand side is a composition of two functions: $\mathbb{E}_Q[\hat{l}]$ (a linear function) and $kl(p, q)$ (a jointly convex function). The superposition of a convex function and an affine mapping is convex, provided that it is finite at least at one point \cite{anatoli2015notesconvexity, boyd2004convex}. Hence, it is established that $kl\left(\mathbb{E}_Q [\hat{l}], r \right)$ is convex in its arguments $(Q, r)$. This implies that the constraint \eqref{eqn:klKLDCcons} is a difference of convex (DC) function and the associated optimization problem ((7) in main paper) is a DC program.

In our bound minimization problem ((7) in main paper) for $B_{\text{kl, KL}}(Q)$, the DC constraint \eqref{eqn:klKLDCcons} is an equality constraint. We need to write it as a set of two inequality constraints to be able to use the linear approximation via supporting hyperplane as described above. Reformulating the original problem ((7) in main paper) in terms of all inequality constraints of the form $f(x) - g(x) \leq 0$, we have:

\begin{subequations} \label{eqn:BklKLOP.preCCP}
\begin{align} 
&\min_{q_1, \ldots, q_H, r} r\\
\hspace{-5mm}\text{s.t.} & \left(\sum_{i = 1}^H \hat{l}_iq_i \right) \ln\left( \tfrac{\sum\limits_{i = 1}^H \hat{l}_iq_i}{r}\right) + \left(1 - \sum\limits_{i = 1}^H \hat{l}_iq_i \right) \ln\left( \tfrac{ 1  - \sum\limits_{i = 1}^H \hat{l}_iq_i}{1 - r}\right) - \tfrac{\sum\limits_{i = 1}^H q_i \ln \tfrac{q_i}{p_i} + \ln\frac{2\sqrt{m}}{\delta}}{m} \leq 0 \label{eqn:klKLDCcons.preCCP}\\ 
&\tfrac{\sum\limits_{i = 1}^H q_i \ln \tfrac{q_i}{p_i} + \ln\frac{2\sqrt{m}}{\delta}}{m} - \left(\sum_{i = 1}^H \hat{l}_iq_i \right) \ln\left( \tfrac{\sum\limits_{i = 1}^H \hat{l}_iq_i}{r}\right) + \left(1 - \sum\limits_{i = 1}^H \hat{l}_iq_i \right) \ln\left( \tfrac{ 1  - \sum\limits_{i = 1}^H \hat{l}_iq_i}{1 - r}\right) \leq 0 \label{eqn:klKLDCcons.reverse.preCCP}\\
& \sum_{i = 1}^H \hat{l}_iq_i - r \leq 0 \label{eqn:r>EQlhat.klKL.preCCP}\\
& \sum_{i = 1}^H q_i = 1 \\
& -q_i \leq 0 \; \forall i = 1, \ldots, H 
\end{align} 
\end{subequations}

To apply the convex-concave procedure (CCP), we determine the approximations to the DC functions \eqref{eqn:klKLDCcons.preCCP} and  \eqref{eqn:klKLDCcons.reverse.preCCP}, at a point $(Q^0, r^0)$ which is feasible to \eqref{eqn:BklKLOP.preCCP}, and equivalently to (7) in the main paper. Let $\widehat{kK}_1((Q,r); (Q^0, r^0))$ denote the linear under-approximation to the function $kK_1(Q, r) := \tfrac{\sum\limits_{i = 1}^H q_i \ln \tfrac{q_i}{p_i} + \ln\frac{2\sqrt{m}}{\delta}}{m}$ in \eqref{eqn:klKLDCcons.preCCP} at $(Q^0, r^0)$.

\begin{align}
\widehat{kK}_1((Q,r); (Q^0, r^0)) &:= kK_1(Q^0, r^0) + \langle \nabla kK_1(Q^0, r^0), \left((Q - Q^0), (r  - r^0)\right) \rangle \nonumber\\
&= \frac{\sum\limits_{i = 1}^H q^0_i \ln \tfrac{q^0_i}{p_i} + \ln\frac{2\sqrt{m}}{\delta}}{m} +  \left( \sum_{i = 1}^{H} \frac{\partial  kK_1 }{\partial q_i} \bigg\vert_{q_i = q^0_i} \cdot (q_i - q^0_i)  \right) + 0 \cdot (r  - r^0) \nonumber\\
&= \frac{\sum\limits_{i = 1}^H q^0_i \ln \tfrac{q^0_i}{p_i} + \ln\frac{2\sqrt{m}}{\delta}}{m} + \left( \sum_{i = 1}^{H} \tfrac{1}{m} \left(1 + \ln\frac{q^0_i}{p_i} \right) (q_i - q^0_i) \right) + 0 \nonumber\\
&= \cancel{\frac{\sum\limits_{i = 1}^H q^0_i \ln \tfrac{q^0_i}{p_i} }{m}}
 + \frac{ \ln\frac{2\sqrt{m}}{\delta}}{m}  + \left( \bcancel{\sum_{i = 1}^{H} q_i} - \bcancel{\sum_{i = 1}^{H} q^0_i} \right) + \frac{\sum\limits_{i = 1}^{H} q_i\ln\frac{q^0_i}{p_i}}{m} - \cancel{\frac{\sum\limits_{i = 1}^{H} q^0_i\ln\frac{q^0_i}{p_i}}{m}} \nonumber\\
&= \frac{\sum\limits_{i = 1}^H q_i \ln \tfrac{q^0_i}{p_i} + \ln\frac{2\sqrt{m}}{\delta}}{m} \label{eqn:hatkK1}
\end{align}

Similarly, let $\widehat{kK}_2((Q,r); (Q^0, r^0))$ identify the linear under-approximation to the function $kK_2(Q, r) :=kl\left(\sum_{i = 1}^H \hat{l}_iq_i, r \right)$ in \eqref{eqn:klKLDCcons.reverse.preCCP} at $(Q^0, r^0)$.
\begin{align}
\widehat{kK}_2((Q,r); (Q^0, r^0)) &:= kK_2(Q^0, r^0) + \langle \nabla kK_2(Q^0, r^0), \left(Q - Q^0, r  - r^0\right) \rangle \nonumber\\
&= kl\left(\sum_{i = 1}^H \hat{l}_iq^0_i, r^0 \right) +  \left( \sum_{i = 1}^{H} \frac{\partial  kK_2 }{\partial q_i} \bigg\vert_{q_i = q^0_i} \cdot (q_i - q^0_i)  \right) + \frac{\partial  kK_2 }{\partial r}\bigg\vert_{r = r^0} \cdot (r  - r^0) \nonumber\\
&= \left(\sum_{i=1}^{H} \hat{l}_i q_i^0\right) \ln{\left(\frac{\sum_{i=1}^{H} \hat{l}_i q_i^0}{r^0}\right)} + \left(1-\sum_{i=1}^{H} \hat{l}_i q_i^0 \right) \ln{\left(\frac{1-\sum_{i=1}^{H} \hat{l}_i q_i^0}{1-r^0}\right)} \nonumber\\
	&\hspace{2cm}+ \sum_{i=1}^{H} \left[ \hat{l}_i \left( \ln{\left(\frac{\sum_{i=1}^{H} \hat{l}_i q^0_i}{r^0} \right)} - \ln{\left(\frac{1-\sum_{i=1}^{H} \hat{l}_i q^0_i}{1-r^0} \right)} \right) \right] (q_i - q^0_i) \nonumber\\
	&\hspace{3cm}+ \left[\frac{- \sum_{i=1}^{H} \hat{l}_i q^0_i }{r^0} + \frac{1- \sum_{i=1}^{H} \hat{l}_i q^0_i}{1-r^0}\right] (r - r^0) \nonumber \displaybreak[3]\\
&\hspace{-1cm}= \cancel{\left(\sum_{i=1}^{H} \hat{l}_i q_i^0\right) \ln{\left(\frac{\sum_{i=1}^{H} \hat{l}_i q_i^0}{r^0}\right)}} + \ln{\left(\frac{1-\sum_{i=1}^{H} \hat{l}_i q_i^0}{1-r^0}\right)}  \nonumber\\
&\hspace{1cm}- \bcancel{\left(\sum_{i=1}^{H} \hat{l}_i q_i^0 \right) \ln{\left(\frac{1-\sum_{i=1}^{H} \hat{l}_i q_i^0}{1-r^0}\right)}} \nonumber\\
	&\hspace{2cm}+ \sum_{i=1}^{H} \hat{l}_iq_i \left[\ln{\left(\frac{\sum_{i=1}^{H} \hat{l}_i q^0_i}{r^0} \right)} - \ln{\left(\frac{1-\sum_{i=1}^{H} \hat{l}_i q^0_i}{1-r^0} \right)} \right] \nonumber\\
	&\hspace{2cm}- \left[ \cancel{\sum_{i=1}^{H} \hat{l}_iq^0_i \ln{\left(\frac{\sum_{i=1}^{H} \hat{l}_i q^0_i}{r^0} \right)}} -  \bcancel{\sum_{i=1}^{H} \hat{l}_iq^0_i \ln{\left(\frac{1-\sum_{i=1}^{H} \hat{l}_i q^0_i}{1-r^0} \right)}} \right] \nonumber\\
	&\hspace{3cm}+ \left[\frac{r^0 - \sum_{i=1}^{H} \hat{l}_i q^0_i }{r^0 (1 - r^0)}\right] (r - r^0) \nonumber \displaybreak[3] \\
&= \ln{\left(\frac{1-\sum_{i=1}^{H} \hat{l}_i q_i^0}{1-r^0}\right)} + \sum_{i=1}^{H} \hat{l}_iq_i \left[\ln{\left(\frac{\sum_{i=1}^{H} \hat{l}_i q^0_i}{r^0} \right)} - \ln{\left(\frac{1-\sum_{i=1}^{H} \hat{l}_i q^0_i}{1-r^0} \right)} \right] \nonumber\\
	&\hspace{3cm}+ \left[\frac{r^0 - \sum_{i=1}^{H} \hat{l}_i q^0_i }{r^0 (1 - r^0)}\right] (r - r^0)  \label{eqn:hatkK2}
\end{align}

Using the linear approximations $\widehat{kK}_1((Q,r); (Q^0, r^0))$ in \eqref{eqn:hatkK1} and $\widehat{kK}_2((Q,r); (Q^0, r^0))$ in \eqref{eqn:hatkK2} for \eqref{eqn:klKLDCcons.preCCP} and \eqref{eqn:klKLDCcons.reverse.preCCP}, we can invoke the CCP procedure described in \cite{CCP2016LippBoyd} to get a local minimizer to the KL-distance based bound minimization problem ((7) in main paper), as illustrated in Section \ref{chap:optQSVM}.

\subsection{Non-convexity of bound function, $B_\text{CH,KL}(Q)$}
We check for convexity of the bound via first order convexity property. We need to show that the following holds:
\begin{equation}
B_\text{CH, KL}  \geq \langle \nabla B_\text{CH, KL}(Q), Q^{'} -Q \rangle  +
B_\text{CH, KL}(Q) \qquad \forall Q, Q^{'}
\end{equation}
\begin{lem}
The bound function $B_\text{CH, KL}(Q)$ defined in (16) in the main paper is a non-convex function and hence the associated bound minimization problem is non-convex program.
\end{lem}
\begin{proof}
We first compute the derivatives of the bound function:
\begin{align*}
\frac{\partial B_\text{CH, KL}}{\partial q_i} &=\hat{l}_i  + \frac{1}{2\sqrt
{r_\text{CH}(R(Q))}} \frac{\partial r_\text{CH}}{\partial R} \frac{\partial 
R}{\partial q_i} =  \hat{l}_i  + \frac{1}{2\sqrt {r_\text{CH}(R(Q))}} \frac{\partial
r_\text{CH}}{\partial R} \frac{1+\ln\frac{q_i}{p_i}}{2m-1}
\end{align*}
where
\begin{align*}
&\hspace{-1.5cm}\frac{\partial r_\text{CH}}{\partial R} = \frac{1}{3}\left[\frac{1225}{512}
+\frac{135R}{32}  + \frac{5}{32} \sqrt{729R^2+ \frac{6615R}{8} +
\frac{208980}{256}}\right]^{\frac{-2}{3}} \left[\frac{135}{32} + \frac{5}{32} \frac{2R*729
+\frac{6615}{8}}{\sqrt{729R^2+ \frac{6615R}{8} + \frac{208980}{256}}}\right]\\
&+   \frac{1}{3}\left[\frac{1225}{512} +\frac{135R}{32}  - \frac{5}{32}
\sqrt{729R^2+ \frac{6615R}{8} + \frac{208980}{256}}\right]^{\frac{-2}{3}}
\left[\frac{135}{32} - \frac{5}{32} \frac{2R*729 +\frac{6615}{8}}{\sqrt{729R^2+
\frac{6615R}{8} + \frac{208980}{256}}}\right]
\end{align*}

Using the above two expressions, we can obtain the following:
\begin{align*}
    \langle  \nabla B_\text{CH, KL}(Q), Q^{'} -Q \rangle &= \sum_{i=1}^{H} \frac{\partial B_\text{CH, KL}}{\partial q_i}(q^{'}_i -q_i) \\
    & = \sum_{i=1}^{H}  \left[\hat{l}_i+ \frac{1}{2\sqrt {r_\text{CH}(R(Q))}} \frac{\partial
r_\text{CH}}{\partial R} \frac{1+\ln\frac{q_i}{p_i}}{2m-1}\right](q^{'}_i -q_i) \\
& = \sum_{i=1}^{H}   (q^{'}_i -q_i)\hat{l}_i+ \sum_{i=1}^{H}  \left[\hat{l}_i+
\frac{1}{2\sqrt {r_\text{CH}(R(Q))}} \frac{\partial r_\text{CH}}{\partial R}
\frac{1+\ln\frac{q_i}{p_i}}{2m-1}(q^{'}_i -q_i)\right] \\
&= \sum_{i=1}^{H}   (q^{'}_i -q_i)\hat{l}_i + \frac{1}{2\sqrt {r_\text{CH}(R(Q))}}
\frac{\partial r_\text{CH}}{\partial R} \frac{\sum_{i=1}^{H} q^{'}_i
\ln\frac{q_i}{p_i}-  \sum_{i=1}^{H} q_i \ln\frac{q_i}{p_i}}{2m-1}
\end{align*}

Thus it  sufficient to check the following conditon:
\begin{multline*}
\sum_{i=1}^{H} \hat{l}_i q^{'}_i + \sqrt{r_{\text{CH}}}(R(Q^{'})) \geq
\sum_{i=1}^{H} \hat{l}_i q^{'}_i - \sum_{i=1}^{H} \hat{l}_i q_i +  \frac{1}{2 \sqrt{r_{\text{CH}}(R,
(Q^{'})) }}\frac{\partial r_{\text{CH}}}{\partial R}\frac{\sum_{i=1}^{H} q^{'}_i \ln\frac{q_i}{p_i}-  \sum_{i=1}^{H} q_i
\ln\frac{q_i}{p_i}}{2m-1} \\
+ \sum_{i=1}^{H} \hat{l}_i q_i + \sqrt{r_{\text{CH}}(R,
(Q)) } 
\end{multline*}
which reduces  to :
\begin{align}
2\sqrt{r_{\text{CH}}(R,Q^{'})}\sqrt{r_{\text{CH}}(R,Q)} \geq \frac{\partial
r_{\text{CH}}}{\partial R} \frac{\sum_{i=1}^{H} q^{'}_i \ln\frac{q_i}{p_i}- 
\sum_{i=1}^{H} q_i \ln\frac{q_i}{p_i}}{2m-1} + 2r_{\text{CH}}(R(Q))
\end{align}
This condition is violated at the following counter example: 

$H=10, m=40, \delta =0.01, P \sim$ Unif$(\mathcal{H})$
 and $Q^{'}\sim$ Unif$(\mathcal{H})\equiv P$
 \begin{multline*}
 Q=  (0.038905393, 0.117691919,0.034856483,0.135564832,
0.039842634\\
0.134196637,0.293501960,0.191581463,0.007108130,0.006750549)
\end{multline*}
 $LHS = 2.539717$ , $RHS= 2.541530$
 
Hence, $B_{\text{CH,KL}}(Q)$ is a non-convex function of $Q$. 
But computations show that this optimization problem has a single local minimum for uniform prior $P$ on $\mathcal{H}$.
\end{proof}
 
\section{Optimal posterior for linear distance function}
\subsection{The function $\mathcal{I}^K_{\text{lin}} (m, l)$ for linear distance based bound}
\begin{figure}[ht]
    \centering
    \includegraphics[width = 0.46\textwidth]{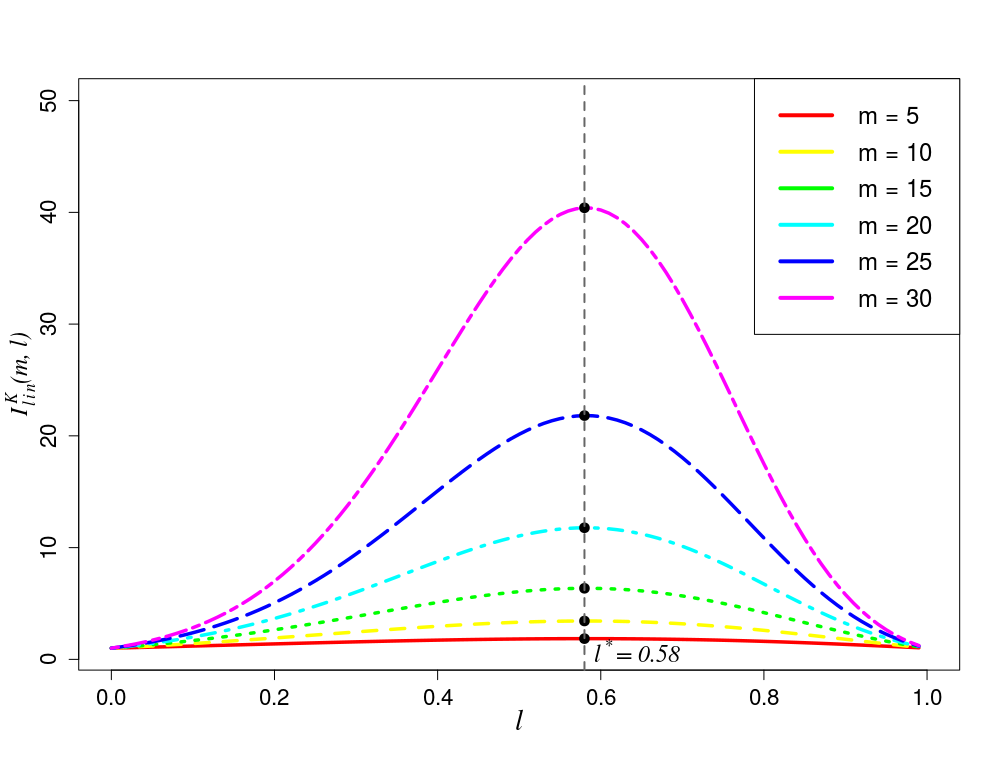}
    \includegraphics[width = 0.52\textwidth]{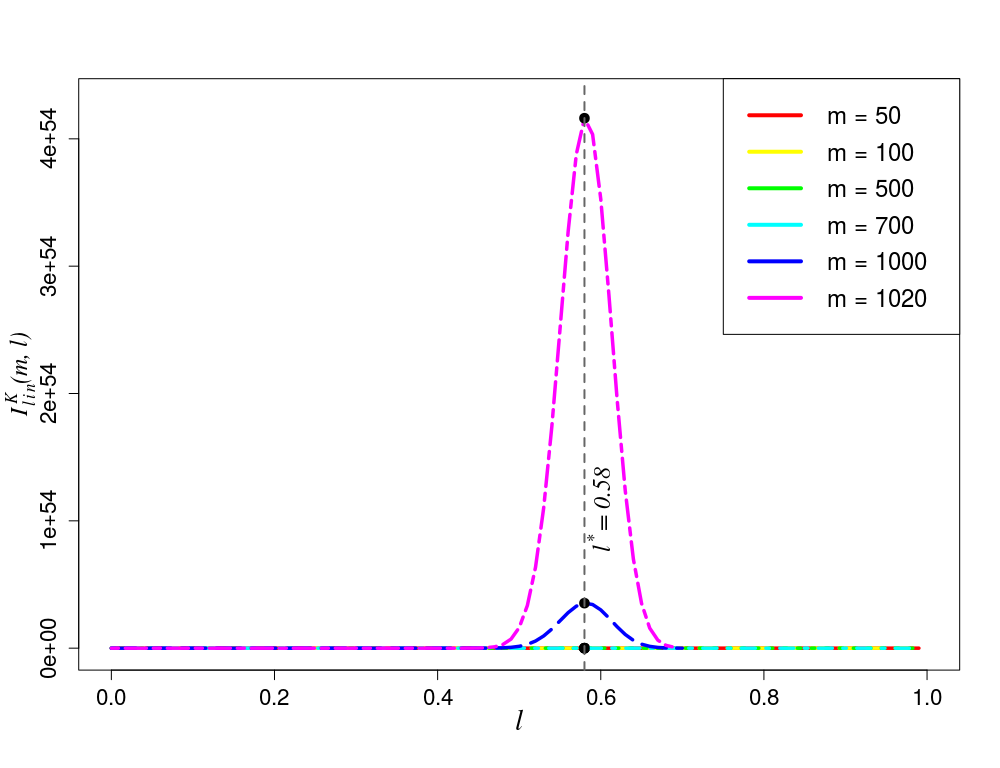}
    \caption{Plot of the function $\mathcal{I}^K_{\text{lin}} (m, l)$ for large and small sample sizes. For $m \geq 1020$, computation is difficult due to storage limitations in the range of floating point numbers. For other case, $l^{\ast} = 0.58$ as observed graphically.}  
    \label{fig:I_K_lin}
\end{figure}
 
\begin{table}[ht]
\centering
\begin{tabular}{|r|rr|}
  \hline
\makecell{Sample \\ size, $m$} & $I^K_\text{lin}(m)$ & $2\sqrt{m}$ \\ 
  \hline
  5 & 1.85 & 4.47 \\ 
   10 & 3.43 & 6.32 \\ 
   15 & 6.36 & 7.75 \\ 
   20 & 11.78 & 8.94 \\ 
   25 & 21.81 & 10.00 \\ 
   30 & 40.41 & 10.95 \\
   50 & 475.79 & 14.14 \\ 
  100 & 2.3e+05 & 20.00 \\ 
   500 & 5.9e+26 & 44.72 \\ 
  700 & 3e+37 & 52.92 \\ 
  1000 & 3.5e+53 & 63.25 \\ 
  1020 & 4.2e+54 & 63.87 \\ 
   \hline
\end{tabular}
\end{table}
For $m \geq 1020$, computation is difficult due to storage limitations in the range of floating point numbers -- gives  $I^K_\text{lin}(m)$ as NaN.

\subsection{Optimal posterior for linear distance based bound}
We can identify the minimizer for the bound minimization problem (19) (in paper) using the KKT system based on the associated Lagrangian function. 
The Lagrangian function for (19) (in paper) can be written as follows:
\begin{eqnarray}
\mathcal{L}_{\text{lin, KL}}(Q,\mu_0,\mu_i) = \sum\limits_{i=1}^{H}\hat{l}_i q_i + \frac{\sum\limits_{i=1}^{H}q_i\ln\frac{q_i}{p_i} }{m} - \mu_0 \left(\sum\limits_{i=1}^H q_i -1 \right) - \sum_{i = 1}^H \mu_i q_i
\end{eqnarray}
Differentiating Lagrange $\mathcal{L}_{\text{lin, KL}}$ with respect to primal variables $q_i$s and dual variable $\mu_0$, we get:
\begin{eqnarray*}
\frac{\partial \mathcal{L}_{\text{lin, KL}}}{\partial q_i} &=& \hat{l}_i + \frac{1 + \ln \frac{q_i}{p_i}}{m} - \mu_0 -\mu_i  \quad \forall i = 1, \ldots, H\\
\frac{\partial \mathcal{L}_{\text{lin, KL}}}{\partial \mu_0} &=& \sum_{i=1}^{H}q_i -1 
\end{eqnarray*}
We assume that $q_i > 0$ for all $i = 1, \ldots, H$, since otherwise $\ln q_i = \ln(0)$ is undefined. Even if we use fact that $\lim_{x \rightarrow 0^{+}} \ln x = - \infty$, the KKT condition will mean that the dual variable $\mu_i$ is infeasible. Therefore our assumption holds true for a stationary point.
Due to complementary slackness conditions, we have $\mu_i q_i = 0$ which implies $\mu_i = 0$ since $q^{\ast}_i > 0$ for all $i = 1, \ldots, H$. 

At optimality, posterior $Q^{\ast}_{\text{lin, KL}}$ should set the derivatives of this Lagrangian function $\mathcal{L}_{\text{lin, KL}}$ to zero. Setting the derivative of $\mathcal{L}_{\text{lin, KL}}$ with respect to $q_i$'s as zero, we get:

\begin{align}
&\frac{\partial \mathcal{L}_{\text{lin, KL}}}{\partial q_i} = 0 \quad \forall i = 1, \ldots, H\nonumber\\
\Rightarrow \; & \hat{l}_i + \frac{\left(1 + \ln \frac{q_i}{p_i} \right)}{m} - \mu_0 = 0 \nonumber\\
\Rightarrow \; & 1 + \ln \frac{q_i}{p_i} = m(\mu_0 - \hat{l}_i) \nonumber\\
\Rightarrow \; &\ln \frac{q_i}{p_i} = m(\mu_0 - \hat{l}_i) -1 \nonumber\\
\Rightarrow \; &q_i = p_i e^{m(\mu_0 - \hat{l}_i)-1} \quad \forall i = 1, \ldots, H 
\label{eqn:qLagrange_linKL}
\end{align}
And now, setting the derivative of $\mathcal{L}_{\text{lin, KL}}$ with respect to $\mu_0$ as zero, we get:

\begin{align}
&\frac{\partial \mathcal{L}_{\text{lin, KL}}}{\partial \mu_0} = 0 \nonumber\\
\Rightarrow \; &\sum_{i=1}^{H}q_i =1 \nonumber\\
\Rightarrow \; &\sum_{i=1}^H p_ie^{m(\mu_0 - \hat{l}_i)-1} =1 \nonumber\\
\Rightarrow \; &e^{m\mu_0 - 1}\left( \sum_{i=1}^H p_ie^{-m\hat{l}_i}\right) =1 \nonumber\\
\Rightarrow \; &e^{m\mu_0 -1} = \frac{1}{\sum_{i=1}^{H}p_i e^{-m\hat{l}_i}} \label{eqn:Lagrangemultiplier_linKL} \\[3mm]
\Rightarrow \; & \mu_0 = \frac{1 -\ln \left(\sum_{i=1}^{H}p_ie^{-m\hat{l}_i}\right)}{m}
\end{align}  
Therefore, combining \eqref{eqn:qLagrange_linKL} and \eqref{eqn:Lagrangemultiplier_linKL}, we get the following expression for a KKT point solution to our bound minimization problem:
\begin{equation}
q_i^{KKT} = p_ie^{m(\mu_0 - \hat{l}_i)-1} = e^{m\mu_0 -1}(p_ie^{-m\hat{l}_i})= \frac{p_ie^{-m\hat{l}_i}}{\sum_{i=1}^{H}p_i e^{-m\hat{l}_i}} >0 ~ \forall i = 1, \ldots, H. 
\end{equation} 
%\textcolor{blue}{above is partial KKT point or to be called somethings else}

This implies that $q^{KKT}_i \propto p_ie^{-m\hat{l}_i}$. That is, a classifier will be weighed negatively exponentially to the number of misclassfications it makes on the training sample. For a given prior distribution $P$, the optimal posterior $Q_{\text{lin, KL}}$ is a Boltzmann distribution.

\begin{theorem}
When the prior is a uniform distribution on the set $\mathcal{H}$ of classifiers, the optimal posterior $Q_{\text{lin, KL}}^{\ast}$ for the bound minimization problem (19) (in paper) has full support. That is, all the classifiers in $\mathcal{H}$ will have strictly positive posterior weight at optimality.
\end{theorem}
\begin{proof}
Using the result of Theorem 6 (in paper), it is sufficient to compare the bound values corresponding to the best posteriors for ordered subsets of $\mathcal{H}$, ranked by non-decreasing $\hat{l}_i$ values, to determine the optimal posterior for (19) (in paper). Using Theorem 5 (in paper), the optimal posterior $Q_{lin, KL}^{\ast}(H')$ on an ordered subset of classifiers of size $H' \in [H]$ is given as:
\begin{equation*}
q_{i, lin, KL}^{\ast} (H') = 
\begin{cases}
\frac{e^{-m\hat{l}_i}}{\sum_{i=1}^{H'} e^{-m\hat{l}_i}} ~ &\forall i = 1, \ldots, H' \\
0 ~ &\forall i = H' + 1, \ldots, H,
\end{cases}
\end{equation*}
and the optimal objective value is:
\begin{align*}
B_{lin, KL}(Q_{lin, KL}^{\ast}(H')) &= \sum_{i = 1}^{H} \hat{l}_iq^{\ast}_{i, lin, KL} + \frac{\sum_{i =1}^{H} q^{\ast}_{i, lin, KL} \ln (q^{\ast}_{i, lin, KL} H)}{m} \\
&= \sum_{i = 1}^{H'} \hat{l}_i \left( \frac{e^{-m\hat{l}_i}}{\sum_{i=1}^{H'} e^{-m\hat{l}_i}} \right) + \frac{\sum\limits_{i =1}^{H'} \frac{e^{-m\hat{l}_i}}{\sum_{i=1}^{H'} e^{-m\hat{l}_i}} \cdot \ln \left( \frac{H e^{-m\hat{l}_i}}{\sum_{i=1}^{H'} e^{-m\hat{l}_i}}\right)}{m} \\
&=  \frac{\sum_{i = 1}^{H'} \hat{l}_ie^{-m\hat{l}_i}}{\sum_{i=1}^{H'} e^{-m\hat{l}_i}}  + \frac{\sum\limits_{i =1}^{H'} \left( \frac{e^{-m\hat{l}_i}}{\sum_{i=1}^{H'} e^{-m\hat{l}_i}} \left[-m\hat{l}_i + \ln \left( \frac{H}{\sum_{i=1}^{H'} e^{-m\hat{l}_i}}\right) \right] \right)}{m} \\
&=  \frac{\sum_{i = 1}^{H'} \hat{l}_ie^{-m\hat{l}_i}}{\sum_{i=1}^{H'} e^{-m\hat{l}_i}}  + \frac{-\cancel{m} \cdot \sum\limits_{i =1}^{H'} \left( \frac{e^{-m\hat{l}_i}}{\sum_{i=1}^{H'} \hat{l}_i e^{-m\hat{l}_i}} \right)}{\cancel{m}} + \frac{\sum\limits_{i =1}^{H'} \left( \frac{e^{-m\hat{l}_i}}{\sum_{i=1}^{H'} e^{-m\hat{l}_i}} \ln \left( \frac{H}{\sum_{i=1}^{H'} e^{-m\hat{l}_i}}\right) \right)}{m} \\
&=  \cancel{\frac{\sum_{i = 1}^{H'} \hat{l}_ie^{-m\hat{l}_i}}{\sum_{i=1}^{H'} e^{-m\hat{l}_i}}} - \cancel{\frac{\sum_{i = 1}^{H'} \hat{l}_ie^{-m\hat{l}_i}}{\sum_{i=1}^{H'} e^{-m\hat{l}_i}}} + \ln \left( \frac{H}{\sum_{i=1}^{H'} e^{-m\hat{l}_i}}\right) \cdot \frac{\left( \frac{\cancel{\sum_{i =1}^{H'} e^{-m\hat{l}_i}}}{\cancel{\sum_{i=1}^{H'} e^{-m\hat{l}_i}}} \right)}{m} \\
&= \frac{\ln H - \ln \left( \sum_{i =1}^{H'} e^{-m\hat{l}_i} \right)}{m}
\end{align*}
Since $e^{-x} > 0 $ for all $x \in \mathbb{R}$, the sum $\sum_{i =1}^{H'} e^{-m\hat{l}_i}$ is an increasing function of $H' \in [H]$. Using the monotone increasing property of natural logarithm function, the bound function, $B_{lin, KL}(Q_{lin, KL}^{\ast}(H'))$ is a decreasing function of $H' = 1, \ldots, H$. Therefore the least bound value is achieved when all the classifiers are assigned strictly positive weights, that is, when the optimal posterior has full support.
\end{proof}

\section{Optimal posterior for squared distance function}
\begin{lem}
For a given sample size, $m$, $l^\ast = 0.5$ is the maximizer of $\mathcal{I}^K_{\text{sq}}(m,l) := \sum_{k=0}^{m}{m \choose k} l^k (1-l)^{m-k}e^{m\left(\frac{k}{m}-l\right)^2} $ for $l \in [0, 1]$.
\end{lem}
\begin{proof}
$\mathcal{I}^K_{\text{sq}}(m,l)$ is a real valued function on the interval $[0, 1]$, hence we can identify its maximizer(s) via derivative test. We need to show that $\left.\frac{d}{dl}\mathcal{I}^K_{\text{sq}}(m,l)\right|_{l=0.5} =0$ and $\left.\frac{d^2}{dl^2}\mathcal{I}^K_{\text{sq}}(m,l)\right|_{l=0.5} \leq 0$.

Considering individual terms in the sum $\mathcal{I}^K_{\text{sq}}(m,l)$, we observe that except for $k=0$ and $k=m$, all other terms involve product of powers of both $l$ and $1-l$.
\begin{multline}
\mathcal{I}^K_{\text{sq}}(m,l) = {m \choose 0} l^0 (1-l)^{m}e^{m(0-l)^2} + {m \choose m} l^m (1-l)^{0}e^{m(1-l)^2} \\
+ \sum\limits_{k=1}^{m-1}{m \choose k} l^k (1-l)^{m-k}e^{m\left(\frac{k}{m}-l\right)^2}
\end{multline}

First derivative:
\begin{align*}
    \frac{d}{dl}\mathcal{I}^K_{\text{sq}}(m,l) &=  \frac{d}{dl}(1-l)^{m}e^{ml^2} + \frac{d}{dl}l^m e^{m(1-l)^2} + \frac{d}{dl}\Bigg(\sum\limits_{k=1}^{m-1}{m \choose k} l^k (1-l)^{m-k}e^{m\left(\frac{k}{m}-l\right)^2}\Bigg)\\
    &= -m(1-l)^{m-1}e^{ml^2} + (1-l)^{m}e^{ml^2} 2ml + ml^{m-1} e^{m(1-l)^2} - l^m e^{m(1-l)^2} 2m(1-l)\\
    &\quad + \sum\limits_{k=1}^{m-1}{m \choose k} \Bigg[kl^{k-1} (1-l)^{m-k} - (m-k)(1-l)^{m-k-1}l^k \\
    &\qquad - l^k (1-l)^{m-k} 2m\left(\frac{k}{m}-l\right)\Bigg]e^{m\left(\frac{k}{m}-l\right)^2}
\end{align*}

\begin{align*}
    \left.\frac{d}{dl}\mathcal{I}^K_{\text{sq}}(m,l)\right|_{l=0.5} &= \bcancel{-m(0.5)^{m-1}e^{m(0.5)^2}} + \cancel{(0.5)^{m}e^{m(0.5)^2} m} \\
    &\quad + \bcancel{m(0.5)^{m-1} e^{m(0.5)^2}} - \cancel{(0.5)^m e^{m(0.5)^2} m}\\
    &\quad + \sum\limits_{k=1}^{m-1}{m \choose k} \Bigg[k(0.5)^{k-1} (0.5)^{m-k} - (m-k)(0.5)^{m-k-1}(0.5)^k \\
    &\qquad - (0.5)^k (0.5)^{m-k} 2m\left(\frac{k}{m}-0.5\right)\Bigg]e^{m\left(\frac{k}{m}-0.5\right)^2} \\
    &= 0 + \sum\limits_{k=1}^{m-1}{m \choose k} (0.5)^{m-1} \Bigg( \underbrace{k- (m-k) -k + 0.5m}_{0}\Bigg)e^{m\left(\frac{k}{m}-0.5\right)^2} \\
    &= \sum\limits_{k=1}^{m-1}{m \choose k} (0.5)^{m} \Bigg( \frac{k}{m} - 0.5\Bigg)e^{m\left(\frac{k}{m}-0.5\right)^2}= 0
\end{align*}
~\\
The second order derivative test is not conclusive, 
but we can refer to the adjoining Figure \ref{fig:I_K_sq} where we have plotted the function $\mathcal{I}^K_{\text{sq}} (m, l) = \sum_{k=0}^{m}{m \choose k} l^k (1-l)^{m-k}e^{m\left(\frac{k}{m}-l\right)^2}$ for different sample size, $m$. 
but we observe in the graph that, for each $m$, $\mathcal{I}^K_{\text{sq}} (m, l)$ is a non-monotone function of $l \in [0, 1]$ which attains its maximum at $l = 0.5$. Hence the proof.

Thus, we have
\begin{align*}
\mathcal{I}^K_{\text{sq}} (m) &= \sup_{l \in [0,1]} \mathcal{I}^K_{\text{sq}} (m, l)\\
&= \sup_{l \in [0,1]}\sum\limits_{k=0}^{m}{m \choose k} l^k (1-l)^{m-k}e^{m\left(\frac{k}{m}-l\right)^2} \\
&= \sum\limits_{k=0}^{m}{m \choose k} (0.5)^m e^{m\left(\frac{k}{m}-0.5\right)^2}
\end{align*}

\begin{table}[ht]
\centering
\begin{tabular}{|r|rr|}
  \hline
\makecell{Sample \\ size, $m$} & $I^K_\text{sq}(m)$ & $2\sqrt{m}$ \\ 
  \hline
 10 & 1.39 & 6.32 \\ 
   50 & 1.41 & 14.14 \\ 
  100 & 1.41 & 20.00 \\ 
  200 & 1.41 & 28.28 \\ 
  500 & 1.41 & 44.72 \\ 
  1000 & 1.41 & 63.25 \\ 
   \hline
\end{tabular}
\end{table}
For $m \geq 1020$, computation is difficult due to storage limitations in the range of floating point numbers -- gives  $I^K_\text{sq}(m)$ as NaN.

\begin{figure}[htpb]
\centering
\includegraphics[width=0.7\textwidth]{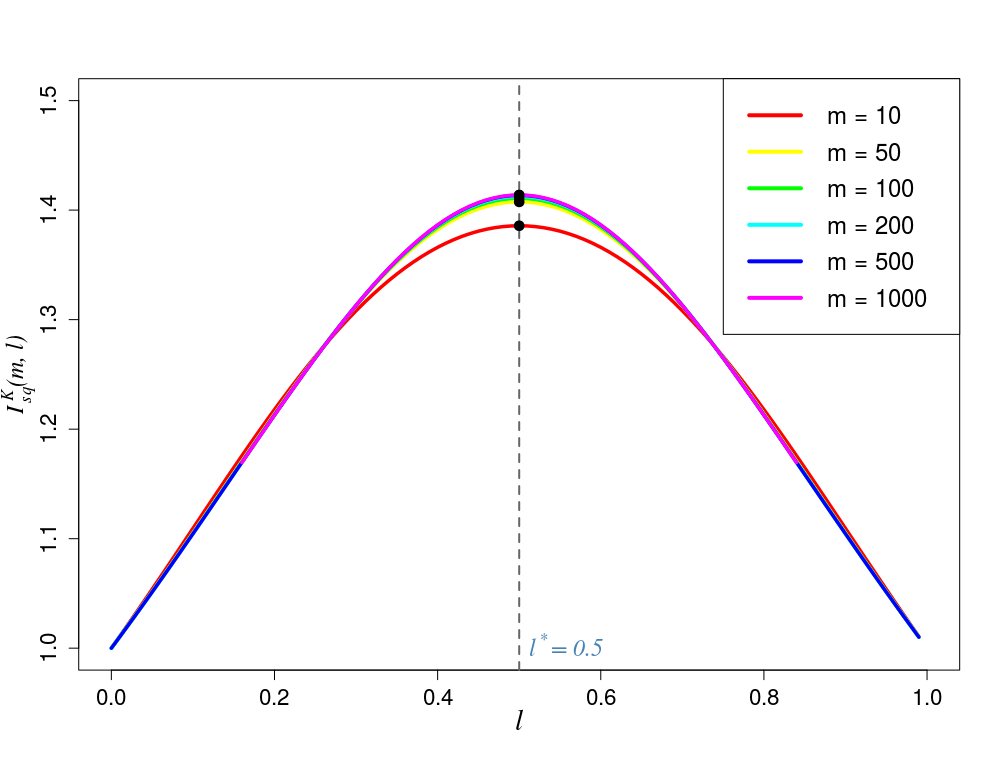}
\caption[Plot of the function $\mathcal{I}^K_{\text{sq}} (m, l)$ for PAC-Bayesian bound with squared distance function and KL divergence measure]{Plot of the threshold function $\mathcal{I}^K_{\text{sq}} (m, l) = \sum\limits_{k=0}^{m}{m \choose k} l^k (1-l)^{m-k}e^{m\left(\frac{k}{m}-l\right)^2}$ as a function of the true risk $l \in [0, 1]$ for different values of the sample size, $m$ represented by different curves in the above graph. We observe that the function $\mathcal{I}^K_{\text{sq}} (m, l)$ is concave in its domain and symmetric about $l = 0$. We are interested in the quantity $\mathcal{I}^K_{\text{sq}} (m) = \sup_{l \in [0,1]} \mathcal{I}^K_{\text{sq}} (m, l)$ as a function of $m$ which we identify graphically (and mark it by a $\bullet$ on each curve)}
\label{fig:I_K_sq}
\end{figure}
\end{proof}

\subsection{Convexity of the bound function, $B_{\text{sq, KL}}(Q)$}
% \begin{propn}
% The bound function, $B_{\text{sq, KL}}(Q) = \sum_{i=1}^{H} \hat{l}_iq_i + \sqrt {\frac{\sum_{i=1}^{H} q_i\ln\frac{q_i}{p_i} + \ln \frac{2\sqrt{m}}{\delta}}{m}}$ is non-convex.
% \end{propn}
% \begin{proof}
We use the first order condition to verify convexity of our bound function. For convexity, we need the following condition to hold for any pair of distributions $Q$ and $Q'$ (that are absolutely continuous with respect to the prior distribution $P$) on classifier space $\mathcal{H}$:
\begin{equation}
B_{\text{sq, KL}}(Q') \geq B_{\text{sq, KL}}(Q) + \langle\nabla B_{\text{sq, KL}}(Q), Q'-Q\rangle \quad \forall Q, Q' \label{eqn:1ConvexBsqKL}
\end{equation}
Our classifier space $\mathcal{H}$ is a finite set, $\mathcal{H} = \lbrace h_i\rbrace_{i =1}^H$. So, any distribution on $\mathcal{H}$ is a discrete distribution which can be represented as $Q = \left( q_1, \ldots, q_H \right)$. To find the gradient $\nabla B_{\text{sq, KL}}$, we compute the (first) derivative of $B_{\text{sq, KL}}$ with respect to variable $q_i, i \in \lbrace 1, \ldots, H\rbrace$:
\begin{align*}
\frac{\partial B_{\text{sq, KL}}(Q)  }{\partial q_i} &=  \hat{l}_i + 
\frac{1}{2\sqrt{m}} \frac{1}{\sqrt{\sum_{i=1}^{H} q_i \ln\frac{q_i}{p_i} + \ln\frac{2\sqrt{m}}{\delta}}} \cdot \frac{\partial}{\partial q_i} \left(\sum_{j=1}^{H} q_j\ln\frac{q_j}{p_j}\right) \\[3mm]
&= \hat{l}_i +\frac{1}{2\sqrt{m}} \frac{1 + \ln\frac{q_i}{p_i}}{\sqrt{\sum_{i=1}^{H}q_i\ln\frac{q_i}{p_i} + \ln \frac{2\sqrt{m}}{\delta}}}
\end{align*}

Consider the following inner product:
\begin{align*}
\langle \nabla  B_{\text{sq, KL}}(Q), Q'-Q\rangle
&=\sum_{i=1}^{H} \left[ \frac{\partial B_{\text{sq, KL}}(Q)}{\partial q_i} \cdot (q'_i - q_i) \right] \\
&= \sum_{i=1}^{H} \left( \hat{l_i} + \frac{1}{2\sqrt{m}} \frac{1 + \ln\frac{q_i}{p_i}}{\sqrt{\sum_{i=1}^{H} q_i\ln\frac{q_i}{p_i} + \ln \frac{2\sqrt{m}}{\delta}}} \right) (q'_i - q_i) \\
&= \sum_{i=1}^{H} \hat{l}_iq'_i - \sum_{i=1}^{H} \hat{l}_iq_i + \frac{1}{2\sqrt m}  \frac{\sum_{i=1}^{H}q_i^{'}( 1+ \ln\frac{q_i}{p_i}) -
\sum_{i=1}^{H}q_i( 1+ \ln\frac{q_i}{p_i})}{\sqrt{\sum_{i=1}^{H} q_i\ln\frac{q_i}{p_i} + \ln \frac{2\sqrt{m}}{\delta}}} 
\end{align*}
%  \underline{	We  will going to check the the first order condition of convexity}
To  check the  first order condition we need  to  verify the inequality \eqref{eqn:1ConvexBsqKL}:
\begin{align}
&B_{\text{sq, KL}}(Q') \geq B_{\text{sq, KL}}(Q) + \langle\nabla B_{\text{sq, KL}}(Q), Q'-Q\rangle \nonumber\\
\Rightarrow \; &\sum_{i=1}^{H} \hat{l}_iq'_i  +   \sqrt{\frac{ \sum_{i=1}^{H} q'_i \ln\frac{q'_i}{p_i} + \ln \frac{2\sqrt{m}}{\delta}}{m}} 
\geq 
\sum_{i=1}^{H} \hat{l}_iq^{'}_i - \sum_{i=1}^{H} \hat{l}_iq_i + \frac{1}{2\sqrt m} \frac{\sum_{i=1}^{H}q'_i \ln\frac{q_i}{p_i} - \sum_{i=1}^{H}q_i \ln\frac{q_i}{p_i}}{\sqrt{\sum_{i=1}^{H} q_i\ln\frac{q_i}{p_i} + \ln \frac{2\sqrt{m}}{\delta}}} \nonumber\\
&\hspace{8cm} + \sum_{i=1}^{H} \hat{l}_iq_i + \sqrt {\frac{\sum_{i=1}^{H} q_i\ln\frac{q_i}{p_i} + \ln \frac{2\sqrt{m}}{\delta}}{m}} \nonumber\\
\Rightarrow \; &\sqrt {\sum_{i=1}^{H} q'_i\ln\frac{q'_i}{p_i} + \ln
\frac{2\sqrt{m}}{\delta}}  \geq     \frac{\sum_{i=1}^{H}q'_i \ln\frac{q_i}{p_i} -\sum_{i=1}^{H}q_i \ln\frac{q_i}{p_i} + 2 \left(\sum_{i=1}^{H} q_i\ln\frac{q_i}{p_i} + \ln
\frac{2\sqrt{m}}{\delta} \right)}{2\sqrt{\sum_{i=1}^{H} q_i\ln\frac{q_i}{p_i} + \ln \frac{2\sqrt{m}}{\delta}}} \nonumber%\\
\end{align}
\begin{multline}
\Rightarrow \; 2\left(\sqrt {\sum_{i=1}^{H} q'_i\ln\frac{q'_i}{p_i} + \ln\frac{2\sqrt{m}}{\delta}}\right) \left(\sqrt{\sum_{i=1}^{H} q_i\ln\frac{q_i}{p_i} +
\ln \frac{2\sqrt{m}}{\delta}}\right) \\
\geq \left(  \sum_{i=1}^{H}q^{'}_i \ln\frac{q_i}{p_i} + \ln \frac{2\sqrt{m}}{\delta} \right) +\left(  \sum_{i=1}^{H}q_i \ln\frac{q_i}{p_i} + \ln \frac{2\sqrt{m}}{\delta} \right) \label{eqn:nonconvexityBsqKL}
\end{multline}
%\end{align}
A theoretical proof could not be obtained which shows that the above condition holds for any pair of distributions $Q$ and $Q'$ that are absolutely continuous with respect to $P$ for any set of system parameters: $P, m, \delta, H$. 

The bound function $B_{\text{sq, KL}}(Q)$ is non-convex if there exists a pair of distributions $Q$ and $Q'$ for given system parameters such that the above condition is violated. For different combinations of the parameter values $m, \delta, H$, with uniform and non-uniform prior, $P$ and randomly chosen distributions $Q$ and $Q'$ that are absolutely continuous with respect to $P$, we were unable to get a counter-example for this condition. 

Our computations illustrate that $B_{\text{sq, KL}}(Q)$ has a single local minimum for uniform prior on $\mathcal{H}$. This lead us to investigate quasi-convexity of this bound function.

% We observe that the above inequality does not always hold true. We have counter examples which violate this convexity condition. In our computations, we have uniform distribution as our prior on the classifier space. We present an example which violates the above condition even under uniform prior set up.\\

% \added{check}
% %\begin{itemize}
% %\item 
% $H = 10, m = 50, \delta = 0.01, P \sim $Unif$(\mathcal{H})$ 
% \begin{multline*}
% Q = (0.06062457, 0.02286417, 0.04777685, 0.05935214, 0.02788143, 0.04969765,\\
% 0.3873241, 0.04336256, 0.2391948, 0.06192177)
% \end{multline*}
% \begin{multline*}
% Q' = (0.0238796, 0.01469943, 0.08984046, 0.02009139, 0.0006358691, 0.02666464, \\ 0.6843555, 0.01352949, 0.1225872, 0.003716421) 
% \end{multline*}
% For the above parameters, the LHS of \eqref{eqn:nonconvexityBsqKL} takes value 8.059716 which is smaller than the RHS value, 8.144894. Hence the condition \eqref{eqn:1ConvexBsqKL} is violated.
% \begin{multline*}
% Q = (0.05512569, 0.05298013, 0.1798769, 0.1623063, 0.08095783, 0.1338221, \\ 0.04251432, 0.1356951, 0.03212123 0.1246004)
% \end{multline*}
% \begin{multline*}
% Q' = (0.03953025, 0.03083162, 0.2650598, 0.1876214, 0.04875376, 0.07567512, \\ 0.03654652, 0.2316004, 0.008255068, 0.07612613) \text{give relative values of Eq 4.23?}
% \end{multline*}
% \end{proof}

\subsection{Quasiconvexity of the bound function, $B_{\text{sq,KL}}(Q)$}
We are interested in checking whether $B_{\text{sq, KL}}(Q)$ is strictly quasi-convex. If so, we can  claim that a local optimal solution will be a global optimal solution \cite{bazaraa2013nonlinear}.

%\begin{proof}
$B_{\text{sq,KL}}(Q)$ is defined on the simplex $\Delta^H$ which is a non-empty convex set in $\mathbb{R}^H$. This function is a sum of two terms:
\begin{equation}
B_{\text{sq,KL}}(Q) = E_{Q} [\hat{l}] +  \sqrt{\frac{KL[Q || P] + \ln \frac{\mathcal{I}^K_{\text{sq}}(m)}{\delta}}{m}}.
\end{equation}
The first term, $E_{Q} [\hat{l}]$, is a linear function of $Q$. The second term is the square root of a positive affine transformation of convex function $KL[Q||P]$, where $KL[Q||P]$ is a convex function of $Q$. Also, convexity implies (strict) quasi-convexity. Thus, we have that for a given prior $P$, for each $Q \neq Q'$, that are absolutely continuous with respect to the prior distribution $P$, such that $KL[Q || P]  \neq KL[Q' || P]  $, the following holds for all $\alpha \in (0, 1)$: 
\begin{equation*}
\frac{KL[(\alpha Q + (1-\alpha)Q') || P] + \ln \frac{\mathcal{I}^K_{\text{sq}}(m)}{\delta}}{m}  \nonumber \\
<  \max \left \lbrace \frac{KL[Q || P] + \ln \frac{\mathcal{I}^K_{\text{sq}}(m)}{\delta}}{m} ,\frac{KL[Q'|| P] + \ln \frac{\mathcal{I}^K_{\text{sq}}(m)}{\delta}}{m} \right \rbrace.
\end{equation*}
We know that square root function is strictly increasing in its argument, which implies that:
\begin{multline} \label{eqn:quasiconvex_sqrtKL}
\sqrt{\frac{KL[(\alpha Q + (1-\alpha)Q') || P] + \ln \frac{\mathcal{I}^K_{\text{sq}}(m)}{\delta}}{m}}  \\
<  \max \left \lbrace \sqrt{\frac{KL[Q || P] + \ln \frac{\mathcal{I}^K_{\text{sq}}(m)}{\delta}}{m}} , \sqrt{\frac{KL[Q'|| P] + \ln \frac{\mathcal{I}^K_{\text{sq}}(m)}{\delta}}{m}} \right \rbrace. 
\end{multline}
Thus, we can claim that $\sqrt{\frac{KL[Q || P] + \ln \frac{\mathcal{I}^K_{\text{sq}}(m)}{\delta}}{m}}$ is a (strictly) quasiconvex function of $Q$. Thus, both the components of $B_{\text{sq,KL}}(Q)$ are quasiconvex, but their sum need not be quasiconvex.

Note, in the remaining of the analysis, ``for any $Q, Q'$ " implies ``for any $Q, Q'$ that are absolutely continuous with respect to $P$". This condition is required for $KL[Q||P]$ to be defined.

To claim $B_{\text{sq,KL}}(Q)$ is quasiconvex, we need to show that for each $Q \neq Q'$, that are absolutely continuous with respect to the prior distribution $P$, such that $B_{\text{sq,KL}}(Q) \neq B_{\text{sq,KL}}(Q') $, the following holds: 
\begin{equation*}
    B_{\text{sq,KL}}[\alpha Q + (1-\alpha)Q'] < \max \lbrace B_{\text{sq,KL}}(Q),B_{\text{sq,KL}}(Q') \rbrace \quad \forall \alpha \in (0,1).
\end{equation*}
That is equivalent to showing:
\begin{multline}
E_{\alpha Q + (1-\alpha)Q'} [\hat{l}] + \sqrt{\frac{KL[(\alpha Q + (1-\alpha)Q') || P] + \ln \frac{\mathcal{I}^K_{\text{sq}}(m)}{\delta}}{m}}  \\
<  \max \left \lbrace E_{Q} [\hat{l}] +  \sqrt{\frac{KL[Q || P] + \ln \frac{\mathcal{I}^K_{\text{sq}}(m)}{\delta}}{m}} , E_{Q'} [\hat{l}] +  \sqrt{\frac{KL[Q'|| P] + \ln \frac{\mathcal{I}^K_{\text{sq}}(m)}{\delta}}{m}} \right \rbrace
\end{multline}
We assume that  $B_{\text{sq,KL}}(Q') < B_{\text{sq,KL}}(Q)$. This implies that we need to show that $ B_{\text{sq,KL}}(\alpha Q + (1-\alpha)Q') < B_{\text{sq,KL}}(Q)$. We consider 4 possible cases as follows: 
\begin{enumerate}
\item[Case I]: If $E_{Q'}[\hat{l}] \leq E_{Q}[\hat{l}]$ and $KL[Q'||P] < KL[Q||P]$, then we have:
\begin{equation}
E_{\alpha Q + (1-\alpha)Q'}[\hat{l}] = \alpha E_{Q}[\hat{l}] +(1-\alpha)E_{Q'}[\hat{l}] \leq E_{Q}[\hat{l}] \quad \text{ for each } \alpha \in (0, 1).
\end{equation}
We know that $\sqrt{\frac{KL[Q || P] + \ln \frac{\mathcal{I}^K_{\text{sq}}(m)}{\delta}}{m}}$ is quasiconvex using \eqref{eqn:quasiconvex_sqrtKL}, and we have assumed $KL[Q'||P] < KL[Q||P]$. Hence, the following holds for any $Q, Q'$ for each $\alpha \in (0, 1)$:
\begin{equation}
\sqrt{\frac{KL[(\alpha Q + (1-\alpha)Q') || P] + \ln \frac{\mathcal{I}^K_{\text{sq}}(m)}{\delta}}{m}} <  \sqrt{\frac{KL[Q || P] + \ln \frac{\mathcal{I}^K_{\text{sq}}(m)}{\delta}}{m}}. 
\end{equation} 
Adding the above two inequalities, we get that for any $Q, Q'$ for each $\alpha \in (0, 1)$:
\begin{eqnarray}
B_{\text{sq,KL}}(\alpha Q + (1-\alpha)Q') &<& B_{\text{sq,KL}}(Q) \nonumber \\
&=& \max \lbrace B_{\text{sq,KL}}(Q),B_{\text{sq,KL}}(Q') \rbrace
\end{eqnarray}
Hence quasiconvexity holds under Case I.

\item[Case II]: If $E_{Q'}[\hat{l}] < E_{Q}[\hat{l}]$ and $KL[Q||P] = KL[Q'||P]$, then we have:
\begin{equation}
E_{\alpha Q + (1-\alpha)Q'}[\hat{l}] = \alpha E_{Q}[\hat{l}] +(1-\alpha)E_{Q'}[\hat{l}] < E_{Q}[\hat{l}] \quad \text{ for each } \alpha \in (0, 1).
\end{equation}
Since $\sqrt{\frac{KL[Q || P] + \ln \frac{\mathcal{I}^K_{\text{sq}}(m)}{\delta}}{m}}$ is quasiconvex using \eqref{eqn:quasiconvex_sqrtKL}, and also $KL[Q||P] = KL[Q'||P]$ by assumption, we can claim the following for any $Q, Q'$ for each $\alpha \in (0, 1)$:
\begin{equation}
\sqrt{\frac{KL[(\alpha Q + (1-\alpha)Q') || P] + \ln \frac{\mathcal{I}^K_{\text{sq}}(m)}{\delta}}{m}} <  \sqrt{\frac{KL[Q || P] + \ln \frac{\mathcal{I}^K_{\text{sq}}(m)}{\delta}}{m}}. 
\end{equation} 
Adding the above two inequalities, it is clear that for any $Q, Q'$ for each $\alpha \in (0, 1)$:
\begin{eqnarray}
B_{\text{sq,KL}}(\alpha Q + (1-\alpha)Q') &<& B_{\text{sq,KL}}(Q) \nonumber \\
&=& \max \lbrace B_{\text{sq,KL}}(Q),B_{\text{sq,KL}}(Q') \rbrace
\end{eqnarray}
Hence quasiconvexity holds in Case II as well.

\item[Case III]: If $E_{Q}[\hat{l}] \leq E_{Q'}[\hat{l}]$ and $KL[Q||P] > KL[Q'||P]$, such that $ B_{\text{sq,KL}}(Q') < B_{\text{sq,KL}}(Q)$. 

This implies that:
\begin{align*}
E_{Q'}[\hat{l}] + \sqrt{\frac{KL[Q' || P] + \ln \frac{\mathcal{I}^K_{\text{sq}}(m)}{\delta}}{m}} < E_{Q}[\hat{l}] + \sqrt{\frac{KL[Q|| P] + \ln \frac{\mathcal{I}^K_{\text{sq}}(m)}{\delta}}{m}} \\
\Leftrightarrow E_{Q'} [\hat{l}] < E_{Q}[\hat{l}] + \underbrace{\sqrt{\frac{KL[Q|| P] + \ln \frac{\mathcal{I}^K_{\text{sq}}(m)}{\delta}}{m}} -  \sqrt{\frac{KL[Q'|| P] + \ln \frac{\mathcal{I}^K_{\text{sq}}(m)}{\delta}}{m}}}_{\geq 0 \text{ because } KL[Q||P] > KL[Q'||P]}
\end{align*}
Now, consider the bound function at the convex combination $\alpha Q + (1 - \alpha) Q'$:
\begin{align*}
\hspace{-5mm}
B_{\text{sq,KL}}(\alpha Q + (1 - \alpha) Q') &= E_{\alpha Q + (1-\alpha)Q'} [\hat{l}] + \sqrt{\frac{KL[(\alpha Q + (1-\alpha)Q') || P] + \ln \frac{\mathcal{I}^K_{\text{sq}}(m)}{\delta}}{m}} \\
&< E_{Q'}[\hat{l}] + \sqrt{\frac{KL[Q || P] + \ln \frac{\mathcal{I}^K_{\text{sq}}(m)}{\delta}}{m}} \\
&< E_{Q}[\hat{l}] + 2\sqrt{\frac{KL[Q || P] + \ln \frac{\mathcal{I}^K_{\text{sq}}(m)}{\delta}}{m}} -  \sqrt{\frac{KL[Q'|| P] + \ln \frac{\mathcal{I}^K_{\text{sq}}(m)}{\delta}}{m}} \\
&< B_{\text{sq, KL}}(Q) + \underbrace{\sqrt{\frac{KL[Q|| P] + \ln \frac{\mathcal{I}^K_{\text{sq}}(m)}{\delta}}{m}} -  \sqrt{\frac{KL[Q'|| P] + \ln \frac{\mathcal{I}^K_{\text{sq}}(m)}{\delta}}{m}}}_{\geq 0 \text{ because } KL[Q||P] > KL[Q'||P]}
\end{align*}
For quasi-convexity to hold, we need to show that $B_{\text{sq,KL}}(\alpha Q + (1 - \alpha) Q')< B_{\text{sq,KL}}(Q)$ for any pair $Q, Q'$ and for any $\alpha \in (0, 1)$.

\item[Case IV]: If $E_{Q}[\hat{l}] \geq E_{Q'}[\hat{l}]$ and $KL[Q||P] < KL[Q'||P]$, such that $ B_{\text{sq,KL}}(Q') < B_{\text{sq,KL}}(Q)$.
This implies that: 
\begin{align*}
    E_{Q,}[\hat{l}] + \sqrt{\frac{KL[Q' || P] + \ln \frac{\mathcal{I}^K_{\text{sq}}(m)}{\delta}}{m}} < E_{Q}[\hat{l}] + \sqrt{\frac{KL[Q|| P] + \ln \frac{\mathcal{I}^K_{\text{sq}}(m)}{\delta}}{m}} \\
    \Leftrightarrow \sqrt{\frac{KL[Q'|| P] + \ln \frac{\mathcal{I}^K_{\text{sq}}(m)}{\delta}}{m}} < \underbrace{E_{Q}[\hat{l}] - E_{Q'}[\hat{l}]}_{\geq 0 \text{ under Case IV assumption}}  + \sqrt{\frac{KL[Q|| P] + \ln \frac{\mathcal{I}^K_{\text{sq}}(m)}{\delta}}{m}}
\end{align*}
Now, consider the bound function at the convex combination $\alpha Q + (1 - \alpha) Q'$:
\begin{align*}
%\text{Now, } \; 
B_{\text{sq,KL}}(\alpha Q + (1 - \alpha) Q') &= E_{\alpha Q + (1-\alpha)Q'} [\hat{l}] + \sqrt{\frac{KL[(\alpha Q + (1-\alpha)Q') || P] + \ln \frac{\mathcal{I}^K_{\text{sq}}(m)}{\delta}}{m}} \\
&< E_{Q}[\hat{l}] + \sqrt{\frac{KL[Q' || P] + \ln \frac{\mathcal{I}^K_{\text{sq}}(m)}{\delta}}{m}} \\
&< 2E_{Q}[\hat{l}] - E_{Q'}[\hat{l}]+ \sqrt{\frac{KL[Q|| P] + \ln \frac{\mathcal{I}^K_{\text{sq}}(m)}{\delta}}{m}} \\
&< B_{\text{sq, KL}}(Q) + \underbrace{E_{Q}[\hat{l}] - E_{Q'}[\hat{l}]}_{\geq 0 \text{ under Case IV assumption}}
\end{align*}
For quasi-convexity to hold, we need to show that $B_{\text{sq,KL}}(\alpha Q + (1 - \alpha) Q')< B_{\text{sq,KL}}(Q)$ for any pair $Q, Q'$ and for any $\alpha \in (0, 1)$.
\end{enumerate}

Quasi-convexity of $B_{\text{sq, KL}}(Q)$ could not be guaranteed under Cases III and IV above. But based on the computational results that we have for minimization of $B_{\text{sq, KL}}(Q)$, we observe that it has single local minimum in case of uniform prior $P$. This observation propels us to make the following claim:

\begin{conj}
The bound function, $B_{\text{sq, KL}}(Q) = \sum_{i=1}^{H} \hat{l}_iq_i + \sqrt {\frac{\sum_{i=1}^{H} q_i\ln\frac{q_i}{p_i} + \ln \frac{2\sqrt{m}}{\delta}}{m}}$ is quasi-convex when $P$ is uniform prior on $\mathcal{H}$.
\end{conj}

We seek an optimal posterior for $B_{\text{sq, KL}}(Q)$ which minimizes this bound. We use the partial KKT system to derive the fixed point equation of this bound minimization problem.

\subsection{The posterior based on fixed point scheme, $Q^{FP}_{\text{sq,KL}}$}
We can identify the minimizer for the bound minimization problem (24) (in paper) using the KKT system based on the associated Lagrangian function. 
The Lagrangian function for (24) (in paper) can be written as follows:
\begin{multline}
\mathcal{L}_{\text{sq, KL}}(Q,\mu_0,\mu_i) = \sum\limits_{i=1}^{H}\hat{l}_i q_i + \sqrt{\frac{\sum\limits_{i=1}^{H}q_i\ln\frac{q_i}{p_i} + \ln \left( \frac{\mathcal{I}^{K}_{\text{sq}}(m)}{\delta} \right) }{m}} - \mu_0 \left(\sum\limits_{i=1}^H q_i -1 \right) - \sum_{i = 1}^H \mu_i q_i
\end{multline}
Here, $\mu_0 \in \mathbb{R}$ is the Lagrange multiplier for the sum of the posterior weights, and $\mu_i$ is the Lagrange multiplier for the positivity of posterior weight, $q_i$ for all $i = 1, \ldots, H$.

\begin{theorem}
The bound minimization problem (24) (in paper) for the bound $B_{\text{sq, KL}}(Q) =\sum\limits_{i=1}^{H}\hat{l}_i q_i + \sqrt{\frac{\sum\limits_{i=1}^{H}q_i\ln\frac{q_i}{p_i} + \ln \left( \frac{\mathcal{I}^{K}_{\text{sq}}(m)}{\delta} \right) }{m}}$ has a stationary point which can be obtained as the solution to the following fixed point equation:
\begin{equation}
q_i = \frac{p_ie^{ \left(-2\sqrt{m}\hat{l}_i \sqrt{\sum_{i=1}^{H}q_i\ln\frac{q_i}{p_i} + \ln \frac{\mathcal{I}^{K}_{\text{sq}}(m)}{\delta}}\right)}}{\displaystyle\sum\limits_{i =1}^{H} p_i e^{ \left(-2\sqrt{m}\hat{l}_i \sqrt{\sum_{i=1}^{H}q_i\ln\frac{q_i}{p_i} + \ln \left( \frac{\mathcal{I}^{K}_{\text{sq}}(m)}{\delta} \right)} \right)}} 
%\label{eqn:optQ_sqKL}
\end{equation}
\end{theorem}

\begin{proof}
Differentiating Lagrange $\mathcal{L}_{\text{lin, KL}}$ with respect to primal variables $q_i$s and dual variable $\mu_0$, we get:
\begin{align}
\frac{\partial \mathcal{L}_{\text{sq, KL}}}{\partial q_i} &= \hat{l}_i + \frac{1}{2\sqrt{m} \sqrt{\sum_{i=1}^{H}q_i\ln\frac{q_i}{p_i} + \ln \left( \frac{\mathcal{I}^{K}_{\text{sq}}(m)}{\delta} \right)}} \cdot \left(1 + \ln \frac{q_i}{p_i} \right) \nonumber\\
&\hspace{7cm} - \mu_0 -\mu_i  \quad \forall i = 1, \ldots, H \label{eqn:KKT1sqKL}\\
\frac{\partial \mathcal{L}_{\text{sq, KL}}}{\partial \mu_0} &= \sum_{i=1}^{H}q_i -1 \label{eqn:mu0KKTsqKL}
\end{align}

We assume that $q_i > 0$ for all $i = 1, \ldots, H$, since otherwise $\ln q_i = \ln(0)$ is undefined. Even if we use fact that $\lim_{x \rightarrow 0^{+}} \ln x = - \infty$, the KKT condition \eqref{eqn:KKT1sqKL} will mean that the dual variable $\mu_i$ is infeasible. Therefore our assumption holds true for a stationary point. Due to complementary slackness conditions, we have $\mu_i q_i = 0$ which implies $\mu_i = 0$ since $q_i > 0$ for all $i = 1, \ldots, H$  (\textit{by assumption}). 

At optimality, the derivatives of the Lagrange function $\mathcal{L}_{\text{sq, KL}}$ should be set to zero. From \eqref{eqn:KKT1sqKL}, we have:
\begin{align}
&\hat{l}_i + \frac{1}{2\sqrt{m} \sqrt{\sum_{i=1}^{H}q_i\ln\frac{q_i}{p_i} + \ln \left( \frac{\mathcal{I}^{K}_{\text{sq}}(m)}{\delta} \right)}} \cdot \left(1 + \ln \frac{q_i}{p_i} \right)- \mu_0 = 0 \quad \forall i = 1, \ldots, H \nonumber\\
\Rightarrow \; &1 + \ln \frac{q_i}{p_i} = 2\sqrt{m}(\mu_0 - \hat{l}_i)  \sqrt{\sum_{i=1}^{H}q_i\ln\frac{q_i}{p_i} + \ln \left( \frac{\mathcal{I}^{K}_{\text{sq}}(m)}{\delta} \right)}  \quad \forall i = 1, \ldots, H  \nonumber\\
\Rightarrow \; & q_i = p_i e^{ \left(2\sqrt{m}(\mu_0 - \hat{l}_i)  \sqrt{\sum_{i=1}^{H}q_i\ln\frac{q_i}{p_i} + \ln \left( \frac{\mathcal{I}^{K}_{\text{sq}}(m)}{\delta} \right)} \right) - 1}  \quad \forall i = 1, \ldots, H  \label{eqn:qderive_sqKL}
\end{align}
Setting the derivative at \eqref{eqn:mu0KKTsqKL} to zero, we have:
\begin{align}
&\sum_{i= 1}^{H} q_i = 1 \nonumber\\
\Rightarrow\; &\sum_{i =1}^{H} p_i e^{ \left(2\sqrt{m}(\mu_0 - \hat{l}_i)  \sqrt{\sum_{i=1}^{H}q_i\ln\frac{q_i}{p_i} + \ln \left( \frac{\mathcal{I}^{K}_{\text{sq}}(m)}{\delta} \right)} \right) - 1} = 1\nonumber\\
\Rightarrow & e^{ \left(2\sqrt{m}\mu_0 \sqrt{\sum_{i=1}^{H}q_i\ln\frac{q_i}{p_i} + \ln \left( \frac{\mathcal{I}^{K}_{\text{sq}}(m)}{\delta} \right)} \right) - 1} \left( \sum_{i =1}^{H} p_i e^{ \left(-2\sqrt{m}\hat{l}_i \sqrt{\sum_{i=1}^{H}q_i\ln\frac{q_i}{p_i} + \ln \left( \frac{\mathcal{I}^{K}_{\text{sq}}(m)}{\delta} \right)} \right)} \right) = 1\nonumber\\
\Rightarrow &e^{ \left(2\sqrt{m}\mu_0 \sqrt{\sum_{i=1}^{H}q_i\ln\frac{q_i}{p_i} + \ln \left( \frac{\mathcal{I}^{K}_{\text{sq}}(m)}{\delta} \right)} \right) - 1} = \frac{1}{\displaystyle\sum\limits_{i =1}^{H} p_i e^{ \left(-2\sqrt{m}\hat{l}_i \sqrt{\sum_{i=1}^{H}q_i\ln\frac{q_i}{p_i} + \ln \left( \frac{\mathcal{I}^{K}_{\text{sq}}(m)}{\delta} \right)} \right)}} \label{eqn:mu0_sqKL} \\
\Rightarrow \; & \mu_0 = \frac{1 - \ln \left(\sum\limits_{i =1}^{H} p_i e^{ \left(-2\sqrt{m}\hat{l}_i \sqrt{\sum_{i=1}^{H}q_i\ln\frac{q_i}{p_i} + \ln \left( \frac{\mathcal{I}^{K}_{\text{sq}}(m)}{\delta} \right)} \right)} \right)}{2\sqrt{m} \sqrt{\sum_{i=1}^{H}q_i\ln\frac{q_i}{p_i} + \ln \left( \frac{\mathcal{I}^{K}_{\text{sq}}(m)}{\delta} \right)}}
\end{align}
Combining the above two equations \eqref{eqn:qderive_sqKL} and \eqref{eqn:mu0_sqKL}, we get the following equation in variable $q_i$s:
\begin{equation}
q_i = \frac{p_ie^{ \left(-2\sqrt{m}\hat{l}_i \sqrt{\sum_{i=1}^{H}q_i\ln\frac{q_i}{p_i} + \ln \frac{\mathcal{I}^{K}_{\text{sq}}(m)}{\delta}}\right)}}{\displaystyle\sum\limits_{i =1}^{H} p_i e^{ \left(-2\sqrt{m}\hat{l}_i \sqrt{\sum_{i=1}^{H}q_i\ln\frac{q_i}{p_i} + \ln \left( \frac{\mathcal{I}^{K}_{\text{sq}}(m)}{\delta} \right)} \right)}} 
\label{eqn:optQ_sqKL}
\end{equation}
Note that the right hand side involves an implicit function of variable $q_i$s. Hence the above is a fixed point equation. 
It can be easily verified from \eqref{eqn:optQ_sqKL} that all $q^{FP}_{i, \text{sq, KL}} > 0$ and they sum up to 1. Hence $q^{FP}_{i, \text{sq, KL}} $ is a feasible solution to the bound minimization problem (24) (in paper). Also, it is derived using the KKT conditions, hence it is a stationary point.
\end{proof}

\section{Optimal PAC-Bayesian posteriors for a finite set of SVM classifiers \label{chap:optQSVM}}

Support vector machines (SVMs) are convex classification algorithms with a regularization parameter $\lambda > 0,$ which controls the trade off between the training error and learner complexity. We want to recommend values of parameter $\lambda$ corresponding to classifiers with `good' generalization performance. To do this, we use the PAC-Bayesian framework. The PAC-Bayesian optimal posterior yields a stochastic SVM that has a tight upper bound on the averaged true risk. A stochastic SVM makes predictions by choosing a $\lambda$ value randomly from a prefixed set of values according to the governing distribution, determining the classifier corresponding to this $\lambda$ value and using this classifier to predict the label of an unknown example. Since PAC-Bayesian posterior is determined on a fixed set of classifiers, we determine beforehand our SVM classifiers for the values in the set of regularization parameter values. A stochastic SVM is preferred over a deterministic SVM since the former is robust to sample biases as illustrated in Table \ref{tab:top30perclambda_mamm} and performs well on an average with high probability, as shown here. 

We report the solver outputs and fixed point (FP) solutions for bound minimization problems arising from different combinations of the distance functions, $\phi$s with KL-divergence measure. While some of them are convex and have a closed form expression for the global optimum, others are non-convex and have a fixed point characterization, which converges to a local minimizer. We observe that fixed point scheme always converges to a local/global minimizer even when the solver fails to solve the bound minimization problem. 
%Also, when we use the KL-distance function and the classifier set is sufficiently large $\ldots$. 

We first describe about the datasets that we have considered for our computations, the scheme used to generate classifiers and compute risk values and then compare the optimal PAC-Bayesians posteriors obatined using the FP scheme and the solver for the different distance functions.

\subsection{Datasets categorization and computation scheme} \label{secn:datacategory.compscheme}
We did the computations on some real datasets with binary classes from UCI repository \cite{UCI:2017}. The details about the number of features, number of examples and class distribution of these datasets are listed in Table \ref{tab:dataset.details}. Care was taken to include datasets with various attributes -- small to moderate number of examples (306 examples to 5463 examples) and small to moderate number of features (3 features to 57 features). We have datasets with various combinations -- small number of features with small number of examples (Bupa and Haberman), small number of features with moderate number of examples (Banknote and Mammographic), moderate number of features with small number of examples (Wdbc and Ionosphere) and moderate number of features with moderate number of examples (Spambase and Waveform). There is an even distribution of balanced datasets, that is, datasets with almost same number of positive and negative examples (Bupa, Mammographic, Banknote and Waveform) and imbalanced datasets (Spambase, Wdbc, Mushroom, Ionosphere, Haberman). These datasets span a variety ranging from almost linearly separable (Banknote, Mushroom and Wave datasets) to moderately inseparable (Wdbc, Mammographic and Ionosphere datasets) to inseparable data (Spambase, Bupa and Haberman datasets).

\begin{table}[]
\centering
\small
\setlength{\tabcolsep}{0.1em}
\begin{tabular}{|c|c|c|c|c|c|c|}
\hline
\textbf{Dataset} & \makecell{ \textbf{Number of} \\ \textbf{features}, $n$} & \makecell{\textbf{Number of } \\ \textbf{examples}}& \textbf{Pos/Neg} & \makecell{ \textbf{Training} \\ \textbf{set size}, $m$} & \makecell{ \textbf{Validation} \\ \textbf{set size}, $v$} & \makecell{ \textbf{Test} \\ \textbf{set size}, $t$} \\ \hline
\textbf{Spambase} & 57 & 4601 & 2788/1813 & 1840 & 1840 & 921 \\ \hline
\textbf{Bupa} & 6 & 345 & 176/169 & 138 & 138 & 69 \\ \hline
\textbf{Mammographic} & 5 & 830 & 427/403 & 332 & 332 & 166 \\ \hline
\textbf{\makecell{Wdbc}} & 30 & 569 & 357/212 & 227 & 227 & 115 \\ \hline
\textbf{Banknote} & 4 & 1372 & 610/762 & 548 & 549 & 275 \\ \hline
\textbf{Mushroom} & 22 (\textit{116 \footnote{after one-hot encoding for categorical features}})  & 5643 \footnote{after removing the rows with missing values from the data}& 3488/2155 & 2257 & 2257 & 1129 \\ \hline
\textbf{Ionosphere} & 34 & 351 & 225/126 & 140 & 140 & 71 \\ \hline
\textbf{Waveform} & 40 & 3308 \footnote{number of examples when class `0' is removed} & 1653/1655 & 1323 & 1323 & 662 \\ \hline
\textbf{Haberman} & 3 & 306 & 225/81 & 122 & 122 & 62\\ \hline
\end{tabular}
\caption[Details of various UCI datasets used for computational experiments]{Details of various UCI datasets used for computational experiments. We list the number of features $n$, total number of examples with distribution into positive and negative classes for each dataset. We also give the number of examples in training, validation and test sets, according to the random partition created by 0.4:0.4:0.2 ratio of the total dataset size.}
\label{tab:dataset.details}
\end{table}

We consider a finite set of SVM regularization parameter values $\Lambda = \lbrace \lambda_i \rbrace_{i = 1}^{H}$, of the values of the regularization parameter, say, between $0$ and an upper bound $\lambda_0 > 0$, since small values of $\lambda_i$'s are preferable. $H$ denotes the number of regularization parameter values used for training the SVMs. We took the set $\Lambda = \lbrace 0.1, 0.11, \ldots, 20\rbrace$ at a granularity of 0.01. The smallest $\lambda$ value in the set is taken to be strictly positive and slightly away from zero. This is because, for very small $\lambda$ values, the corresponding SVMs tend to be in proximity by due to continuity property of SVM classifier with respect to the regularization parameter ($\lambda$), and hence have same/similar error rates. In fact, for infinitesimally small values, say $\lambda < 10^{-8}$, the SVM QP may encounter numerical instabilities. Similarly, very large values of $\lambda$ should be avoided since they yield `bad' classifiers with considerably high error rates.

Each of these datasets was partitioned such that 80\% of the examples formed a composition of training set and validation set (in equal proportion) used for constructing the set $\mathcal{H} = \lbrace h(\lambda_i) | \lambda_i \in \Lambda \; \forall i = 1, \ldots, H \rbrace$ of SVM classifiers and remaining 20\% used for computing their test error rates. The training set size ($m$), validation set size ($v$) and test set size ($t$) for the datasets are given in Table \ref{tab:dataset.details}. Typically, $m:v:t = 0.4:0.4:0.2$. The role of the validation set is to compute the empirical risk $\hat{l}_i$ of the SVM $h(\lambda_i) \in \mathcal{H}$ which will be used for deriving the PAC-Bayesian bound. Training error cannot be considered as empirical risk for a classifier in our set up since PAC-Bayesian theorem requires that the classifiers should be fixed and should not rely on training examples \cite{begin2016pac}. This is needed to define a sample independent, classifier set independent right hand side threshold $\mathcal{I}^{K}_{\phi}(m)$ for the PAC-Bayesian bound which holds uniformly for all samples. %(Please refer to the derivation of $\mathcal{I}^{K}_{\phi}(m)$ in Equation \eqref{eqn:I_K_phi.S_h_swap}). 

The classifier set,  $\mathcal{H} = \lbrace h(\lambda_i) | \lambda_i \in \Lambda \; \forall i = 1, \ldots, H \rbrace$, consists of RBF kernel SVMs generated from these datasets with regularization parameter values in the set $\Lambda= \lbrace 0.1, 0.11, \ldots, 20\rbrace$ chosen above. We follow the scheme provided in \cite{begin2016pac, thiemann2016quasiconvexPACB} to generate our classifier set $\mathcal{H}$. A common test set of size $t$ is kept aside beforehand. The remaining subset of the dataset constitutes a training and validation set composition with $m + v$ examples. Each classifier $h(\lambda_i) \in \mathcal{H}$ is trained on $m$ training examples subsampled from this composite set and validated on the remaining $v$ examples. Overlaps between training sets of different classifiers are allowed. Same is true for their validation sets. Any two validation sets have a difference of at least one example, which means that the validation errors on these sets are i.i.d random variables. The validation error of SVM $h(\lambda_i) \in \mathcal{H}$ is taken to be its empirical risk value, $\hat{l}_i$.

Depending on the dataset, these SVMs have different ranges and degrees of variation in their empirical risk values. Generally, these empirical risk values show an increasing trend as the value of $\lambda$ increases, but the rate of growth differs from dataset to dataset. Some datasets show steady increase with stabilized values (Banknote, Haberman, Mushroom and Wave), while others have steep increase and haphazard values (Bupa, Ionosphere and Spambase). Gradual increment might be accompanied by lot of variation (Mammographic dataset) and stabilzed nature may not hold for the whole range of $\lambda$ (Wdbc dataset, with low, stable values for $\lambda \leq 15$ and a heavy variation for $\lambda > 15$) This phenomenon can be captured by variance of the empirical values across its range, but the variance of the empirical risk values across the subintervals of $\Lambda$ is equally important to quantify the rate of increase. For a visual illustration of the variance in the empirical risk values and test error rates of the SVMs that we have constructed on the different UCI datasets, please refer to Figure \ref{fig:emprisk.testerr.BBHIM} and Figure \ref{fig:emprisk.testerr.MSWW}. 

%Mammographic dataset has significant training error rates for the base classifiers with low variation. Spambase dataset has a highly variable training error rate among the base classifiers. Wisconsin breast cancer dataset and Banknote dataset are almost separable, even admitting zero training error for classifiers generated using very low values of the regularization parameter.

\paragraph{Computational Framework}
SVM QP (with RBF kernels) was implemented using \texttt{ksvm} function in \texttt{kernlab} package \cite{kernlab} in { \em R (version 3.1.3 (2015-03-09))}. The Gaussian width parameter is estimated  by \texttt{kernlab} using \texttt{sigest} function which estimates the 0.1 and 0.9 quantile of distance between the points in the data.

The optimization problem for finding the optimal posterior that minimizes the PAC-Bayesian bound was implemented in AMPL Interface and solved using \texttt{Ipopt} software package {\em (version 3.12 (2016-05-01))} \cite{ipopt}. All the computations were done on a machine equipped with 12 Intel Xeon 2.20 GHz cores and 64 GB RAM.

\begin{figure}
    \includegraphics[width = 0.48\textwidth]{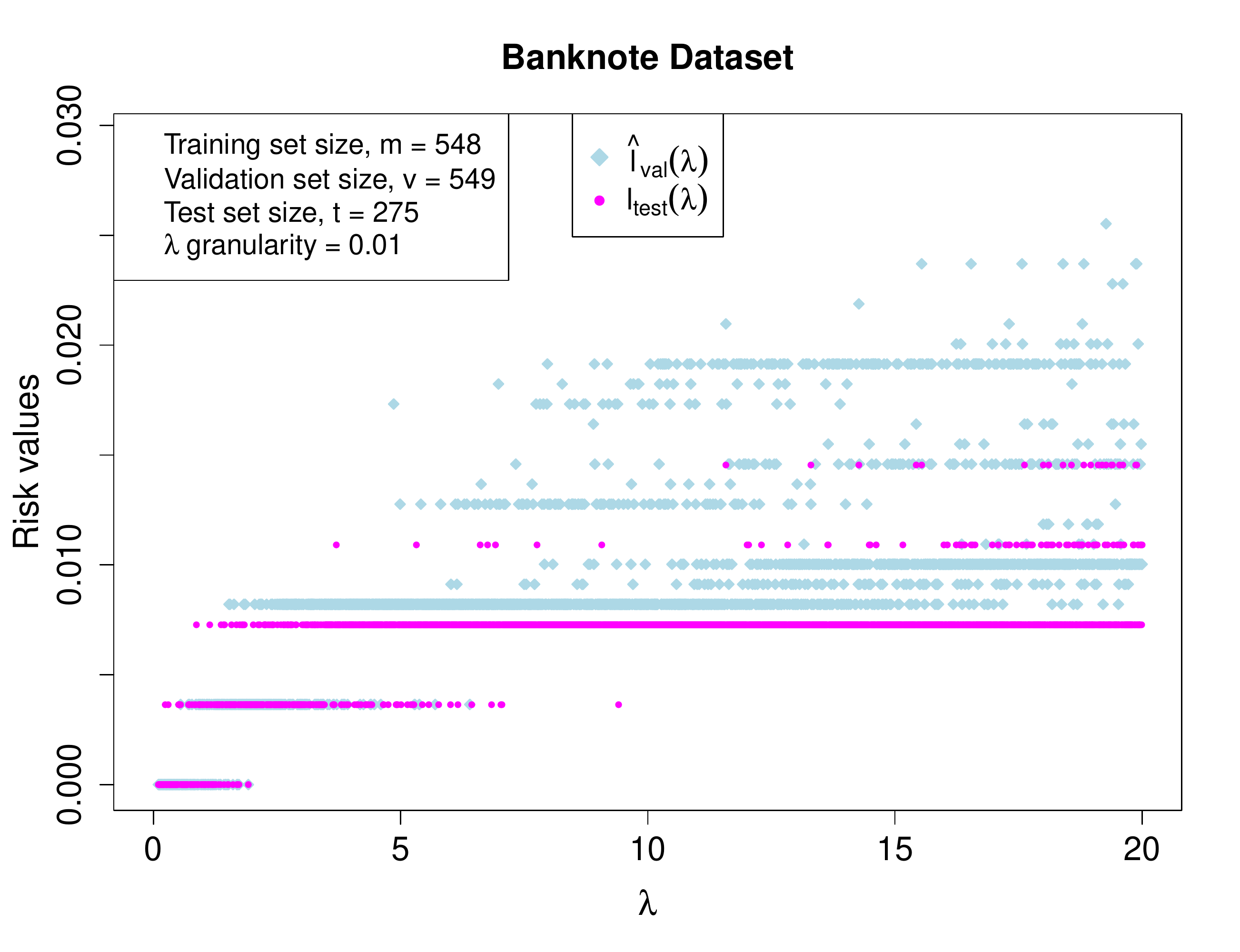}
    \includegraphics[width = 0.48\textwidth]{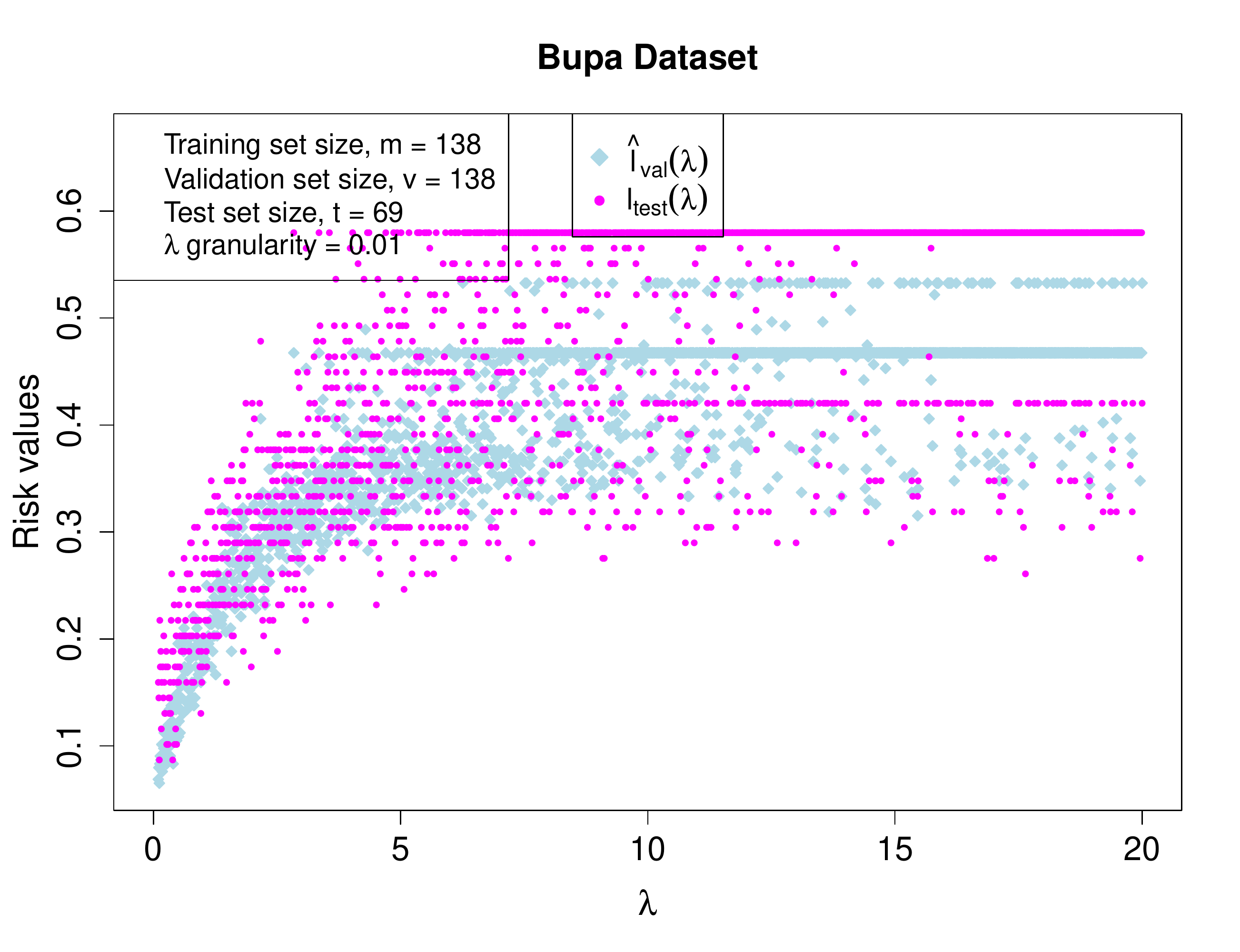}
    \includegraphics[width = 0.48\textwidth]{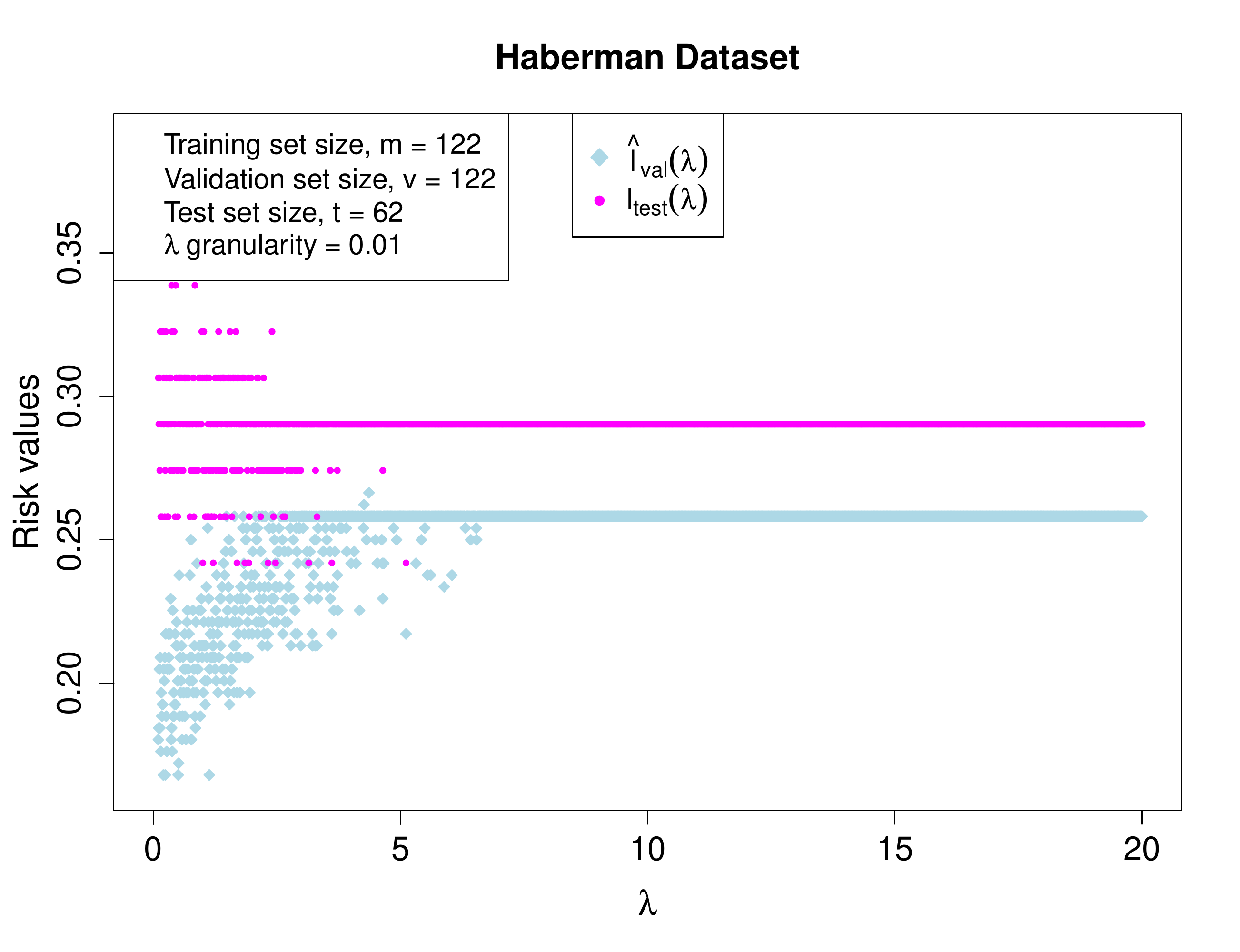}
    \includegraphics[width = 0.48\textwidth]{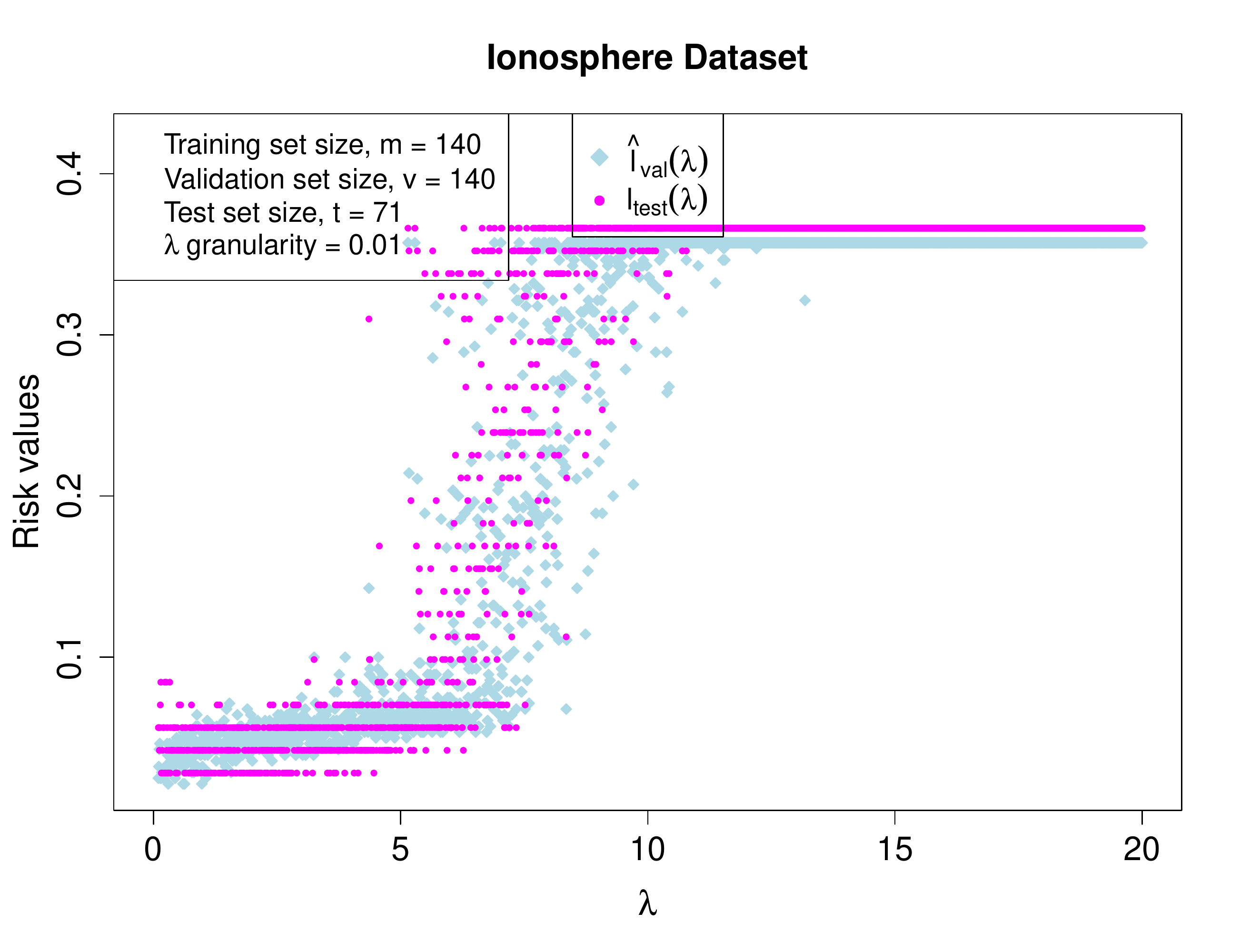}
    \includegraphics[width = 0.48\textwidth]{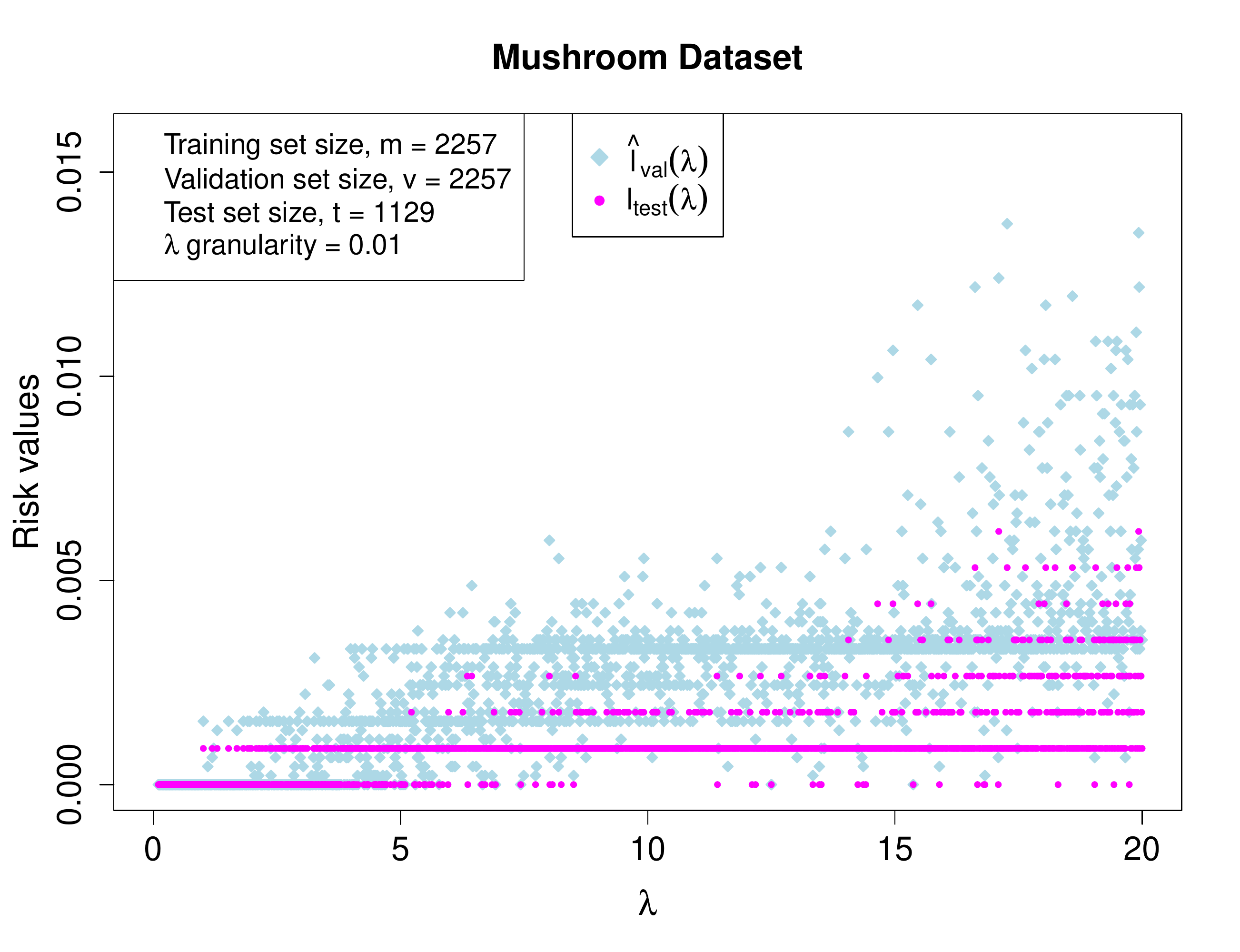}
\caption[Empirical risk values and test error rates of SVMs  on Banknote, Bupa, Haberman, Ionosphere and Mushroom datasets]{Empirical risk values (validation errors) $\hat{l}_{\text{val}}(\lambda)$ and test errors $\hat{l}_{\text{test}}(\lambda)$ of SVMs corresponding to regularization parameter $\lambda \in \Lambda = \lbrace 0.0, 0.11, \ldots, 20 \rbrace$ trained on 40\% of the dataset and validated on the other 40\% of the dataset, with test errors computed on the remaining 20\% of the dataset as described in Section \ref{secn:datacategory.compscheme}.} \label{fig:emprisk.testerr.BBHIM}
\end{figure}

\begin{figure}[]
    \includegraphics[width = 0.48\textwidth]{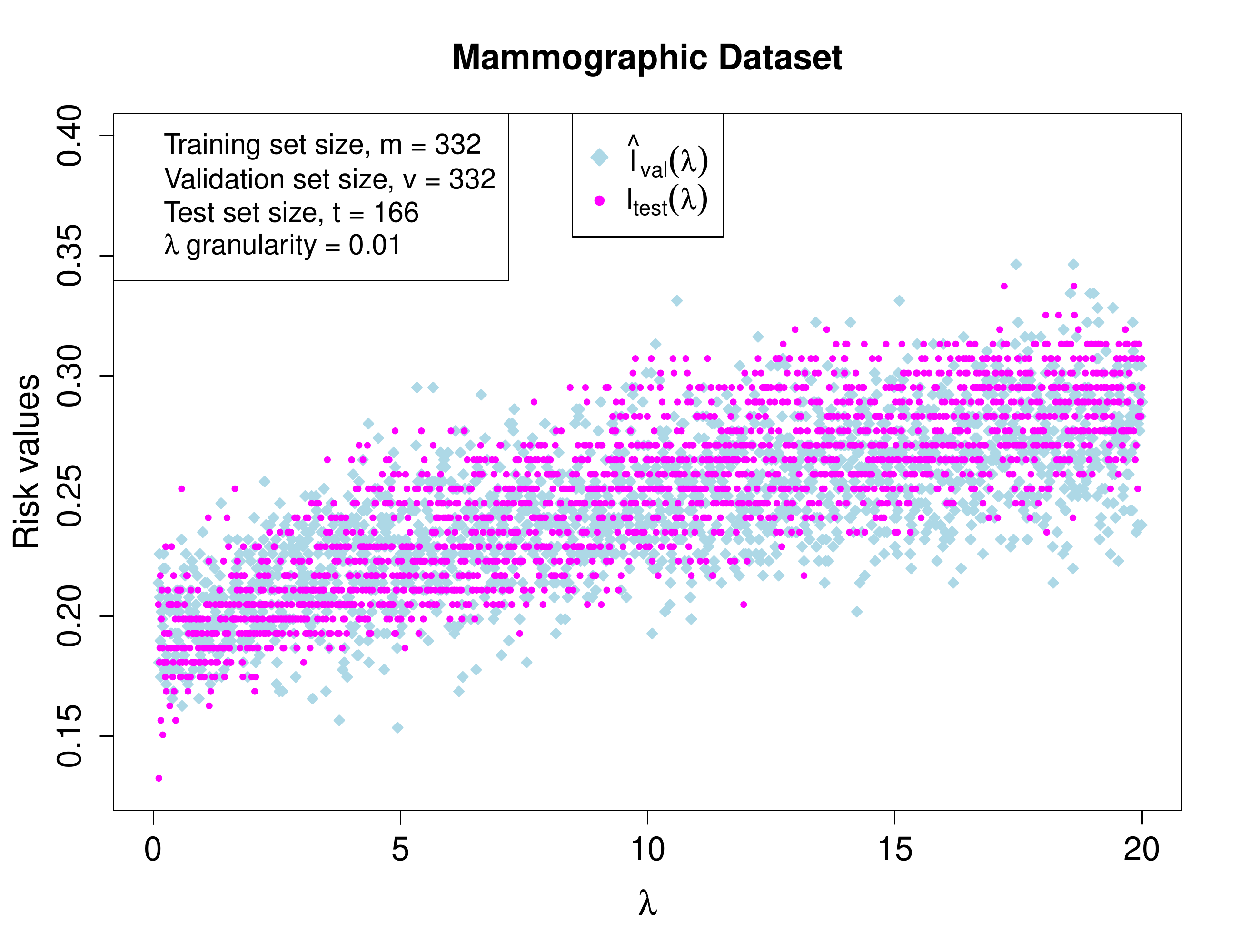}    
    \includegraphics[width = 0.48\textwidth]{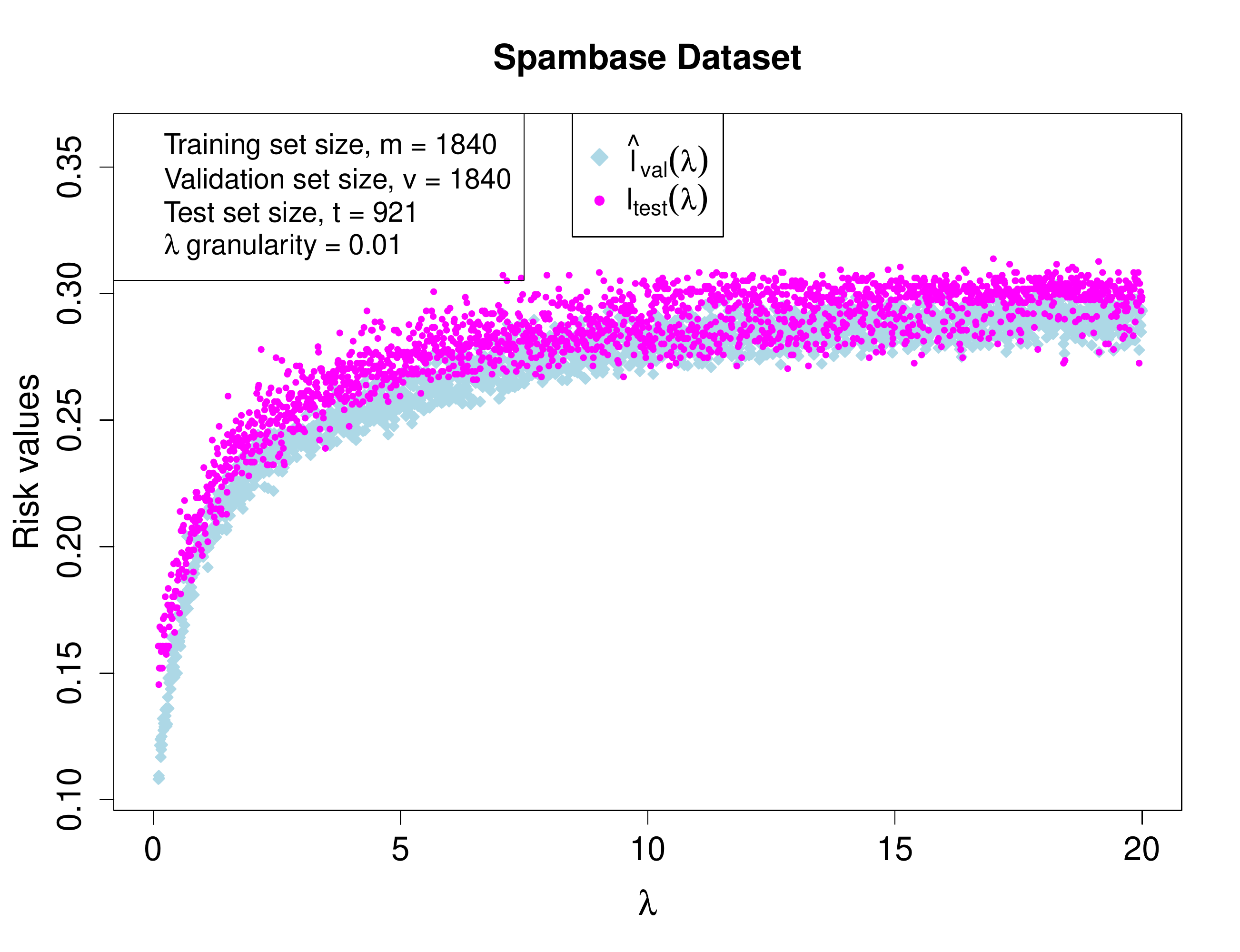}
    \includegraphics[width = 0.48\textwidth]{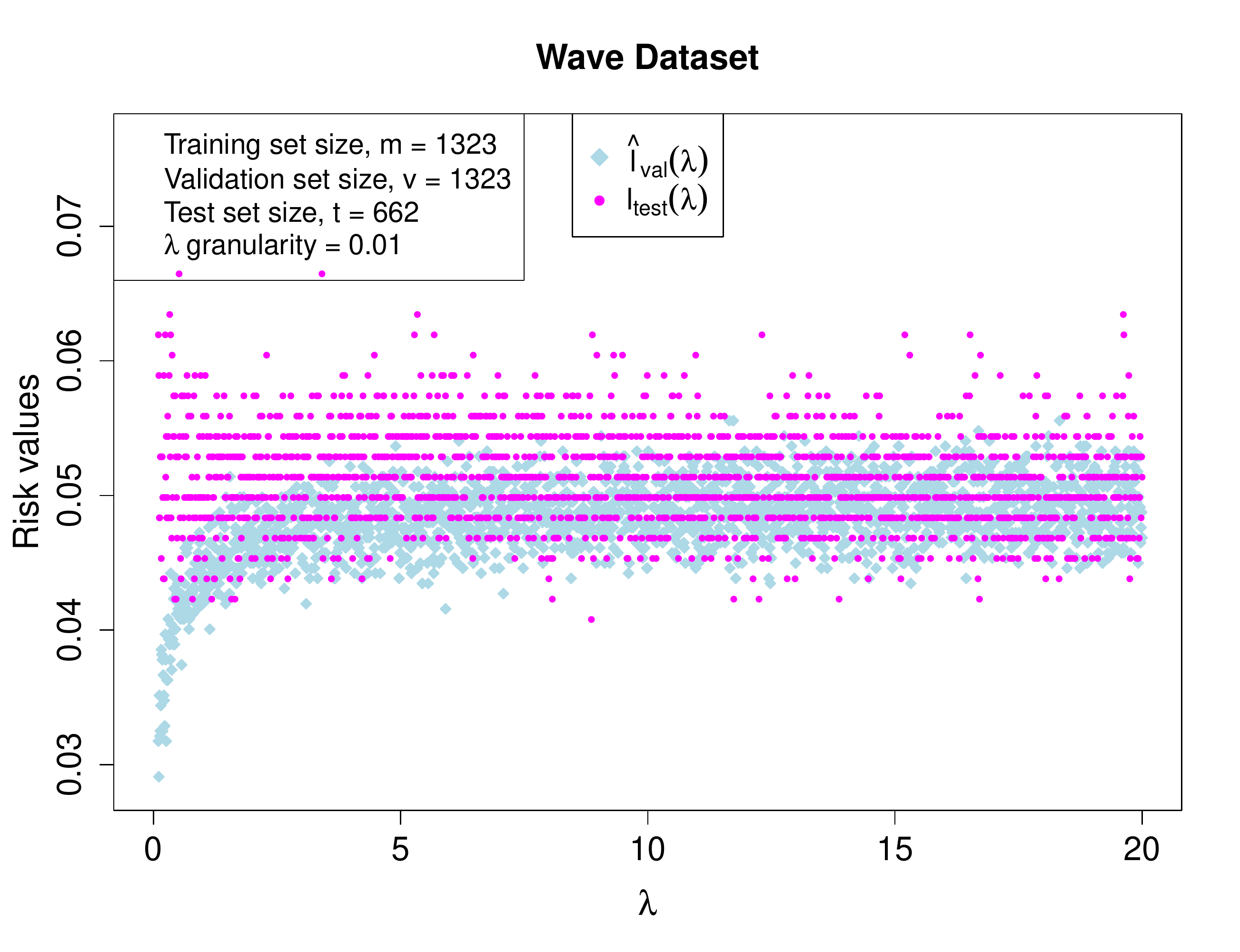}
    \includegraphics[width = 0.48\textwidth]{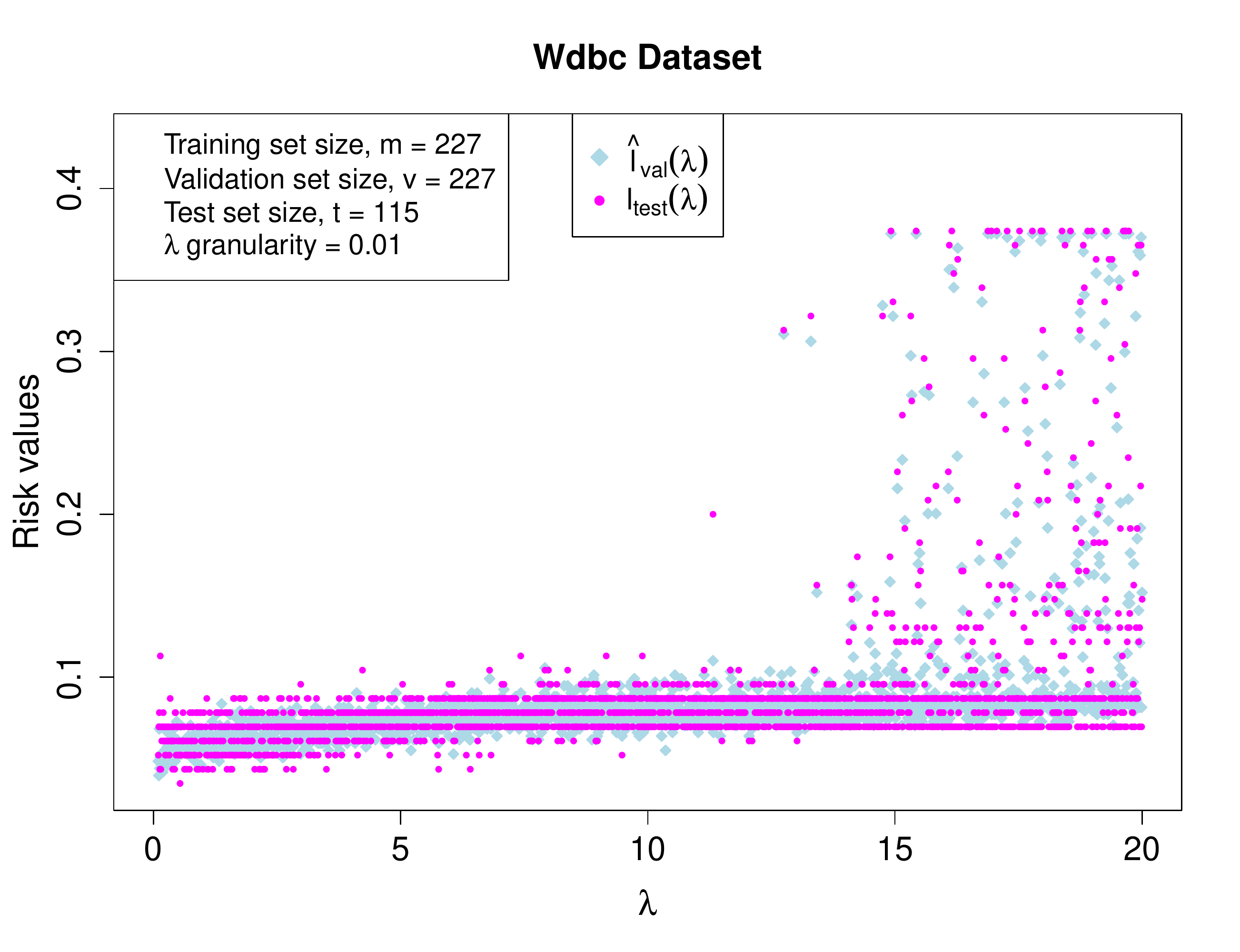}
    \caption[Empirical risk values and test error rates of SVMs on Mammographic, Spambase, Wave and Wdbc datasets]{Empirical risk values (validation errors) $\hat{l}_{\text{val}}(\lambda)$ and test errors $\hat{l}_{\text{test}}(\lambda)$  of SVMs trained on 40\% of the dataset and validated on the other 40\% of the dataset, with test errors computed on the remaining 20\% of the dataset as described in Section \ref{secn:datacategory.compscheme}.} \label{fig:emprisk.testerr.MSWW}
\end{figure}

\section{Comparing various PAC-Bayesian optimal posteriors}
In our analysis with finite classifier set, we have determined optimal posterior minimizing the PAC-Bayesian bounds formed from combinations of different distance functions and divergence measures. These are illustrated in the previous section. Mainly, five distance functions (between the averaged empirical risk and averaged true risk of a stochastic classifier) were considered: KL-divergence as distance function, its Pinsker's approximation and a sixth degree polynomial approximation; linear and squared distances.  

The posterior weight, $q^{\ast}_{i, \phi, \text{div}}$, is negatively proportional to the empirical risk, $\hat{l}_i$, of the classifier in the support set, but the constant of proportionality is different in the two classes. The optimal posteriors corresponding to the class derived using KL-divergence measure exhibit exponentially decreasing weights as the empirical risk increases, and generally have full support (entire classifier set). 

To quantify the level of concentration that these posteriors have on their supports, we use Herfindahl-Hirschman Index (HHI) \cite{Hirschman1945HHI, HHIwiki}, which is perhaps the most widely used measure of economic concentration.  It is defined as the sum of the squares of the market shares of the firms within the industry (sometimes limited to the 50 largest firms), where the market shares are expressed as fractions. For probability distributions, HHI is equivalent to their $\ell_2$-norm.

In our computations, we observe that the posteriors $Q^{\ast}_{\phi, \text{KL}}$ have high HHI, which indicates that they have more concentration around the low values of $\hat{l}_i$s even though they have full support. They display a greedy behaviour towards classifiers (regularization parameter values) yielding low sample errors.  This explains why such posteriors have a good test set performance. This behaviour hints at an underlying regularization done by the divergence function that we use in the PAC-Bayesian bound.

\begin{table}[]
{\footnotesize
\begin{center}
\setlength{\tabcolsep}{0.08em}
\begin{tabular}{|c|c c c c c||c c c c c|}
\hline
\textbf{Dataset} & \multicolumn{5}{c||}{\textbf{PAC-Bayesian Bound}, $B^{\ast}_{\phi, \text{KL}}$ } & \multicolumn{5}{c|}{\textbf{Average Test Error}, $T_{\phi, \text{KL}}$} \\
\cline{2-11}
 & $B^{\ast}_{\text{lin, KL}}$ & $B^{\ast}_{\text{sq, KL}}$ & $B^{\ast}_{\text{P, KL}}$ & $B^{\ast}_{\text{CH, KL}}$ &$B^{\ast}_{\text{kl, KL}}$ & $T_{\text{lin, KL}}$ & $T_{\text{sq, KL}}$ & $T_{\text{P, KL}}$ & $T_{\text{CH, KL}}$ & $T_{\text{kl, KL}}$
\\ \hline
Spambase & NaN & 0.20289 & 0.17671 & 0.18279 & \textbf{0.15737}\textcolor{magenta}{$\star$} & \textbf{0.10206} & \textbf{0.10353} & \textbf{0.10277} &\textbf{ 0.10263} & \textbf{0.10231}\textcolor{magenta}{$\star$}\\\hline
Bupa & \textbf{0.29382} & \makecell{ 0.40292 \\ \textit{0.36536} } & 0.31596 & 0.32896 & \textbf{0.27439}\textcolor{magenta}{$\star$} & \textbf{0.14139} & \makecell{ \textbf{0.15103} \\ \textit{0.15400}} & \textbf{0.14425} & \textbf{0.14269} & \textbf{0.13738}\textcolor{magenta}{$\star$}\\\hline
\makecell{Mammographic} & 0.31857 & \makecell{0.35706 \\\textit{ 0.32592} } &\textbf{ 0.30442} & 0.31596 & \textbf{0.28583}\textcolor{magenta}{$\star$} & \textbf{0.13805} & \makecell{\textbf{0.13934} \\  \textit{\textbf{0.13847}}} & \textbf{0.14015} & \textbf{0.14008} & \textbf{0.13978}\textcolor{magenta}{$\star$} \\ \hline
Wdbc & 0.20369 & \makecell{ 0.25657 \\ \textit{0.21754 }} & 0.19908 & 0.21318 & \textbf{0.14237} &\textbf{ 0.03315} & \makecell{ \textbf{0.03168} \\ \textit{\textbf{0.03168}}} & \textbf{0.03192} &\textbf{ 0.03209} &\textbf{0.03351}\\ \hline
Banknote & 0.13371 & \makecell{0.12752 \\ \textit{0.09855} } & 0.09094 & 0.10241 & \textbf{0.01758} & 0.00030 & \makecell{ 0.00103 \\ \textit{0.00112 }} & 0.00087 &0.00081  &\textbf{5.5e-05}\\ \hline
Mushroom & NaN & 0.06388 & 0.04521 & 0.05226 & \textbf{0.00415} & 2.29e-05 & 6.31e-05 & 5.8e-05 & 5.61e-05 & \textbf{1.1e-05}\\ \hline
Ionosphere & 0.20024 & \makecell{ 0.28773 \\ \textit{ 0.24171}} & 0.21540 & 0.23470 & \textbf{0.13208} & \textbf{0.07174} & \makecell{ \textbf{ 0.07236} \\ \textbf{0.07247}} & \textbf{0.07212} & \textbf{0.07202} & \textbf{0.07059}\\ \hline
Waveform & NaN & 0.12990 & 0.10529 & 0.11355 & \textbf{0.07254} & \textbf{0.05138} & \textbf{0.05231} & \textbf{0.05219} & \textbf{0.05212} & \textbf{0.05152}\\ \hline
Haberman & \textbf{0.37065} & \makecell{ 0.47695 \\ \textit{0.43052}} & \textbf{0.39487} & \textbf{0.40945} & \textbf{0.37762}\textcolor{magenta}{$\star$} & \textbf{0.29485} & \makecell{ \textbf{0.28140} \\ \textit{\textbf{0.28000} }} & \textbf{0.29101} & \textbf{0.29341} & \textbf{0.26900}\textcolor{magenta}{$\star$}\\ 
\hline
\end{tabular}
\end{center}
}
\caption[PAC-Bayesian bounds and averaged test error rates for $Q^{\ast}_{\phi,\text{KL}}$]{\textbf{PAC-Bayesian bounds and averaged test error rates for $Q^{\ast}_{\phi,\text{KL}}$} We compare the bound values $B^{\ast}_{\phi,\text{KL}}$ and average test error rates $T_{\phi,\text{KL}}$ of the optimal posteriors due to five distance functions, $\phi$:  KL-divergence $kl$, its Pinsker's approximation $\phi_\text{P}$ and a sixth degree polynomial approximation $\phi_\text{CH}$; linear $\phi_{\text{lin}}$ and squared distances $\phi_{\text{sq}}$ for $H=500$ SVM classiifers. For large sample size ($m \geq 1028$), the constant $\mathcal{I}^K_\text{lin}(m)$ cannot be computed due to storage limitations for floating point numbers and in that case, $B^{\ast}_{\text{lin,KL}}$ is denoted by NaN. $B^\ast_\text{sq, KL}$ and corresponding $Q^\ast_\text{sq, KL}$ were determined using two values:  $2\sqrt{m}$ (in regular font face) and $\mathcal{I}^K_\text{sq}(m)$ (in italicized font face). $\mathcal{I}^K_\text{sq}(m)$ cannot be computed for datasets with high sample size ($m \geq 1028$) due to storage limitation on floating point numbers. Hence we have only one bound value for such datasets (namely Spambase, Mushroom and Waveform) which is computed using $2\sqrt{m}$. \textcolor{magenta}{$\star$} refers to values obtained using fixed point(FP) equation because the solver \texttt{Ipopt} does not converge to a solution. Lowest 10\% bound values and test error rates for each dataset are denoted in bold face. KL-distance has the tightest bound and lowest 10\% error rate for almost all the datasets, but is computationally expensive and has multiple local minima. Between the approximations $\phi_\text{P}$ and $\phi_\text{CH}$, the latter has lower test error values but a slightly complicated bound evaluation. $\phi_\text{sq}$ is ranked lowest on bound values and test error rates, followed by $\phi_\text{CH}$ and $\phi_\text{P}$. $\phi_\text{sq}$ and $\phi_\text{P}$ are related by a scaling ($\phi_\text{P} = 2\phi_\text{sq}$). $\phi_\text{P}$ provides a lower bound value than that of $\phi_\text{sq}$, but both have comparable test set performances with differences of at most 3\%. The global solution for $\phi_\text{lin}$ has second lowest bound value for all the datasets considered (except for the ones where $m \geq 1028$, namely, Spambase, Mushroom and Waveform, where the bound $B^{\ast}_{\text{lin}, KL}$ cannot be computed) and also has the lowest 10\% test error rates for most of the datasets. All the five distance functions have good generalization performance (lowest 10\% test error values) on most of the datasets considered, except for Bupa dataset and two almost separable datasets, Banknote and Mushroom, where $\phi_\text{lin}$ and $\phi_\text{kl}$ do better than other three $\phi$s.\label{tab:phiKLCompare}}
\end{table}

\begin{table}[]
{\footnotesize
\begin{center}
\setlength{\tabcolsep}{0.1em}
\begin{tabular}{|c|c c c c c||c c c c c|}
\hline
\textbf{Dataset} & \multicolumn{5}{c||}{\textbf{PAC-Bayesian Bound}, $B^{\ast}_{\phi, \text{KL}}$ } & \multicolumn{5}{c|}{\textbf{Average Test Error}, $T_{\phi, \text{KL}}$} \\
\cline{2-11}
 & $B^{\ast}_{\text{lin, KL}}$ & $B^{\ast}_{\text{sq, KL}}$ & $B^{\ast}_{\text{P, KL}}$ & $B^{\ast}_{\text{CH, KL}}$ &$B^{\ast}_{\text{kl, KL}}$ & $T_{\text{lin, KL}}$ & $T_{\text{sq, KL}}$ & $T_{\text{P, KL}}$ & $T_{\text{CH, KL}}$ &$T_{\text{kl, KL}}$
\\ \hline
Spambase & NaN & 0.20046 & 0.17361 & 0.17958 & \textbf{0.15332}\textcolor{magenta}{$\star$} & \textbf{0.15684} & \textbf{0.15392} & \textbf{0.15423} & \textbf{0.15434} & \textbf{0.15487}\textcolor{magenta}{$\star$} \\ \hline
Bupa & 0.27005 & \makecell{0.38167 \\ \textit{0.34547}} & 0.29265 & 0.30537 & \textbf{0.23851}\textcolor{magenta}{$\star$} & \textbf{0.13207} & \makecell{0.145801 \\ \textit{0.14873}} & 0.13631 & 0.13382 & \textbf{0.11998}\textcolor{magenta}{$\star$} \\ \hline
\makecell{Mammographic} & 0.29518 & \makecell{0.34187 \\ \textit{0.31290}}& 0.28790 & 0.29659 & \textbf{0.26063}\textcolor{magenta}{$\star$} & \textbf{0.20462} &\makecell{\textbf{0.21120} \\ \textit{0.21386}} &\textbf{0.20716} & \textbf{0.20628} & \textbf{0.20519}\textcolor{magenta}{$\star$} \\ \hline
Wdbc & 0.20706 & \makecell{0.26000 \\ \textit{0.22122}}& 0.20236 & 0.21646 & \textbf{0.14759}\textcolor{magenta}{$\star$} & \textbf{0.06489} & \makecell{\textbf{0.06901} \\ \textit{0.07052}}& \textbf{0.06650} & \textbf{0.06584} & \textbf{0.06541}\textcolor{magenta}{$\star$}\\ \hline
Banknote & 0.13647 & \makecell{0.13225 \\ \textit{0.10343}}& 0.09538 & 0.10672 & \textbf{0.02051} & 0.00161 & \makecell{0.00561 \\ \textit{0.00592}} & 0.00500 & 0.00469 & \textbf{0.00037} \\ \hline
Mushroom & NaN & 0.06584 & 0.04702 & 0.05399 & \textbf{0.00489} & \textbf{8.92{e-05}} & 0.00066 & 0.00057 & 0.00053 & \textbf{1.39{e-05}}\\ \hline
Ionosphere & 0.20816 & \makecell{0.30151 \\ \textit{0.25884}} & 0.22508 & 0.24011 & \textbf{0.14707}\textcolor{magenta}{$\star$} & \textbf{0.04494} & \makecell{\textbf{0.04781} \\ \textit{0.04899}}& \textbf{0.04393}
& \textbf{0.04553} & \textbf{0.04359}\textcolor{magenta}{$\star$} \\ \hline
Waveform & NaN & 0.12875 & 0.10335 & 0.11103 & \textbf{0.06338} & 0.05847 & \textbf{0.05175} & \textbf{0.05276} & \textbf{0.05345} & 0.05792 \\ \hline
Haberman & \textbf{0.37277} & \makecell{ 0.48385 \\ \textit{0.43977}} & \textbf{0.39769} & 0.41178 & \textbf{0.37998}\textcolor{magenta}{$\star$} & \textbf{0.29157} & \makecell{\textbf{0.29069} \\ \textit{0.29007}} & \textbf{0.29163} & \textbf{0.29162} & \textbf{0.28997}\textcolor{magenta}{$\star$} \\ 
\hline
\end{tabular}
\end{center}
}
\caption{\textbf{Earlier computations for comparison}. We obtain improved bound values and test error rates for optimal posteriors $Q^{\ast}_{\phi, \text{KL}}$ in Table \ref{tab:phiKLCompare} above by restricting the base classifiers generated using $\Lambda = \lbrace 10^{-5}, 5 \rbrace$ obtained as a combination of arithmetic-geometric progression for the interval $(10^{-5}, 0.1)$ and linearly spaced values between 0.1 and 5 at a granularity of 0.01. The test error rates for Spambase (improved from 0.15 to 0.10 across 5 distance functions), Mammographic, Wdbc and Banknote datasets decreased significantly. For other datasets, the test error rates are comparable.}
\end{table}

\begin{sidewaystable}[]
{ \setlength{\tabcolsep}{0.1em}
\footnotesize
\begin{tabular}{|c|c|c|c|c|c|c|c|c|c|c|c|}
\hline
\diagbox{\textbf{Dataset}}{\textbf{H}} & \multicolumn{2}{c|}{\textbf{50}} & \multicolumn{2}{c|}{\textbf{200}} & \multicolumn{2}{c|}{\textbf{500}} & \multicolumn{2}{c|}{\textbf{1000}} & \multicolumn{2}{c|}{\textbf{1990}} \\  \cline{2-11}
(Validation  set size, $v$)& $B^{FP}_{\text{kl, KL}}$ &  $B^{solver}_{\text{kl, KL}}$  &  $B^{FP}_{\text{kl, KL}}$ &  $B^{solver}_{\text{kl, KL}}$ & $B^{FP}_{\text{kl, KL}}$ &  $B^{solver}_{\text{kl, KL}}$ & $B^{FP}_{\text{kl, KL}}$ &  $B^{solver}_{\text{kl, KL}}$  & $B^{FP}_{\text{kl, KL}}$ &  $B^{solver}_{\text{kl, KL}}$ \\ \hline
\makecell{Spambase \\ $(v = 1840)$ } & 0.14726 & 0.147260  & 0.149424 & 0.149424  & 0.15157 & 0.270042(\textbf{E}) & 0.152023 & 0.294836(\textbf{E})  & 0.153324 & 0.314523(\textbf{E})
 \\ \hline
 \makecell{Bupa \\ $( v = 138)$ }  & 0.208330 & 0.208333    & 0.220062 & 0.220065   & 0.227504 &  0.437317(\textbf{E})  & 0.232998 & 0.508671(\textbf{E})   & 0.238509 & 0.576823(\textbf{E})
 \\ \hline
 \makecell{Mammographic \\ $( v = 332)$ } & 0.241706 & 0.241680  & 0.249234 & 0.249235 & 0.253854 & 0.253847  & 0.257411 & 0.302582(\textbf{E})   &0.260632 & 0.335105(\textbf{E})
 \\ \hline
\makecell{Wdbc \\ $( v = 227)$ }  & 0.127827 & 0.127827  & 0.134727 & 0.134714 & 0.139659 & 0.139655  & 0.14363 & 0.143656   & 0.147595 & 0.187134(\textbf{E})
 \\ \hline
\makecell{Banknote \\ $( v = 549)$ } & 0.015278 & 0.015278 & 0.016358 & 0.016356 & 0.018065 & 0.018065   & 0.232998 & 0.513805(\textbf{E})  & 0.238509 & 0.573999(\textbf{E})
 \\ \hline
\makecell{Mushroom \\ $( v = 2257)$ } & 0.004050 & 0.004050 &  0.004050 & 0.004050 & 0.004150  & 0.004150 &0.004517 &0.004517 &0.004882 &0.004883 
 \\ \hline
\makecell{Ionosphere \\ $( v = 140)$ } & 0.119248 & 0.122997(\textbf{M}) & 0.129552 & 0.141167(\textbf{M}) & 0.13658 & 0.136579 & 0.141938 & 0.275581(\textbf{E}) & 0.147074 & 0.404999(\textbf{E}) 
 \\ \hline
\makecell{Waveform \\ $( v = 1323)$ } & 0.058419  & 0.058424 & 0.060210 & 0.060206 & 0.061562 & 0.06157 & 0.062467 & 0.062473 & 0.063376 & 0.063387 
 \\ \hline
\makecell{Haberman \\ $( v = 122)$ } & 0.342978 & 0.350085(\textbf{M}) & 0.356983 & 0.356982 & 0.366412 & 0.407535(\textbf{E}) & 0.373351 & 0.421606(\textbf{E}) & 0.379982 & 0.427411(\textbf{E})  \\ \hline
\end{tabular}}
\caption[Bound values for kl-KL case]{\small \textbf{Bound values for kl-KL case}: Comparing the PAC-Bayesian bounds for the case of KL divergence as distance function with KL divergence measure between prior and posterior distributions. $B^{KKT}_{\text{kl, KL}}$ represents the bound due to the fixed point solution obtained by solving the partial KKT system for the bound minimization problem (7) in the main paper. $B^{solver}_{\text{kl, KL}}$ is the value of the bound output by the solver for the same case. We observe that the bound values are exactly same for certain cases. The fixed point equation always converges to a solution, even when the \texttt{Ipopt} solver is not able to identify a locally optimal solution and throws up an error. These cases are denoted by `\textbf{E}' (Unknown Error) or `\textbf{R}' (Restoration Phase Failed) or `\textbf{M}' (Maximum Number of Iterations Exceeded), as output by the solver.}
\label{tab:Bnd.klKL}
\end{sidewaystable}

%=========================CCP for kl-KL==================================
\begin{table}[ht]
\centering 
{ \footnotesize
\setlength{\tabcolsep}{0.25em}
\begin{tabular}{ |c | c c c| c c c| }
\hline
\textbf{Dataset} & \multicolumn{3}{c|}{\textbf{PAC-Bayesian Bound}} & \multicolumn{3}{c|}{\textbf{Average Test Error}} \\
\cline{2-7}
  & $B^{FP}_{\text{kl, KL}}$ & Range($B^{CCP}_{\text{kl, KL}}$) & Mean($B^{CCP}_{\text{kl, KL}}$) & $T^{FP}_{\text{kl, KL}}$ & Range($T^{CCP}_{\text{kl, KL}}$) & Mean($T^{CCP}_{\text{kl, KL}}$)
\\ 
\hline
Spambase  & 0.14726 & [0.16632, 0.19290] & 0.18257$\pm$ 0.00301
 &  0.15465 & [0.16412, 0.18537] & 0.17578 $\pm$ 0.00235 \\
Bupa &  0.20833 & [0.23380, 0.26191] & 0.24741 $\pm$ 0.00412 & 0.12502 & [0.14943, 0.18810] & 0.16754 $\pm$ 0.00599 \\
\makecell{Mammographic} & 0.24171 & [0.24760, 0.25558] & 0.25190 $\pm$ 0.00116 & 0.20566 & [0.20665, 0.21793] & 0.21209 $\pm$ 0.00195\\
Wdbc & 0.12782 & [0.13061, 0.13659] &  0.13320 $\pm$ 0.00085 & 0.06630 & [0.05925, 0.07212] &  0.06492 $\pm$ 0.00183 \\
Banknote & 0.01528 & NA & NA & 0.00036 & NA & NA\\
Mushroom & 0.00405 & NA & NA & 0 & NA & NA\\
Ionosphere & 0.11925 & [0.12284, 0.13132] & 0.12631$\pm$ 0.00119 & 0.04409 & [0.03889, 0.05328] & 0.04562 $\pm$ 0.00214\\
Waveform & 0.05842 & [0.06353, 0.06711] & 0.06525 $\pm$ 0.00061 & 0.05749 & [0.05003, 0.05451] & 0.05213 $\pm$ 0.00073\\
Haberman & 0.34298 & [0.34857, 0.36011] & 0.35417 $\pm$ 0.00175 & 0.29257 & [0.28524, 0.30430] & 0.29346 $\pm$ 0.00286\\ 
\hline
\end{tabular}
}
\caption[Comparing bound values and test error rates of FP based posterior with CCP based posterior for kl-KL case]{We compare the bound values and test error rates of the optimal posterior obtained via Fixed Point (FP) scheme and the posterior based on Convex-Concave Procedure (CCP) minimizing the PAC-Bayesian bound $B_{\text{kl, KL}}$ based on KL-distance function with KL-divergence measure. The CCP based posteriors are identified by the bound minimization model described in Section \ref{secn:klKLCCP}. The bound values and test error rates for FP scheme based solution are denoted by $B^{FP}_{\text{kl, KL}}$ and $T^{FP}_{\text{kl, KL}}$. Similarly, the bound values and test error rates of the CCP based posterior are denoted by $B^{CCP}_{\text{kl, KL}}$ and $T^{CCP}_{\text{kl, KL}}$. For computations, we consider SVM classifiers generated on nine datasets from UCI repository \cite{UCI:2017} using the scheme in Section \ref{secn:datacategory.compscheme} for $H=50$ values in $\Lambda = \lbrace 0.1, 0.11, \ldots, \rbrace$. We run the CCP procedure for 1000 different initializations of posterior $Q^0$ (as done in \cite{CCP2016LippBoyd}). The range, mean and standard deviation of the bound values and average test error rates of the CCP based posteriors obtained by these 1000 initializations are tabulated above. We notice that $B^{FP}_{\text{kl, KL}}$ is always better than $B^{CCP}_{\text{kl, KL}}$ and $T^{FP}_{\text{kl, KL}}$ is comparable with mean value of $T^{CCP}_{\text{kl, KL}}$ for different datasets considered. This might be because FP scheme identifies the global minimum for kl-KL based bound minimization problem, whereas CCP converges to a local solution or a stationary point. `NA' denotes the cases where the CCP cannot provide linear approximation to $kl(\mathbb{E}_Q[\hat{l}], r)$ because a subgradient cannot be determined when $\mathbb{E}_Q[\hat{l}]$ takes the boundary value zero. Such cases usually occur for almost separable datasets -- Banknote and Mushroom, where the quantity $\mathbb{E}_Q[\hat{l}] = 0$ for any distribution $Q$ since all $\hat{l}_i$s take value zero for $i = 1, \ldots, 50$.}
\label{tab:Bnd.TestErr.FPnCCP.klKL}
\end{table}

\subsection{Comparison of posterior on full support with that on subset support}
We have shown that linear distance based bound $B_{\text{lin, KL}}(Q)$ has full support when prior is uniform. For other four distance functions, $\phi$s (squared distance, KL-distance, Pinsker's approximation and sixth degree polynomial approximation), we analyze their support set by computations on UCI datasets. For uniform prior on classifier set $\mathcal{H}$, we compare the local minimizers of $B_{\phi, \text{KL}}(Q)$ on $H$-dimensional simplex (allowing for subset support), with the one computed on interior of $H$-dimensional simplex (full support). $H$ denotes the classifier set size and $H^{\ast}$ denotes the optimal support size. We observe that datasets with low and moderate variation in empirical risk values have full support, $H = H^{\ast}$, whereas those with high variation have a smaller support but can be approximated by optimal posterior determined on a full support as reported in Tables \ref{tab:subsetvsfullsupport.sqKL} - \ref{tab:subsetvsfullsupport.CHKL}.

%========================SubsetSupport for sqKL case===========================================
\begin{sidewaystable}
\setlength{\tabcolsep}{0.24em}
{\footnotesize
\begin{tabular}{|c|c|c|c|c|c|c|c|c|c|c|}
\hline
\makecell{Dataset \\ (Validation \\ set size, $v$)} & \makecell{\# Classifiers \\ $H$} & \makecell{Optimal \\ Support \\ Size, $H^{\ast}$} & $B^{solver}_{\text{sq, KL}}(H)$ & $B^{FP}_{\text{sq, KL}}(H^{\ast})$ & $T^{solver}_{\text{sq, KL}}(H)$ & $T^{FP}_{\text{sq, KL}}(H^{\ast})$ & $\Vert Q^{solver}_{\text{sq, KL}}(H) - Q^{FP}_{\text{sq, KL}}(H^{\ast})\Vert_2$ & Time $(H)$ & Time $(H^{\ast})$ \\
\hline
\makecell{Spambase \\ $(v = 1840)$}& 500 & 31 & 0.19625 & 0.19625 & 0.15390 & 0.15389 & 7.66 e-05 & 2.409 s &  0.112 s \\
\hline
\makecell{Bupa \\ $(v = 138)$} & 1000 & 275 & 0.33576 & 0.33576 & 0.14955 & 0.14954 & 0.000103 & 18.58 s & 0.398 s\\
\hline
\makecell{Mammographic \\ $(v = 332)$ }& 1000 & 1000 & 0.30587 & 0.30587 & 0.21461 & 0.21461 & 0.000907 & 13.223 s & 0.407 s \\
\hline
\makecell{Wdbc \\ $(v = 227)$ } & 1990 & 1922 & 0.22121 & 0.22121 & 0.70523 & 0.70523 & 0.000391 & 174.024 s & 0.842 s \\
\hline
\makecell{Banknote \\ $(v = 549)$} & 200 & 200 & 0.09646 & 0.09646 & 0.001756 & 0.001757 & 3.232 e-05 & 0.166 s & 0.04 s\\
\hline
\makecell{Mushroom \\ $(v = 2257)$} & 1990 & 1990 & 0.06584 & 0.06584 & 0.00066 & 0.00066 & 0.000675 & 63.025 s & 0.755 s \\
\hline
\makecell{Ionosphere \\ $(v = 140)$} & 200 & 200 & 0.22720 & 0.22720 & 0.43921  & 0.43291 & 1.146 e-05 & 0.164 s & 0.033 s \\
\hline
\makecell{Waveform \\ $(v = 1323)$} & 1000 & 1000 & 0.12685 & 0.12685 & 0.05200 & 0.05200  & 0.000311 & 10.655 s & 0.276 s \\
\hline
\makecell{Haberman \\ $(v = 122)$} & 500 & 500 & 0.41943 & 0.41943 & 0.28989 & 0.28989 & 0.000737 & 2.048 s & 0.093 s \\
\hline
\end{tabular}
}
\caption[Comparing minimizers of $B_{\text{sq,KL}}(Q)$ on full support versus optimal subset support for uniform prior]{For uniform prior on classifier set $\mathcal{H}$, we compare the (local) minimizer, $Q^{FP}_{\text{sq, KL}}(H^{\ast})$, of the bound function, $B_{\text{sq,KL}}(Q)$ on the whole of $H$-dimensional simplex (allowing for subset support), with the (local) minimizer, $Q^{solver}_{\text{sq, KL}}(H)$, computed on the interior of the $H$-dimensional simplex (full support). $H$ denotes the size of the classifier set considered and $H^{\ast}$ denotes the size of the support set for $Q^{FP}_{\text{sq,KL}}$, the local minimum for (24) in the main paper. We call $H^{\ast}$ as the `optimal support size'. $Q^{FP}_{\text{sq, KL}}(H^{\ast})$ was determined via a linear search among the optimal posteriors with support on the increasing ordered subsets of $\mathcal{H}$. $B^{solver}_{\text{sq,KL}}(H)$ and $B^{FP}_{\text{sq,KL}}(H^{\ast})$ denote the bound values of the two posteriors; and $T^{solver}_{\text{sq,KL}}(H)$ and $T^{FP}_{\text{sq,KL}}(H^{\ast})$ denote their average test error rates. All the computations were done with $\delta = 0.01$. We notice that, in many datasets (for example, Mammographic, Ionosphere and Banknote), $H^{\ast} = H$, indicating that the local minimizer indeed has full support. These datasets have low to moderate variance in the empirical risk values. Whereas the ones with notably high variation in empirical risk values (namely, Spambase, Bupa and Wdbc) have optimizers on a much smaller support size ($H^{\ast} \ll H$). Yet the difference in the bound values and test error rates of the full support posterior, $Q^{solver}_{\text{sq, KL}}(H)$ and the local minimizer $Q^{FP}_{\text{sq, KL}}(H^{\ast})$ is negligible ($O(10^{-6})$). If we compare the posteriors themselves, the $\ell_2$-norm of the difference between them is very small ($O(10^{-4})$). This suggests that the posteriors lie within a small neighbourhood. Considering the computation time taken, we observe that $Q^{FP}_{\text{sq, KL}}(H)$ with full support can be computed within a fraction of a second, even with the linear search, whereas computation for $Q^{solver}_{\text{sq, KL}}(H^{\ast})$ can take anywhere between 0.164 seconds to 174.024 seconds depending on the model parameters -- $H, S_m, \delta, \lbrace l_i \rbrace_{i = 1}^{H}$. We can obtain a very close approximation to $Q^{FP}_{\text{sq, KL}}(H^{\ast})$ by considering a full support posterior obtained by minimizing $B_{\text{sq, KL}}(Q)$ in the interior of the $H$-dimensional probability simplex.}
\label{tab:subsetvsfullsupport.sqKL}
\end{sidewaystable}

%========================SubsetSupport for klKL case===========================================

\begin{sidewaystable}
\setlength{\tabcolsep}{0.24em}
{\footnotesize
\begin{tabular}{|c|c|c|c|c|c|c|c|c|c|c|}
\hline
\makecell{Dataset \\ (Validation \\ set size, $v$)} & \makecell{\# Classifiers \\ $H$} & \makecell{Optimal \\ Support \\ Size, $H^{\ast}$} & $B^{solver}_{\text{kl, KL}}(H)$ & $B^{FP}_{\text{kl, KL}}(H^{\ast})$ & $T^{solver}_{\text{kl, KL}}(H)$ & $T^{FP}_{\text{kl, KL}}(H^{\ast})$ & $\Vert Q^{solver}_{\text{kl, KL}}(H) - Q^{FP}_{\text{kl, KL}}(H^{\ast})\Vert_2$ & Time $(H)$ & Time $(H^{\ast})$ \\
\hline
\makecell{Spambase \\ $(v = 1840)$}& 500 & 3 & 0.26885 (\textbf{E}) & 0.15115 & 0.24152 (\textbf{E})& 0.15480 & 1.995252 (\textbf{E}) & 0.035 s (\textbf{E}) & 0.403 s\\
\hline
\makecell{Bupa \\ $(v = 138)$} & 1000 & 19 & 0.51500 (\textbf{E}) & 0.23295 & 0.38550 (\textbf{E})& 0.12097 & 1.971075 (\textbf{E}) & 0.064 s (\textbf{E}) & 1.207 s \\
\hline
\makecell{Mammographic \\ $(v = 332)$ }& 1000 & 40 & 0.30370 (\textbf{E}) & 0.25731 & 0.22828 (\textbf{E}) & 0.20505 & 1.929077 (\textbf{E}) & 0.067 s (\textbf{E}) & 1.479 s \\
\hline
\makecell{Wdbc \\ $(v = 227)$ } & 1990 & 613 & 0.18527 (\textbf{E}) & 0.14745 & 0.08870 (\textbf{E})& 0.06536 & 1.757514 (\textbf{E}) & 0.185 s (\textbf{E}) & 4.207 s\\
\hline
\makecell{Banknote \\ $(v = 549)$} & 200 & 110 & 0.01635 & 0.01635 & 0.00037 & 0.00037 & 1.952 e-05 & 0.066 s & 0.063 s \\
\hline
\makecell{Mushroom \\ $(v = 2257)$} & 1990 & 336 & 0.00488 & 0.00488 & 1.399 e-05& 1.318 e-05 & 0.003724 & 13.4.57 s & 2.126 s \\
\hline
\makecell{Ionosphere \\ $(v = 140)$} & 200 & 186 & 0.12955 & 0.12952 & 0.04378 & 0.04378 & 0.002737 & 12.132 s & 0.182 s \\
\hline
\makecell{Waveform \\ $(v = 1323)$} & 1000 & 7 & 0.06247 & 0.06240 & 0.05785 & 0.05794 & 0.024940 & 35.99 s & 1.064 s\\
\hline
\makecell{Haberman \\ $(v = 122)$} & 500 & 180 & 0.40832 (\textbf{E}) & 0.36638 & 0.28768  (\textbf{E})& 0.29161 & 1.718098 (\textbf{E}) & 0.037 s (\textbf{E})& 0.482 s\\
\hline
\end{tabular}
}
\caption[Comparing minimizers of $B_{\text{kl,KL}}(Q)$ on full support versus optimal subset support for uniform prior]{For uniform prior on classifier set $\mathcal{H}$, we compare the (local) minimizer, $Q^{FP}_{\text{kl, KL}}(H^{\ast})$, of the bound function, $B_{\text{kl,KL}}(Q)$ on the whole of $H$-dimensional simplex (allowing for subset support), with the (local) minimizer, $Q^{solver}_{\text{kl, KL}}(H)$, computed on the interior of the $H$-dimensional simplex (full support). $H$ denotes the size of the classifier set considered and $H^{\ast}$ denotes the size of the support set for $Q^{FP}_{\text{kl,KL}}$, the local minimum for (7) in the main paper. We call $H^{\ast}$ as the `optimal support size'. $Q^{FP}_{\text{kl, KL}}(H^{\ast})$ was determined via a linear search among the optimal posteriors with support on the increasing ordered subsets of $\mathcal{H}$. $B^{solver}_{\text{kl,KL}}(H)$ and $B^{FP}_{\text{kl,KL}}(H^{\ast})$ denote the bound values of the two posteriors; and $T^{solver}_{\text{kl,KL}}(H)$ and $T^{FP}_{\text{kl,KL}}(H^{\ast})$ denote their average test error rates. All the computations were done with $\delta = 0.01$. We notice that, if $H$ is large, the solver does not converge to a solution in many datasets (for example, Mammographic, Bupa and Haberman) which have moderate to considerable variance in the empirical risk values. These cases are denoted by the symbol `\textbf{E}'. Whereas the almost separable datasets with low variation in empirical risk values (namely, Banknote, Mushroom, Ionosphere and Waveform) have optimizers on a much smaller support size ($H^{\ast} \ll H$). Yet the difference in the bound values and test error rates of the full support posterior, $Q^{solver}_{\text{kl, KL}}(H)$ and the local minimizer $Q^{FP}_{\text{kl, KL}}(H^{\ast})$ is negligible ($O(10^{-5})$). If we compare the posteriors themselves, the $\ell_2$-norm of the difference between them is very small ($O(10^{-3})$). This suggests that the posteriors lie within a small neighbourhood. Considering the computation time taken, we observe that $Q^{FP}_{\text{kl, KL}}(H)$ with full support can be computed within fractions of a upto 4.207 seconds, even with the linear search. Whereas computation for $Q^{solver}_{\text{kl, KL}}(H^{\ast})$ can take anywhere between 0.066 seconds to 35.99 seconds depending on the model parameters -- $H, S_m, \delta, \lbrace l_i \rbrace_{i = 1}^{H}$. We can obtain a very close approximation to $Q^{FP}_{\text{kl, KL}}(H^{\ast})$ by considering a full support posterior obtained by minimizing $B_{\text{kl, KL}}(Q)$ in the interior of the $H$-dimensional probability simplex.}
\label{tab:subsetvsfullsupport.klKL}
\end{sidewaystable}

%========================SubsetSupport for 2sqKL case===========================================

\begin{sidewaystable}
\setlength{\tabcolsep}{0.24em}
{ \footnotesize
\begin{tabular}{|c|c|c|c|c|c|c|c|c|c|c|}
\hline
\makecell{Dataset \\ (Validation \\ set size, $v$)} & \makecell{\# Classifiers \\ $H$} & \makecell{Optimal \\ Support \\ Size, $H^{\ast}$} & $B^{solver}_{\text{P, KL}}(H)$ & $B^{FP}_{\text{P, KL}}(H^{\ast})$ & $T^{solver}_{\text{P, KL}}(H)$ & $T^{FP}_{\text{P, KL}}(H^{\ast})$ & $\Vert Q^{solver}_{\text{P, KL}}(H) - Q^{FP}_{\text{P, KL}}(H^{\ast})\Vert_2$ & Time $(H)$ & Time $(H^{\ast})$ \\
\hline
\makecell{Spambase \\ $(v = 1840)$}& 500 & 21 & 0.17065 & 0.17065 & 0.15416 & 0.15416 & 7.594 e-05 & 2.389 s & 0.101 s \\
\hline
\makecell{Bupa \\ $(v = 138)$} & 1000 & 100 & 0.28683 & 0.28683 & 0.13714 & 0.13714 & 7.896 e -05 & 15.441 s & 0.316 s \\
\hline
\makecell{Mammographic \\ $(v = 332)$ } & 1000 & 954 & 0.28207 & 0.28207 & 0.20728 & 0.20728 & 0.000668 & 13.269 s & 0.390 s\\
\hline
\makecell{Wdbc \\ $(v = 227)$ } & 1990 & 1860 & 0.20236 & 0.20236  & 0.06650 & 0.06650 & 0.000249 & 114.228 s & 0.759 s \\
\hline
\makecell{Banknote \\ $(v = 549)$} & 200 & 200 & 0.08909 & 0.08909 & 0.00148 & 0.00148 & 7.181 e-05 & 0.203 s & 0.025 s \\
\hline
\makecell{Mushroom \\ $(v = 2257)$} & 1990 & 1990 & 0.04702 & 0.04702 & 0.00057 & 0.00057 & 0.001312 & 88.313 s & 0.641 s \\
\hline
\makecell{Ionosphere \\ $(v = 140)$} & 200 & 200 & 0.20473 & 0.20473 & 0.04406 & 0.04406 & 0.000186 & 0.230 s & 0.028 s\\
\hline
\makecell{Waveform \\ $(v = 1323)$} & 1000 & 1000 & 0.10161 & 0.10161 & 0.05289 & 0.05289 & 0.000912 & 11.146 s & 0.323 s\\
\hline
\makecell{Haberman \\ $(v = 122)$} & 500 & 500 & 0.38421 & 0.38421 & 0.29159 & 0.29159 & 0.000190 & 2.184 s & 0.103 s \\
\hline
\end{tabular}
}
\caption[Comparing minimizers of $B_{\text{P,KL}}(Q)$ on full support versus optimal subset support for uniform prior]{For uniform prior on classifier set $\mathcal{H}$, we compare the (local) minimizer, $Q^{FP}_{\text{P, KL}}(H^{\ast})$, of the bound function, $B_{\text{P,KL}}(Q)$ on the whole of $H$-dimensional simplex (allowing for subset support), with the (local) minimizer, $Q^{solver}_{\text{P, KL}}(H)$, computed on the interior of the $H$-dimensional simplex (full support). $H$ denotes the size of the classifier set considered and $H^{\ast}$ denotes the size of the support set for $Q^{FP}_{\text{P,KL}}$, the local minimum for (11) in the main paper. We call $H^{\ast}$ as the `optimal support size'. $Q^{FP}_{\text{P, KL}}(H^{\ast})$ was determined via a linear search among the optimal posteriors with support on the increasing ordered subsets of $\mathcal{H}$. $B^{solver}_{\text{P,KL}}(H)$ and $B^{FP}_{\text{P,KL}}(H^{\ast})$ denote the bound values of the two posteriors; and $T^{solver}_{\text{P,KL}}(H)$ and $T^{FP}_{\text{P,KL}}(H^{\ast})$ denote their average test error rates. All the computations were done with $\delta = 0.01$. We notice that, in many datasets (for example, Haberman, Ionosphere and Banknote), $H^{\ast} = H$, indicating that the local minimizer indeed has full support. These datasets have low to moderate variance in the empirical risk values. Whereas the ones with notably high variation in empirical risk values (namely, Spambase, Bupa, Mammographic and Wdbc) have optimizers on a much smaller support size ($H^{\ast} \ll H$). Yet the difference in the bound values and test error rates of the full support posterior, $Q^{solver}_{\text{P, KL}}(H)$ and the local minimizer $Q^{FP}_{\text{P, KL}}(H^{\ast})$ is negligible ($O(10^{-6})$). If we compare the posteriors themselves, the $\ell_2$-norm of the difference between them is very small ($O(10^{-4})$). This suggests that the posteriors lie within a small neighbourhood. Considering the computation time taken, we observe that $Q^{FP}_{\text{P, KL}}(H)$ with full support can be computed within a fraction of a second, even with the linear search, whereas computation for $Q^{solver}_{\text{P, KL}}(H^{\ast})$ can take anywhere between 0.203 seconds to 114.228 seconds depending on the model parameters -- $H, S_m, \delta, \lbrace l_i \rbrace_{i = 1}^{H}$. We can obtain a very close approximation to $Q^{FP}_{\text{P, KL}}(H^{\ast})$ by considering a full support posterior obtained by minimizing $B_{\text{P, KL}}(Q)$ in the interior of the $H$-dimensional probability simplex.}
\label{tab:subsetvsfullsupport.2sqKL}
\end{sidewaystable}

%========================SubsetSupport for chKL case===========================================

\begin{sidewaystable}
\setlength{\tabcolsep}{0.24em}
{ \footnotesize
\begin{tabular}{|c|c|c|c|c|c|c|c|c|c|c|}
\hline
\makecell{Dataset \\ (Validation \\ set size, $v$)} & \makecell{\# Classifiers \\ $H$} & \makecell{Optimal \\ Support \\ Size, $H^{\ast}$} & $B^{solver}_{\text{CH, KL}}(H)$ & $B^{FP}_{\text{CH, KL}}(H^{\ast})$ & $T^{solver}_{\text{CH, KL}}(H)$ & $T^{FP}_{\text{CH, KL}}(H^{\ast})$ & $\Vert Q^{solver}_{\text{CH, KL}}(H) - Q^{FP}_{\text{CH, KL}}(H^{\ast})\Vert_2$ & Time $(H)$ & Time $(H^{\ast})$ \\
\hline
\makecell{Spambase \\ $(v = 1840)$}& 500 & 18 & 0.17688 & 0.17688 & 0.15428 & 0.15428 & 7.461 e-05 & 2.643 s & 0.171 s \\
\hline
\makecell{Bupa \\ $(v = 138)$} & 1000 & 89 & 0.30002 & 0.30002 & 0.13462 & 0.13461 & 7.857 e -05 & 19.56 s & 0.593 s \\
\hline
\makecell{Mammographic \\ $(v = 332)$ } & 1000 & 892 & 0.29317 & 0.29317 & 0.20647 & 0.20646 & 0.000674 & 14.663 s & 0.706 s\\
\hline
\makecell{Wdbc \\ $(v = 227)$ } & 1990 & 1856 & 0.21646 & 0.21646  & 0.06584 & 0.06584 & 0.000305 & 143.57 s & 1.371 s \\
\hline
\makecell{Banknote \\ $(v = 549)$} & 200 & 200 & 0.10063 & 0.10063 & 0.00141 & 0.00141 & 7.381 e-05& 0.189 s & 0.041 s \\
\hline
\makecell{Mushroom \\ $(v = 2257)$} & 1990 & 1990 & 0.05398 & 0.05398 & 0.00053 & 0.00053 & 0.001625 & 71.778 s & 0.972 s \\
\hline
\makecell{Ionosphere \\ $(v = 140)$} & 200 & 200 & 0.22104 & 0.22104 & 0.04410 & 0.04410 & 0.000139 & 0.223 s & 0.055 s\\
\hline
\makecell{Waveform \\ $(v = 1323)$} & 1000 & 1000 & 0.10940 & 0.10940 & 0.05350 & 0.05350 & 0.000384 & 13.745 s & 0.597 s\\
\hline
\makecell{Haberman \\ $(v = 122)$} & 500 & 500 & 0.39950 & 0.39950 & 0.21655 & 0.21655 & 0.000177 & 2.499 s & 0.181 s \\
\hline
\end{tabular}
}
\caption[Comparing minimizers of $B_{\text{CH,KL}}(Q)$ on full support versus optimal subset support for uniform prior]{For uniform prior on classifier set $\mathcal{H}$, we compare the (local) minimizer, $Q^{FP}_{\text{CH, KL}}(H^{\ast})$, of the bound function, $B_{\text{CH,KL}}(Q)$ on the whole of $H$-dimensional simplex (allowing for subset support), with the (local) minimizer, $Q^{solver}_{\text{CH, KL}}(H)$, computed on the interior of the $H$-dimensional simplex (full support). $H$ denotes the size of the classifier set considered and $H^{\ast}$ denotes the size of the support set for $Q^{FP}_{\text{CH,KL}}$, the local minizer for the bound $B_{\text{CH, KL}}$ in (16) in the main paper. We call $H^{\ast}$ as the `optimal support size'. $Q^{FP}_{\text{CH, KL}}(H^{\ast})$ was determined via a linear search among the optimal posteriors with support on the increasing ordered subsets of $\mathcal{H}$. $B^{solver}_{\text{CH,KL}}(H)$ and $B^{FP}_{\text{CH,KL}}(H^{\ast})$ denote the bound values of the two posteriors; and $T^{solver}_{\text{CH,KL}}(H)$ and $T^{FP}_{\text{CH,KL}}(H^{\ast})$ denote their average test error rates. All the computations were done with $\delta = 0.01$. We notice that, in many datasets (for example, Mammographic, Ionosphere and Banknote), $H^{\ast} = H$, indicating that the local minimizer indeed has full support. These datasets have low to moderate variance in the empirical risk values. Whereas the ones with notably high variation in empirical risk values (namely, Spambase, Bupa, Mammographic and Wdbc) have optimizers on a much smaller support size ($H^{\ast} \ll H$). Yet the difference in the bound values and test error rates of the full support posterior, $Q^{solver}_{\text{CH, KL}}(H)$ and the local minimizer $Q^{FP}_{\text{CH, KL}}(H^{\ast})$ is negligible ($O(10^{-6})$). If we compare the posteriors themselves, the $\ell_2$-norm of the difference between them is very small ($O(10^{-4})$). This suggests that the posteriors lie within a small neighbourhood. Considering the computation time taken, we observe that $Q^{FP}_{\text{CH, KL}}(H)$ with full support can be computed within a fraction of a second, even with the linear search, whereas computation for $Q^{solver}_{\text{CH, KL}}(H^{\ast})$ can take anywhere between 0.223 seconds to 71.778 seconds depending on the model parameters -- $H, S_m, \delta, \lbrace l_i \rbrace_{i = 1}^{H}$. We can obtain a very close approximation to $Q^{FP}_{\text{CH, KL}}(H^{\ast})$ by considering a full support posterior obtained by minimizing $B_{\text{CH, KL}}(Q)$ in the interior of the $H$-dimensional probability simplex.}
\label{tab:subsetvsfullsupport.CHKL}
\end{sidewaystable}

%\bibliography{ACML2019bibDOI}

\end{document}